\documentclass[11pt,a4paper]{article}
\usepackage{authblk}
\usepackage{graphicx}
\usepackage{amsmath}
\usepackage{amssymb}
\usepackage{algorithm}
\usepackage{algorithmic}
\usepackage{amsthm}
\usepackage{epstopdf}
\usepackage{caption}
\usepackage{url}
\usepackage{subcaption}
\usepackage{natbib}
\usepackage{cases}
\newtheorem{theorem}{Theorem}
\newtheorem{lemma}{Lemma}

\newtheorem{corollary}{Corollary}

\theoremstyle{definition}

\title{\bf SAGA and Restricted Strong Convexity }
\date{}
\author[1]{Chao Qu}
\author[2]{Yan Li}
\author[2]{Huan Xu}
\affil[1]{Department of Mechanical Engineering, National University of Singapore}
\affil[2]{H. Milton Stewart School of Industrial and Systems Engineering, Georgia Institute of Technology}

\begin{document} 
	\maketitle
\begin{abstract} 
	SAGA is a fast incremental gradient method on the finite sum problem and its effectiveness has been tested on a vast of applications. In this paper, we analyze SAGA on a class of non-strongly convex and non-convex statistical problem such as Lasso, group Lasso, Logistic regression with $\ell_1$ regularization, linear regression with SCAD regularization and Correct Lasso. We prove that SAGA enjoys the linear convergence rate up to the statistical estimation accuracy, under the assumption of restricted strong convexity (RSC). It significantly extends the applicability of SAGA in convex and non-convex optimization.
\end{abstract} 

\section{Introduction}

We study the finite sum problem in the following forms:

\begin{itemize}
	\item Convex $G(\theta)$:
	\begin{equation}\label{obj1}
	\begin{split}
	\underset{\psi(\theta)\leq \rho}{\mbox{Minimize:}}\,\,  G(\theta)&\triangleq f(\theta)+\lambda \psi(\theta)\\
	&=\frac{1}{n}\sum_{i=1}^{n}f_i(\theta)+\lambda \psi(\theta),
	\end{split}
	\end{equation}
	where $f_i(\theta)$ is a convex loss such as $f_i(\theta)=\frac{1}{2}(y_i-\theta^Tx_i)^2$ and $\psi(\theta)$ is a norm, $\rho$ is some predefined radius. We denote the dual norm of $\psi (\theta)$ as $\psi^*(\theta)$ and assume that each $f_i(\theta)$ is $L$ smooth.
	\item Non-convex $G(\theta)$:
	\begin{equation}\label{obj2}
	\begin{split}
	\underset{g_{\lambda} (\theta)\leq \rho}{\mbox{Minimize:}}\,\,  G(\theta) &\triangleq f(\theta)+ g_{\lambda,\mu}(\theta)\\
	&=\frac{1}{n}\sum_{i=1}^{n}f_i(\theta)+ g_{\lambda,\mu}(\theta),
	\end{split}
	\end{equation}
	where $f_i(\theta)$ is convex and $L$ smooth, $g_{\lambda,\mu}(\theta)$ is some non-convex regularizer, $g_\lambda(\theta)$ is close related to $g_{\lambda,\mu} (\theta)$ and we defer the formal definition to the section \ref{section:nonconvex_regularizer}. 
\end{itemize}

Such finite sum structure is common for machine learning problems particularly in the empirical risk minimization (\textit{ERM})  setting. To solve the above problem, the standard  Prox-full gradient (FG) method update $\theta^k$ iterative by
$$ \theta^{k+1}=prox_{\gamma\lambda \psi} (\theta^k-\gamma \nabla f(\theta^k)).$$
It is well known that FG enjoys fast linear convergence under smoothness and strong convexity assumption.  However this result may be less appealing when $n$ is large since the cost of calculation of full gradient scales with $n$. Stochastic gradient (SG) method remedies this issue but only possess the sub-linear convergence rate.

Recently, a set of stochastic algorithms including SVRG~\cite{johnson2013accelerating,xiao2014proximal}, SAGA~\cite{defazio2014saga}, SAG~\cite{schmidt2013minimizing} SDCA~ \cite{shalev2014accelerated} and many others \cite{harikandeh2015stopwasting,qu2015quartz,zhang2015stochastic}  have been proposed to exploit the finite sum structure and enjoy  linear rate convergence under smoothness and strong convexity assumption on $f_i(\theta)$. We study SAGA in this paper. From a high level, SAGA is a midpoint between SAG and SVRG, see the discuss in \citet{defazio2014saga} for more details. Different from SVRG, it is a fully incremental gradient method. Comparing with SAG, it uses an unbiased estimator of the gradient, which results in an easier proof among other things. In fact, to the best of our knowledge, the analysis of SAG has not yet been extended to proximal operator version. 

A second trendy topic in  optimization and statistical estimation is the study of 
non-convex problems, due to a vast array of applications such as SCAD \cite{fan2001variable}, MCP \cite{zhang2012general}, robust regression (Corrected Lasso \cite{loh2011high})  and  deep learning \cite{Goodfellow-et-al-2016}. Some previous work have established fast convergence  for  {\em batch gradient methods}   without assuming strong convexity  or even convexity: \citet{xiao2013proximal} proposed a homotopy method to solve Lasso with  RIP condition.      \citet{agarwal2010fast} analyzed the convergence rate of batched composite gradient method on several models, such as Lasso, logistic regression with $\ell_1$ regularization and noisy matrix decomposition,  and showed that the convergence is linear under mild conditions of the solution (sparse or low rank). \citet{loh2011high,loh2013regularized} extended the above work to the non-convex case. 

These two line of research thus motivate this work to investigate whether SAGA  enjoys the linear convergence rate without strong convexity or even in the non-convex problem. 
Specifically, we prove that under \textit{Restricted strong convexity} assumption, SAGA converges linearly up to the fundamental statistical precision of the model, which covers five statistical models we mentioned above but not limited to these. In a high level, it is a stochastic counterpart of the work in  \citet{loh2013regularized}, albeit with more involved analysis due to the stochastic nature of SAGA.

We list some notable  non-strongly convex and non-convex problems  in the following. Indeed, our work proves that SAGA converges linearly in all these models.
Note that the first three belong to the non-strongly convex category especially when $p>n$ and the last two are non-convex. 
\begin{enumerate}
	\item Lasso: $f_i(\theta)=\frac{1}{2} (\langle \theta, x_i \rangle-y_i)^2$ and $\psi(\theta)=\|\theta\|_1$.
	\item Group Lasso: $f_i(\theta)=\frac{1}{2} (\langle \theta, x_i \rangle-y_i)^2$,  $\psi(\theta)=\|\theta\|_{1,2}$.
	\item Logistic Regression with $l_1$ regularization: $f_i(\theta)=\log (1+\exp(-y_i \langle x_i,\theta \rangle)$ and $\psi(\theta)=\|\theta\|_1$.
	\item Corrected Lasso \cite{loh2011high}: $G(\theta)=\sum_{i=1}^{n}\frac{1}{2n} (\langle \theta, x_i \rangle-y_i)^2-\frac{1}{2}\theta^T\Sigma \theta+\lambda \|\theta\|_1,$ where $\Sigma$ is some positive definite matrix.
	\item Regression with SCAD  regularizer \cite{fan2001variable}:
	$G(\theta)=\sum_{i=1}^{n}\frac{1}{2n} (\langle \theta, x_i \rangle-y_i)^2+SCAD(\theta)$.
\end{enumerate}

Very recently, \citet{qu2016linear} explore the similar idea of us called restrict strong convexity condition (RSC) \cite{negahban2009unified} on SVRG and prove  that under this condition, a class of \textit{ERM} problem has the linear convergence even without strongly convex or even the convex assumption. From a high level perspective, our work can be thought as of similar spirit but for SAGA algorithm. We believe analyzing the SAGA algorithm is indeed important as SAGA enjoys certain advantage compared to SVRG.   As discussed above, SVRG is not a completely incremental algorithm since it need to calculate the full gradient in every epoch, while SAGA avoids the computation of the full gradient by keeping a table of gradient.    Moreover, although in general SAGA costs $O(np)$ storage (which is inferior to SVRG), in many scenarios the requirement of storage can be reduced to $O(n)$. For example, many loss function  $f_i$ take the form $f_i(\theta)=g_i(\theta^Tx_i)$ for a vector $x_i$ and since $x_i$ is a constant we just need to store the scalar $\nabla g_{i_k} (u_i^k)$ for $u_i^k=x_{i_{k}}^T\theta^k$ rather than full gradient. When this scenario is possible, SAGA can perform similarly or even better than SVRG. In addition, SVRG has an additional parameter besides step size to tune~--~the number of iteration $m$ per inner loop. To conclude, both SVRG and SAGA can be more suitable for some problems, and hence it is useful to understand the performance of SAGA for non-strongly convex or non-convex setups. At last, the proof steps are very different. In particular,  we define a Lyapunov function in SAGA and prove it converges geometrically until the optimality gap achieves the statistical tolerance, while \cite{qu2016linear} directly look at evolution of $G(\theta^k)$.

\subsection{Related work}

There are a plethora of work on the finite sum problem and we review those most closely related to ours. \citet{li2016stochastic} consider SVRG on a non-convex sparse linear regression setting different from ours, where $f_{i}$ is convex and the non-convexity comes from the hard-thresholding operator. We focus on a non-convex regularizer such as SCAD and corrected Lasso. In addition, we consider a unified framework on SAGA thus our work not only covers the linear sparse model but also the group sparsity and other model satisfying our assumptions.  \citet{karimi2016linear,reddi2016fast,hajinezhad2016nestt}  proved  global linear convergence of SVRG and SAGA on non-convex problems by revisiting the concept Polyak-\L{}ojasiewicz inequality or its equivalent idea such as error bound . We emphasize that our work looks at the problem from different perspective.  In particular, our theory asserts that the algorithm converges faster with sparser $\theta^*$,  while their results are independent of the sparsity $r$. Empirical observation seems to agree with our theorem.  Indeed, when $r$ is dense enough a phase transition from linear rate to sublinear rate occurs (also observed in \citet{qu2016linear}), which agrees with the prediction of our theorem. Furthermore, their work requires the  epigraph of $\psi(\theta)$ to be a polyhedral set which limits its applicability. For instance, the popular group Lasso  does not satisfy such an assumption. Other non-convex stochastic variance reduction works include \citet{shalev2016sdca,shamir2015fast} and~\citet{allen2016variance}: \citet{shalev2016sdca} considers the setting that $f(\theta)$ is strongly convex but each individual $f_i(\theta)$ is non-convex. \citet{shamir2015fast} discusses a projection version of non-convex SVRG and its specific application on PCA. \citet{allen2016variance} consider a general non-convex problem, which only  achieves a sublinear convergence rate.

\section{Preliminaries}

\subsection{Restricted Strong Convexity}
As  mentioned in the abstract, \textit{Restricted strong convexity} (RSC) is the key assumption underlying our results.  We therefore define  RSC formally. We say a function $f(\theta)$ satisfies   RSC w.r.t.\ to a norm $\psi(\theta)$ with parameter $(\sigma,\tau_\sigma)$ if the following holds.
\begin{equation}
\begin{split}
&f(\theta_2)-f(\theta_1)-\langle \nabla f(\theta_2),\theta_2-\theta_1 \rangle\\
\geq &\frac{\sigma}{2} \|\theta_2-\theta_1\|_2^2-\tau_\sigma \psi^2(\theta_2-\theta_1).
\end{split}
\end{equation}
We remark that we assume $f(\theta)=\frac{1}{n} \sum_{i=1}^{n}f_i(\theta)$ satisfies the RSC rather than individual loss function $f_i(\theta)$. Indeed, $f_i(\theta)$ does not satisfy RSC  in practice. Note that when $f(\theta)$ is  $\sigma-$ strongly convex, obviously we have $\tau_\sigma=0$. For more discussions on RSC, we refer reader to \citet{negahban2009unified}.

\subsection{ Assumptions for the Convex regularizer $\Psi(\theta)$}

\subsubsection{Decomposibility of $\Psi(\theta)$}
Given a pair of subspaces $ M\subseteq \bar{M} $ in $ \mathbb{R}^p$, the orthogonal complement of $\bar{M}$ is 
$$ \bar{M}^{\perp}=\{ v\in \mathbb{R}^p | \langle u,v \rangle=0 \text{~for all~} u\in \bar{M} \}.$$ $M$ is known as the model subspace, where $\bar{M}^{\perp}$ is called the {\em perturbation subspace}, representing the deviation from the model subspace. A  regularizer $\psi $ is {\em decomposable} w.r.t.\ $(M,\bar{M}^{\perp})$ if 
$$ \psi(\theta+\beta)=\psi(\theta)+\psi(\beta) $$ 
for all $ \theta\in M $ and $\beta\in \bar{M}^{\perp}.$
A concrete example is $\ell_1$ regularization for sparse vector supported on subset $S$. We define the subspace pairs with respect to the subset $S\subset\{1,...,p\}$, 
$ M(S)=\{\theta\in \mathbb{R}^p| \theta_j=0 \text{~for all~} j\notin S \} $ and $\bar{M}(S)=M(S).$ The decomposability is thus easy to verify. Other widely used examples include non-overlap group norms such as$\|\cdot\|_{1,2}$, and the nuclear norm $\|| \cdot\||_{*}$ \cite{negahban2009unified}. In the rest of the paper, we denote $\theta_M$ as the projection of $\theta$ on the subspace $M$.

\subsubsection{Subspace compatibility}
Given the regularizer $\psi(\cdot)$, the subspace compatibility $H(\bar{M})$ is given by
$$ H(\bar{M})=\sup_{\theta\in \bar{M} \backslash \{0\}} \frac{\psi(\theta)}{\|\theta\|_2}.$$

In other words, it is the Lipschitz constant of the regularizer restricted in $\bar{M}.$ For instance, in the above-mentioned sparse vector example with cardinality $r$, $H(\bar{M})=\sqrt{r}$. 

\subsection{Assumptions for the Nonconvex regularizer $g_{\lambda,\mu} (\theta)$}\label{section:nonconvex_regularizer}
In the non-convex case, we consider regularizers that  are separable across coordinates, i.e., $g_{\lambda,\mu}(\theta)=\sum_{j=1}^{p} \bar{g}_{\lambda,\mu} (\theta_j)$ . 
Besides the separability, we have additional assumptions on $g_{\lambda,\mu}(\cdot)$. For the univariate function $\bar{g}_{\lambda,\mu} (t)$, we assume
\begin{enumerate}
	\item  $\bar{g}_{\lambda,\mu}(\cdot)$ satisfies $\bar{g}_{\lambda,\mu}(0)=0$ and is symmetric around zero. That is, $\bar{g}_{\lambda,\mu}(t)=\bar{g}_{\lambda,\mu}(-t)$.
	\item On the nonnegative real line, $\bar{g}_{\lambda,\mu} $ is nondecreasing.
	\item For $t>0$,  $\frac{\bar{g}_{\lambda,\mu} (t)}{t}$ is nonincreasing in t.
	\item $\bar{g}_{\lambda,\mu}(t)$ is differentiable at all $t\neq0$ and subdifferentiable at $t=0$, with $\lim_{t\rightarrow0^+} g_{\lambda,\mu}'(t)=\lambda L_g$ for a constant $L_g$.
	\item  $ \bar{g}_{\lambda}(t):= (\bar{g}_{\lambda,\mu}(t)+\frac{\mu}{2}t^2)/\lambda$ is convex.	
\end{enumerate}

We provide two examples satisfying the above assumptions.

$\qquad(1)\quad \mathbf{SCAD_{\lambda,\zeta} (t) \triangleq }$ 
$$
\begin{cases}
\lambda |t|,   & \quad \text{for}~ |t|\leq \lambda,\\
-(t^2-2\zeta\lambda|t|+\lambda^2)/(2(\zeta-1)), & \quad \text{for} ~\lambda<|t|\leq \zeta\lambda,\\
(\zeta+1)\lambda^2/2, &\text{for} ~|t|>\zeta\lambda,
\end{cases}
$$
where $\zeta>2$ is a fixed parameter. It satisfies the assumption with $L_g=1$ and $\mu=\frac{1}{\zeta-1}$ \cite{loh2013regularized}.\\
$$(2)\quad \mathbf{ MCP_{\lambda,b}(t)}\triangleq \mbox{sign}(t)\lambda \int_{0}^{|t|} (1-\frac{z}{\lambda b})_{+} dz ,$$
where $b>0$ is a fixed parameter. MCP satisfies the assumption with $L_g=1$ and $\mu=\frac{1}{b}$~\cite{loh2013regularized}.

\begin{algorithm}[tb]
	\caption{SAGA}
	\label{alg:convex_SAGA}
	\begin{algorithmic}
		\STATE {\bfseries Input:}  Step size $\gamma$, number of iterations $K$, and smoothness parameters $L$. 
		\FOR{$k=1,...,K$}
		\STATE Pick a $j$ uniformly at random\\
		\STATE 1. Take $\phi_j^{k+1}=\theta^k$, and store $\nabla f_j (\phi_j^{k+1})$ in the table. All other entries in the table remain unchanged. \\
		\STATE 2. Update $\theta$ using $\nabla f_j(\phi_j^{k+1})$, $\nabla f_j(\phi_j^k)$ and the table average:
		$$ w^{k+1}=\theta^k-\gamma  [\nabla f_j(\phi_j^{k+1})-\nabla f_j(\phi_j^k)+\frac{1}{n} \sum_{i=1}^{n} \nabla f_i(\phi_i^k) ].$$
		$\theta^{k+1}=\arg \min_{\psi(\theta)\leq \rho} \frac{1}{2} \|\theta-w^{k+1}\|_2^2+\gamma\lambda \psi(\theta).$
		\ENDFOR
	\end{algorithmic}
\end{algorithm}

\subsection{Implementation of the algorithm}

For the convex $G(\theta)$ case, we directly apply the Algorithm \ref{alg:convex_SAGA}. As to the non-convex $G(\theta)$ case, we essentially solve the following equivalent problem 
$$ \underset{g_\lambda (\theta)\leq \rho}{\mbox{Minimize:}}\quad \left(f(\theta)-\frac{\mu}{2}\|\theta\|_2^2\right)+\lambda g_{\lambda} (\theta). $$
We define $F_i(\theta)=f_i(\theta)-\frac{\mu}{2}\|\theta\|_2^2$ and $F(\theta)=f(\theta)-\frac{\mu}{2} \|\theta\|_2^2$. To implement  Algorithm \ref{alg:convex_SAGA} on non-convex $G(\theta)$, we replace $f_i(\cdot)$ and $\psi(\cdot)$ in the algorithm by $F_i(\cdot)$ and $g_\lambda (\cdot)$. Remark that according to the assumptions on $g_{\lambda,\mu} (\cdot)$ in Section \ref{section:nonconvex_regularizer}, $g_\lambda(\cdot)$ is convex thus the proximal step is well-defined. The update rule of proximal operator on several $g_{\lambda,\mu}$ (such as SCAD) can be found in \cite{loh2013regularized}  .

\section{Main result}

In this section, we present the main theoretical results, and some corollaries that instantiate the main results in several well known statistical models.

\subsection{Convex $G(\theta)$}
We first present the results on convex $G(\theta)$. In particular, we prove a Lyapunov function  converges geometrically until $G(\theta^k)-G(\hat{\theta})$ achieves some tolerance. To this end, we first define the Lyapunov function 
\begin{equation*}
\begin{split}
T_{k}\triangleq \frac{1}{n} \sum_{i=1}^{n} \left( f_i(\phi_i^k)-f_i(\hat{\theta})- \langle \nabla f_i (\hat{\theta}),\phi_i^k-\hat{\theta} \rangle\right)\\
+(c+\alpha) \|\theta^k-\hat{\theta
}\|_2^2+b (G(\theta^{k})-G(\hat{\theta
})),
\end{split}
\end{equation*}
where $\hat{\theta}$ is the optimal solution of problem \eqref{obj1}, $c$, $\alpha$, $b$ are some positive constant will be specified later in the theorems. Notice our definition is a little different from the one in the original SAGA paper in \citet{defazio2014saga}. In particular, we have an additional term $G(\theta^k)-G(\hat{\theta})$ and choose different value of $c$ and $\alpha$, which helps us to utilize the idea of RSC.

We list some notations used in the following theorems and corollaries.

\begin{itemize}
	\item  $\theta^*$ is the unknown true parameter. $\hat{\theta}$ is the optimal solution of \eqref{obj1}.
	\item $\psi^*(\cdot)$ is the dual norm of $\psi(\cdot)$.
	\item Modified restricted strongly convex parameter:
	$$\bar{\sigma}=\sigma-64\tau_{\sigma} H^2(\bar{M}).$$
	\item Tolerance 
	$$\delta=24 \tau_\sigma\left( 8H(\bar{M})\|\hat{\theta}-\theta^*\|_2+8\psi(\theta^*_{M^\perp})  \right)^2 $$
	
\end{itemize}

\begin{theorem} \label{Theorem:convex}
	Assume each $f_i(\theta)$ is $L$ smooth and convex, $f(\theta)$ satisfies the RSC condition with parameter $(\sigma,\tau_{\sigma})$ and $\bar{\sigma}>0$, $\theta^*$ is feasible, the regularizer is decomposable w.r.t. $(M,\bar{M})$, if we choose the parameter $\lambda \geq \max (2\psi^* (\nabla f(\theta^*)), c_1\tau_\sigma \rho)$ where $c_1$ is some universal positive constant, then with $\gamma=\frac{1}{9L}$ , $c=\frac{9L}{n}$, $\alpha=\frac{c}{2}$, $b=2\alpha\gamma$, $\frac{1}{\kappa} =\min (\frac{\bar{\sigma}}{14L}, \frac{1}{9n})$, we have 	
	$$\mathbb{E} T_k\leq \left(1-\frac{1}{\kappa}\right)^k T_0,$$
	until $G(\theta^k)-G(\hat{\theta})\leq \delta,$ where the expectation is for the randomness of sampling of $j$ in the algorithm.	
	
\end{theorem}
Some remarks are in order.
\begin{itemize}
	\item The requirement $\bar{\sigma}>0$ is easy to satisfy in some popular statistical models. Take Lasso as an example, where $\tau_\sigma=c_2\frac{\log p}{n}$, $c_2$ are some  positive constant, $H^2 (\bar{M})=r$. Thus $\bar{\sigma}=\sigma-64c_2 \frac{r\log p}{n}$. Hence when $ 64c_2 \frac{r\log p}{n} \leq \frac{1}{2}\sigma $, we have $\bar{\sigma}\geq\frac{\sigma}{2}$.
	\item Since $\frac{1}{\kappa}$ depends on $\bar{\sigma}/L$, the convergence rate is indeed affected by the sparsity $r$ (Lasso for example )as we mentioned in the introduction. Particularly, sparser $r$ leads to larger $\bar{\sigma}$ and faster convergence rate.
	\item In some models, we can choose the subspace pair such that $\theta^*\in M$, thus the tolerance $\delta$ is simplified to $\delta=c_3 \tau_\sigma H^2(\bar{M})\|\hat{\theta}-\theta^*\|_2^2$. In Lasso as we mentioned above, $\delta=c_3\frac{r\log p}{n}\|\hat{\theta}-\theta^*\|_2^2,$ i.e., the tolerance is dominated by the statistical error $\|\hat{\theta}-\theta^*\|_2^2.$
	\item When $G(\theta^k)-G(\hat{\theta})\leq \delta$, use modified restricted strong convexity (Lemma \ref{lemma.RSC_cone} in the appendix), it is easy to derive $ \|\theta^k-\hat{\theta}\|_2^2\leq \frac{c_4 \delta}{\bar{\sigma}}.$
\end{itemize}

Combine all remarks together, the theorem says the Lyapunov function decreases geometrically until $G(\theta^k)-G(\hat{\theta})$ achieves the tolerance $\delta$. This tolerance is dominated by the statistical error $\|\theta^k-\hat{\theta}\|_2^2$, thus can be ignored from the statistical perspective.

\subsubsection{Sparse linear regression}

The first model we consider is Lasso, where $f_i(\theta)=\frac{1}{2} (\langle \theta, x_i \rangle-y_i)^2$ and $\psi(\theta)=\|\theta\|_1$. More concretely, we consider a model where each data point $x_i$ is i.i.d.\  sampled from a zero-mean normal distribution, i.e., $x_i\sim N(0,\Sigma)$. We denote the data matrix by $X\in \mathbb{R}^{n\times p}$   and   the smallest eigenvalue of $\Sigma$ by $\sigma_{\min} (\Sigma)$, and  let $\nu(\Sigma)\triangleq \max_{i=1,...,p}\Sigma_{ii}$.  The observation is generated by $y_i=x_i^T\theta^*+\xi_i$, where $\xi_i$ is a zero mean sub-Gaussian noise with variance $\varsigma^2$. We use $X_j\in \mathbb{R}^n$ to denote $j$-th column of $X$. Without loss of generality, we require $X$ to be column-normalized, i.e., $\frac{\|X_j\|_2}{\sqrt{n}}\leq 1  \quad \mbox{for all} \quad j=1,2,...,p$. Here, the constant $1$ is chosen arbitrarily to simplify the exposition, as we can always rescale the data.

\begin{corollary}\label{cor.lasso}
	Assume $\theta^*$ is the true parameter supported on a subset with cardinality at most $r$, and we choose $\lambda$ such that $\lambda \geq \max (6\varsigma \sqrt{\frac{\log p}{n}}, c_1 \rho \nu (\Sigma)\frac{\log p}{n})$, $\bar{\sigma}=\frac{1}{2} \sigma_{\min}(\Sigma)-c_2 \nu (\Sigma) \frac{r\log p}{n} $, then with $\gamma=\frac{1}{9L}$, $c=\frac{9L}{n}$, $\alpha=\frac{c}{2}$, $b=2\alpha \gamma$, $\frac{1}{\kappa}=\min (\frac{\bar{\sigma}}{14L}, \frac{1}{9n})$, we have
	$$\mathbb{E} T_k\leq \left(1-\frac{1}{\kappa}\right)^k T_0, $$
	with probability at least $1-\exp (-3 \log p)-\exp(-c_3 n)$, until $G(\theta^k)-G(\hat{\theta})\leq \delta,$ where $\delta= c_4 \nu (\Sigma) \frac{r\log p}{n} \|\hat{\theta}-\theta^*\|_2^2.$ Here $c_1,c_2,c_3,c_4$ are some universal positive constants.
\end{corollary}
We offer some discussions on this corollary.
\begin{itemize}
	\item The requirement of $\lambda \geq 6\varsigma \sqrt{\frac{\log p}{n}}$ is known to play an important role in proving bounds on the statistical error of Lasso, see \citet{negahban2009unified} and reference therein for further details.
	\item The requirement $\lambda\geq c_1\rho \nu (\Sigma)\frac{\log p}{n}$ is to guarantee the fast global convergence of the algorithm, which is similar to the requirement in its batch counterpart \cite{agarwal2010fast}. 
	\item When $r$ is small and $n$ is large, which is necessary for  statistical consistency of Lasso, we obtain $\bar{\sigma}>0$, which guarantees the existences of $\kappa$. Under this condition we have $\delta=c_4\nu(\Sigma) \frac{r\log p}{n} \|\hat{\theta}-\theta^*\|_2^2 $, which is  dominated by $\|\hat{\theta}-\theta^*\|_2^2.$
\end{itemize}

\subsubsection{Group Sparse model}
The group sparsity model aims to find a regressors such that predefined groups of covariates  are selected  into or out of a model {\em together}. The most commonly used regularization to encourage group sparsity is $\|\cdot\|_{1,2} $. Formally, we are given a class of disjoint groups of the features, i.e., $\mathcal{G}=\{G_1,G_2,...,G_{N_{\mathcal{G}}}\}$ and $G_i \cap G_j=\emptyset$. The regularization term is $\|\theta\|_{\mathcal{G},q}\triangleq \sum_{g=1}^{N_{g}} \|\theta_g\|_q $. When $q=2$, it reduces to the popular  {\em group Lasso}~\cite{yuan2006model} while another widely used case is $q=\infty$ \cite{turlach2005simultaneous,quattoni2009efficient}.

We now define the subspace pair $(M,\bar{M})$ in the group sparsity model. For a subset $S_{\mathcal{G}}\subseteq \{1,...,N_{\mathcal{G}} \}$ with cardinality $s_{\mathcal{G}}= |S_{\mathcal{G}}|$, we define the subspace 
$$ M (S_{\mathcal{G}})=\{\theta|\theta_{G_{i}}=0 ~~ \text{for all } ~~ i\notin S_{\mathcal{G}} \} , $$and $M=\bar{M}$.
The orthogonal complement is 
$$ \bar{M}^{\perp}(S_{\mathcal{G}}) = \{\theta|\theta_{G_{i}}=0 ~~ \text{for all } ~~ i\in S_{\mathcal{G}} \}. $$ We can easily verify that 
$$ \|\alpha +\beta\|_{\mathcal{G},q}=\|\alpha\|_{\mathcal{G},q}+\|\beta\|_{\mathcal{G},q}, $$
for any $\alpha\in M(S_{\mathcal{G}}) $ and $\beta\in \bar{M}^{\perp} (S_{\mathcal{G}})$.

We mainly focus on the discussion on the case $q=2$, i.e., group Lasso. We require the following condition, which  generalizes the column normalization condition in the Lasso case. Given a group $G$ of size $m$ and $X_G\in \mathbb{R}^{n\times m}$ , the associated operator norm $||| X_{G_i}|||_{q \rightarrow 2} \triangleq \max_{\|\theta\|_q=1}\|X_G\theta\|_2$ satisfies
$$ \frac{||| X_{G_i}|||_{q\rightarrow 2}}{\sqrt{n}}\leq 1 ~~\text{for all} ~~ i=1,2,...,N_{\mathcal{G}}.$$ 
The condition reduces to the column normalized condition when each group contains only one feature (i.e., Lasso).

In the following corollary, we use $q=2$, i.e., group Lasso, as an example. We assume the observation $y_i$ is generated by $y_i=x_i^T \theta^*+\xi_i $, where $x_i\sim N(0,\Sigma)$, and $\xi_i\sim N(0,\varsigma^2)$.

\begin{corollary}(Group Lasso)\label{cor.group_lasso}
	Assume $\theta\in \mathbb{R}^p$ and each group has $m$ parameters, i.e., $p=m N_{\mathcal{G}}$. Denote the cardinality of non-zero group by $s_{\mathcal{G}}$, and we choose parameter $\lambda$ such that  	
	$$\lambda\geq \max \big(4\varsigma (\sqrt{\frac{m}{n}}+\sqrt{\frac{\log N_{\mathcal{G}}}{n}}), c_1\rho \sigma_2(\Sigma) (\sqrt{\frac{m}{n}}+\sqrt{\frac{3 \log N_{\mathcal{G}}}{n}} )^2\big),$$ then with $\gamma=\frac{1}{9L}$, $c=\frac{9L}{n}$, $\alpha=\frac{c}{2}$, $b=2\alpha \gamma$, $\frac{1}{\kappa}=\min (\frac{\bar{\sigma}}{14L}, \frac{1}{9n})$, we have
	$$\mathbb{E} T_k\leq (1-\frac{1}{\kappa})^k T_0 $$
	with probability at least $1-2\exp (-2 \log N_{\mathcal{G}})-c_2 \exp(-c_3n) $   , until $G(\theta^k)-G(\hat{\theta})\leq \delta,$  where $\bar{\sigma}=\sigma_1(\Sigma)-c_2\sigma_2(\Sigma)s_{\mathcal{G}} (\sqrt{\frac{m}{n}}+\sqrt{\frac{3 \log N_{\mathcal{G}}}{n}} )^2$, $\sigma_1 (\Sigma)$ and $\sigma_2 (\Sigma)$ are   positive constant  depending only on $\Sigma$,  $ \delta=c_4\sigma_2 (\Sigma)s_{\mathcal{G}}\big(\sqrt{\frac{m}{n}}+\sqrt{\frac{3 \log N_{\mathcal{G}}}{n}} \big)^2 \|\hat{\theta}-\theta^*\|_2^2$. 	$c_1,c_2,c_3,c_4$ are some universal positive constants.
\end{corollary}

We offer some discussions to put above corollary into context.

\begin{itemize}
	\item To satisfy the requirement of $\bar{\sigma}>0$, it suffices to have 
	$ s_{\mathcal{G}} (\sqrt{\frac{m}{n}} + \sqrt{\frac{3 \log N_{\mathcal{G}}}{n}})^2=o(1) $. It is also the condition to guarantee the statistical consistency of group Lasso \cite{negahban2009unified}.
	
	\item  $s_{\mathcal{G}}$ and $m$ affect the speed of the convergence, in particular, smaller $m$ and $s_{G}$ leads to larger $\bar{\sigma}$ and thus $\bar{\sigma}/L$.
	\item The requirement of $\lambda$ is similar to the batch gradient method in \cite{agarwal2010fast}.
\end{itemize}

\subsection{Non-convex $G(\theta)$}
The definition of Lyapunov function in the non-convex case is same with the convex one, i.e.,
\begin{equation*}
\begin{split}
T_{k}\triangleq \frac{1}{n} \sum_{i=1}^{n} \left( f_i(\phi_i^k)-f_i(\hat{\theta})- \langle \nabla f_i (\hat{\theta}),\phi_i^k-\hat{\theta} \rangle\right)\\
+(c+\alpha) \|\theta^k-\hat{\theta
}\|_2^2+b (G(\theta^{k})-G(\hat{\theta
})).
\end{split}
\end{equation*}
Note that $\hat{\theta}$ is the global optimum of problem \eqref{obj2} and $f_i(\cdot)$ is convex, thus $T_k$ is always positive. In the non-convex case, we require $ f(\theta)$ satisfy the RSC condition with parameter $(\sigma,\tau\frac{\log p}{n})$, where $\tau$ is some positive constant.

We list some notations used in the following theorem and corollaries of it.
\begin{itemize}
	\item $\hat{\theta}$ is the  global optimum of problem \eqref{obj2}, and  $\theta^*$ is the unknown true parameter with cardinality $r$.
	\item Modified restricted strongly convex parameter:
	$$\bar{\sigma}=\sigma-64r\tau\frac{ \log p}{n}-\mu.$$ Recall  $\mu$ is defined in section \ref{section:nonconvex_regularizer} and represent the degree of non-convexity. 
	\item Tolerance $\delta=c_1r\tau\frac{\log p}{n}\|\hat{\theta}-\theta^*\|_2^2$, where $c_1$ is some universal positive constant.
\end{itemize}

\begin{theorem}\label{Theorem:non-convex}
	Suppose $\theta^*$ is $r$ sparse, $\hat{\theta}$ is the global optimum of Problem \eqref{obj2}, each $f_i(\theta)$ is L smooth and convex, $f(\theta)$ satisfies the RSC condition with $(\sigma, \tau \frac{\log p}{n})$, $\bar{\sigma}>3\mu$, $L>3\mu$, $g_{\lambda,\mu}$ satisfies the assumption in Section~\ref{section:nonconvex_regularizer},  and $\lambda L_g\geq \max \{ c_1\rho \tau \frac{\log p}{n},  4 \|\nabla f(\theta^*)\|_\infty \} $, where $c_1$ is some positive constant, then with $\gamma=\frac{1}{24L}$, $c=\frac{24L}{n}$, $\alpha=\frac{c}{2}$, $b=2\alpha \gamma$, $\frac{1}{\kappa}=\frac{1}{24}\min (\frac{2\bar{\sigma}}{5L}, \frac{1}{n})$, we have
	$$\mathbb{E} T_k\leq \left(1-\frac{1}{\kappa}\right)^k T_0, $$
	until $G(\theta^k)-G(\hat{\theta})\leq \delta,$ where the expectation is for the randomness of sampling of $j$ in the algorithm.	
\end{theorem}

\begin{itemize}
	\item Notice that we require $\bar{\sigma}>3\mu$, that is $\sigma-64r\tau\frac{ \log p}{n}-4\mu>0$. Thus to satisfy this requirement, the non-convex parameter $\mu$ can not be large.
	\item The tolerance $\delta=c_2 r\tau \frac{\log p}{n} \|\hat{\theta}-\theta^*\|_2^2$ is dominated by the statistical error $\|\hat{\theta}-\theta^*\|_2^2 $, when the model is sparse ($r$ is small ) and $n$ is large.
	
	\item When $G(\theta^k)-G(\hat{\theta})\leq \delta$,  using the modified restricted strong convexity on non-convex $G(\theta)$ (Lemma \ref{lemma.non_convex_RSC_cone} in the appendix), we obtain $ \|\theta^k-\hat{\theta}\|_2^2\leq c_3 \frac{\delta}{\bar{\sigma}}.$
	\item The requirement of $\lambda$ is similar to the batched gradient algorithm \cite{loh2013regularized}.
\end{itemize}

Again, the theorem says the Lyapunov function decreases geometrically until $G(\theta^k)-G(\hat{\theta})$ achieves the tolerance $\delta$ and this tolerance can be ignored from the statistical perspective.

\subsubsection{Linear regression with SCAD regularization}
The first non-convex model we considered is linear regression with SCAD regularization. The loss function is $f_i(w)=\frac{1}{2} (y-\langle \theta,x_i \rangle)^2$, and $g_{\lambda,\mu} (\cdot)$ is $SCAD (\cdot)$ with parameter $\lambda $ and $\zeta$. The data $(x_i,y_i)$ are generated in the similar way as that in Lasso case.

\begin{corollary}(Linear regression with SCAD regularization)\label{cor.SCAD}
	Suppose $\theta^*$ is the true parameter supported on a subset with cardinality at most $r$, $\hat{\theta}$ is the global optimum, $\bar{\sigma}\geq \frac{3}{\zeta-1}$, $L>\frac{3}{\zeta-1}$ and we choose $\lambda$ such that $\lambda \geq \max \{c_1\rho\nu(\Sigma) \frac{\log p}{n},12\varsigma \sqrt{\frac{\log p}{n}} \} $  then with $\gamma=\frac{1}{24L}$, $c=\frac{24L}{n}$, $\alpha=\frac{c}{2}$, $b=2\alpha \gamma$, $\frac{1}{\kappa}=\frac{1}{24}\min (\frac{2\bar{\sigma}}{5L}, \frac{1}{n})$, we have
	$$\mathbb{E} T_k\leq \left(1-\frac{1}{\kappa}\right)^k T_0, $$
	with probability at least $1-\exp (-3 \log p)-\exp(-c_2 n)$, until $G(\theta^k)-G(\hat{\theta})\leq \delta,$ where  $\bar{\sigma}=\frac{1}{2} \sigma_{\min}(\Sigma)-c_3 \nu (\Sigma) \frac{r\log p}{n}-\frac{1}{\zeta-1}$, $\delta=c_4\nu (\Sigma) \frac{r\log p}{n} \|\hat{\theta}-\theta^*\|_2^2.$ Here $c_1,c_2,c_3,c_4$ are some universal positive constants.
\end{corollary}
We remark that to satisfy the requirement $\bar{\sigma}\geq \frac{3}{\zeta-1}$, we need the  non-convex parameter $\mu=\frac{1}{\zeta-1}$  to be small, the model   sparse (r is small) and the number of  sample $n$ large.

\subsubsection{Linear regression with noisy covariates }
The {\em corrected Lasso} is proposed by \citet{loh2011high}. Suppose data are generated according to a   linear model $y_i=x_i^T\theta^*+\xi_i,$ where $\xi_i$ is a random zero-mean sub-Gaussian noise with variance $\varsigma^2 .$ The observation $z_i$ of $x_i$ is corrupted by addictive noise, in particular,   $z_i=x_i+w_i$, where $w_i\in \mathbb{R}^p$ is a random vector independent of $x_i$,  with zero-mean and known covariance matrix $\Sigma_w$. Define $\hat{\Gamma}=\frac{Z^TZ}{n}-\Sigma_w$ and $\hat{\gamma}=\frac{Z^Ty}{n}$. Our goal is to estimate $\theta^*$ based on $y_i$ and $z_i$ (but not $x_i$ which is not observable), and the corrected Lasso proposes to solve the following: 
$$\hat{\theta}\in \arg \min_{\|\theta\|_1\leq \rho} \frac{1}{2} \theta^T \hat{\Gamma} \theta-\hat{\gamma} \theta+ \lambda \|\theta\|_1. $$
Equivalently, it solves 
$$ \min_{\|\theta\|_1\leq \rho}\frac{1}{2n}\sum_{i=1}^{n} (y_i-\theta^Tz_i)^2-\frac{1}{2}\theta^T\Sigma_w \theta +\lambda \|\theta\|_1 .$$
Notice that due to the term $-\frac{1}{2}\theta^T\Sigma_w \theta $, the optimization problem is non-convex. 
\subsection{Corrected Lasso}

We consider a model where each data point $x_i$ is i.i.d.\  sampled from a zero-mean normal distribution, i.e., $x_i\sim N(0,\Sigma)$. We denote the data matrix by $X\in \mathbb{R}^{n\times p}$ , the smallest eigenvalue of $\Sigma$ by $\sigma_{\min} (\Sigma)$ and the largest eigenvalue by $\sigma_{\max} (\Sigma)$  and  let $\nu(\Sigma)\triangleq \max_{i=1,...,p}\Sigma_{ii}$.   We observe $z_i$ which is $x_i$ corrupted by addictive noise, i.e.,  $z_i=x_i+w_i$, where $w_i\in \mathbb{R}^p$ is a random vector independent of $x_i$, with zero-mean and known covariance matrix $\Sigma_w$.

\begin{corollary}(Corrected Lasso)\label{cor.corrected_lasso}
	Suppose we are given i.i.d. observations $\{(z_i,y_i) \}$ from the linear model with additive noise, $\theta^*$ is $r$ sparse and $\Sigma_w=\gamma_w I$, $\bar{\sigma}>3\gamma_w$, $L>3\gamma_w$ where $\bar{\sigma}=\frac{1}{2} \sigma_{\min}(\Sigma)-c_1 \sigma_{\min}(\Sigma)\max \left( (\frac{\sigma_{\max} (\Sigma)+\gamma_w}{\sigma_{\min}(\Sigma)})^2,1 \right)  \frac{r\log p}{n}-\gamma_w$. Let $\hat{\theta}$ be the global optimum. We choose $\lambda \geq \max\{ c_2\rho\frac{\log p }{n} , c_3 \varphi \sqrt{\frac{\log p}{n}}\} $ where $\varphi=(\sqrt{\sigma_{\max} (\Sigma)}+\sqrt{\gamma_w})(\varsigma+\sqrt{\gamma_w}\|\theta^*\|_2)$, then with $\gamma=\frac{1}{24L}$, $c=\frac{24L}{n}$, $\alpha=\frac{c}{2}$, $b=2\alpha \gamma$, $\frac{1}{\kappa}=\frac{1}{24}\min (\frac{2\bar{\sigma}}{5L}, \frac{1}{n})$, we have
	$$\mathbb{E} T_k\leq (1-\frac{1}{\kappa})^k T_0, $$
	with high probability at least $1-c_4 \exp \left(-c_5 n\min \big( \frac{\sigma^2_{\min} (\Sigma)}{( \sigma_{\max}(\Sigma)+\gamma_w)^2},1 \big)    \right)-\exp(-c_6 \log p)$  until $G(\theta^k)-G(\hat{\theta})\leq \delta,$ where $\delta=c_7 \sigma_{\min}(\Sigma)\max \left( (\frac{\sigma_{\max} (\Sigma)+\gamma_w}{\sigma_{\min}(\Sigma)})^2,1 \right)\frac{r\log p}{n} \|\hat{\theta}-\theta^*\|_2^2. $ $c_1$ to $c_7$ are some universal positive constants.
\end{corollary}

Some remarks are listed below.

\begin{itemize}
	\item The result can be easily extended to more general $\Sigma_w\preceq \gamma_w I.$
	\item To satisfy the requirement $\bar{\sigma}>3\gamma_w$, we need $$\gamma\leq \frac{1}{4} ( \frac{1}{2} \sigma_{\min}(\Sigma)-c_1   \sigma_{\min}(\Sigma)\max \left( (\frac{\sigma_{\max} (\Sigma)+\gamma_w}{\sigma_{\min}(\Sigma)})^2,1 \right)  \frac{r\log p}{n} ).$$ Similar requirement is needed in the batch gradient method \cite{loh2013regularized}.
	\item The requirement of $\lambda$ is similar to that in batch gradient method \cite{loh2013regularized}.
	
\end{itemize}

\subsection{Extension to Generalized linear model }
The results  on Lasso and group Lasso are readily extended to generalized linear models, where we consider the model
$$ \hat{\theta}=\arg\min_{\theta\in \Omega'}\{ \frac{1}{n}\sum_{i=1}^{n} {\Phi(\theta,x_i)-y_i\langle \theta,x_i \rangle}+\lambda \|\theta\|_1   \}, $$
with $\Omega'=\Omega\cap \mathbb{B}_2 (R)$ and $\Omega=\{ \theta| \|\theta\|_1\leq \rho \}$, where $R$ is a universal constant \cite{loh2013regularized}. This requirement is essential, for instance  for the logistic function , the Hessian function $\Phi''(t)=\frac{\exp(t)}{(1+\exp(t))^2}$ approached to zero as its argument diverges. 
Notice that when $\Phi(t)={t^2}/{2}$, the problem reduces to Lasso. The RSC condition admit the form 
$$ \frac{1}{n} \sum_{i=1}^{n} \Phi''(\langle \theta_t,x_i \rangle ) (\langle x_i,\theta-\theta' \rangle )^2\geq \frac{\sigma}{2}\|\theta-\theta'\|_2^2-\tau_\sigma \|\theta-\theta'\|_1, \mbox{for all} \quad \theta,\theta' \in \Omega'$$
For a board class of log-linear models, the RSC condition holds with $\tau_\sigma=c\frac{\log p}{n}$. Therefore, we obtain same results as those of Lasso, modulus change of  constants. For more details of  RSC conditions in generalized linear model, we refer the readers to \cite{negahban2009unified}.

\section{ Empirical Result}
We report the experimental results in this section to validate our theorem that SAGA can enjoys the linear convergence rate without strong convexity or even without convexity. We did experiment both on synthetic and real datasets and compare SAGA with several candidate algorithms. The experiment setup is similar to \cite{qu2016linear}. Due to space constraints, some addition simulation results are presented in the appendix. The algorithms tested  are Prox-SVRG \cite{xiao2014proximal}, Prox-SAG  which is a proximal version of the algorithm in \citet{schmidt2013minimizing}, proximal stochastic gradient (Prox-SGD), regularized dual averaging method (RDA) \cite{xiao2010dual} and the proximal full gradient method (Prox-GD) \cite{nesterov2013introductory}. For the algorithms with a constant learning rate (i.e., SAGA,Prox-SAG, Prox-SVRG, Prox-GD), we tune the learning rate from an exponential grid  $\{ 2, \frac{2}{2^1},...,\frac{2}{2^{12}} \}$ and chose the one with best performance. Below are some remarks on the candidate algorithms.

\begin{itemize}
	\item The linear convergence of SVRG in our setting has been proved in \cite{qu2016linear}. 
	\item We adapt SAG to its Prox version. To the best of our knowledge, the convergence of  Prox-SAG has not been established. In addition, it is not known whether the Prox-SAG converges or not although it works well in the experiment.
	\item	The step size in Prox-SGD is $\eta_k=\eta_0/\sqrt{k}$. The step size for RDA is $\beta_k=\beta_0 \sqrt{k}$ suggested in \cite{xiao2010dual}. $\beta_0$ and $\eta_0$ are chosen from exponential grid ( with power of 10) with the best performance.
\end{itemize}

\subsection{Synthetic data}

We report the experimental result on Lasso,Group Lasso, Linear regression with SCAD regularization and Corrected Lasso. 
\subsubsection{Lasso}

The feature vector $x_i\in \mathbb{R}^{p}$ are drawn independently from $N(0,\Sigma)$, where we set $\Sigma_{ii}=1, ~~\text{for}~~ i=1,...,p$ and $\Sigma_{ij}=b, ~~\text{for}~~ i\neq j$.  The responds $y_i$ is generated as follows: $ y_i=x_i^T \theta^*+\xi_i$, and  $\theta^*\in \mathbb{R}^p$ is a sparse vector with cardinality $r$,  where the non-zero entries are $\pm 1 $ drawn from the Bernoulli distribution with probability $0.5$. The noise $\xi_i$ follows the standard normal distribution. The parameter of regularizer is set to be $\lambda=0.05$. We set $p=5000$, $n=2500$ and vary the value on $r$ and $b$. The results are shown in  Figure \ref{Fig:lasso}.

\begin{figure}
	\begin{subfigure}[b]{0.23\textwidth}
		\includegraphics[width=\textwidth]{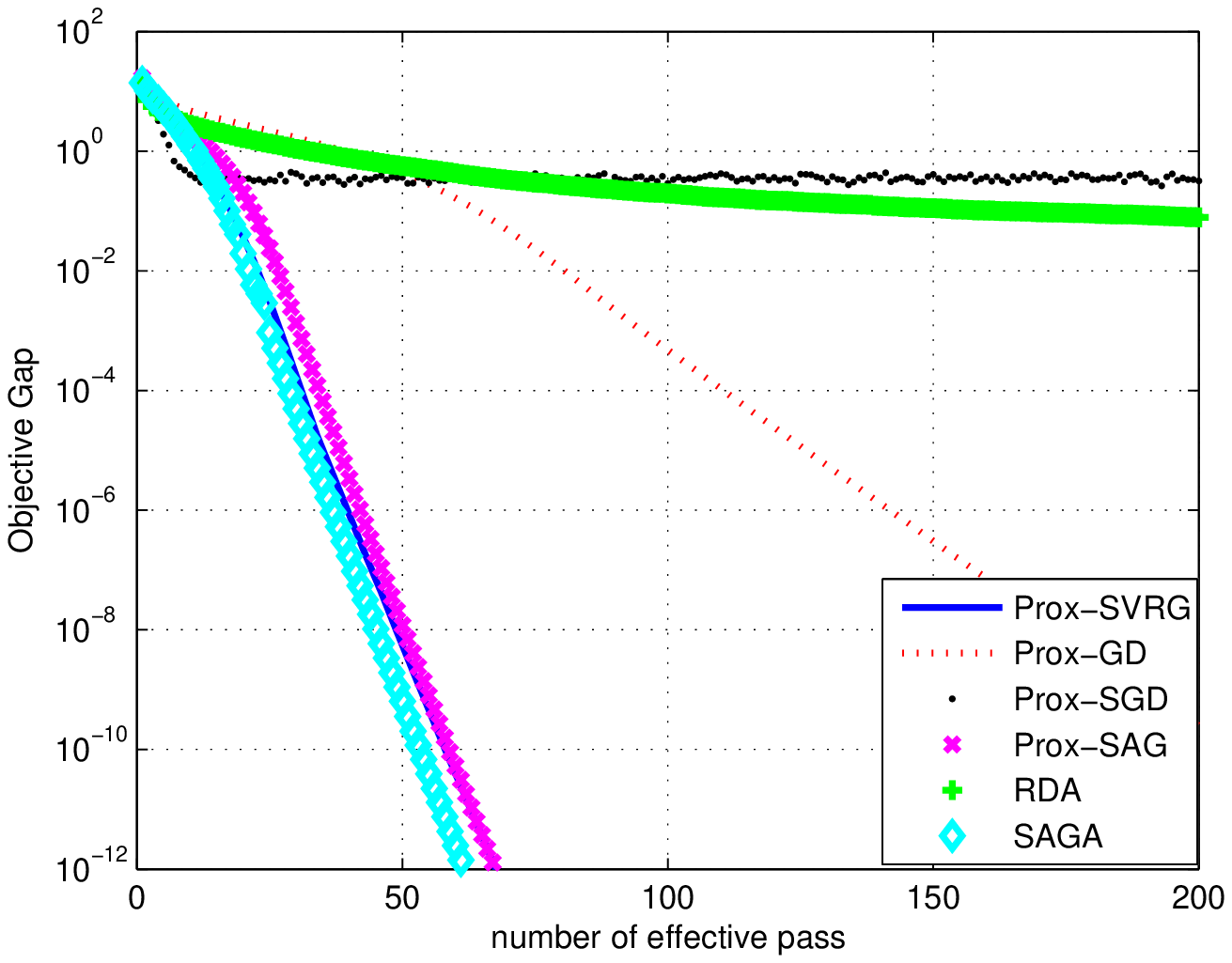}
		\caption{r=50, b=0}
	\end{subfigure}
	\begin{subfigure}[b]{0.23\textwidth}
		\includegraphics[width=\textwidth]{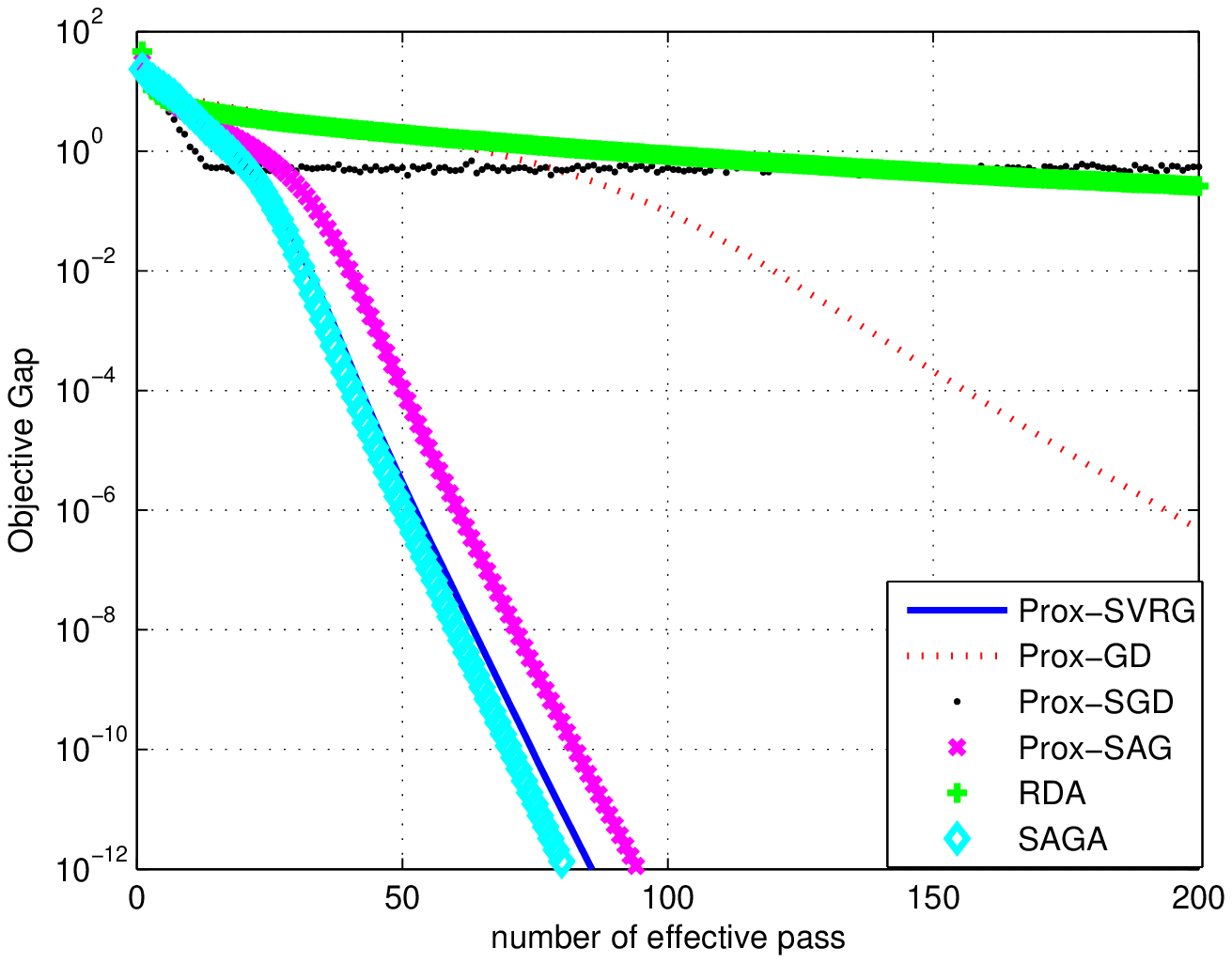}
		\caption{r=100,b=0}
	\end{subfigure}
	\begin{subfigure}[b]{0.23\textwidth}
		\includegraphics[width=\textwidth]{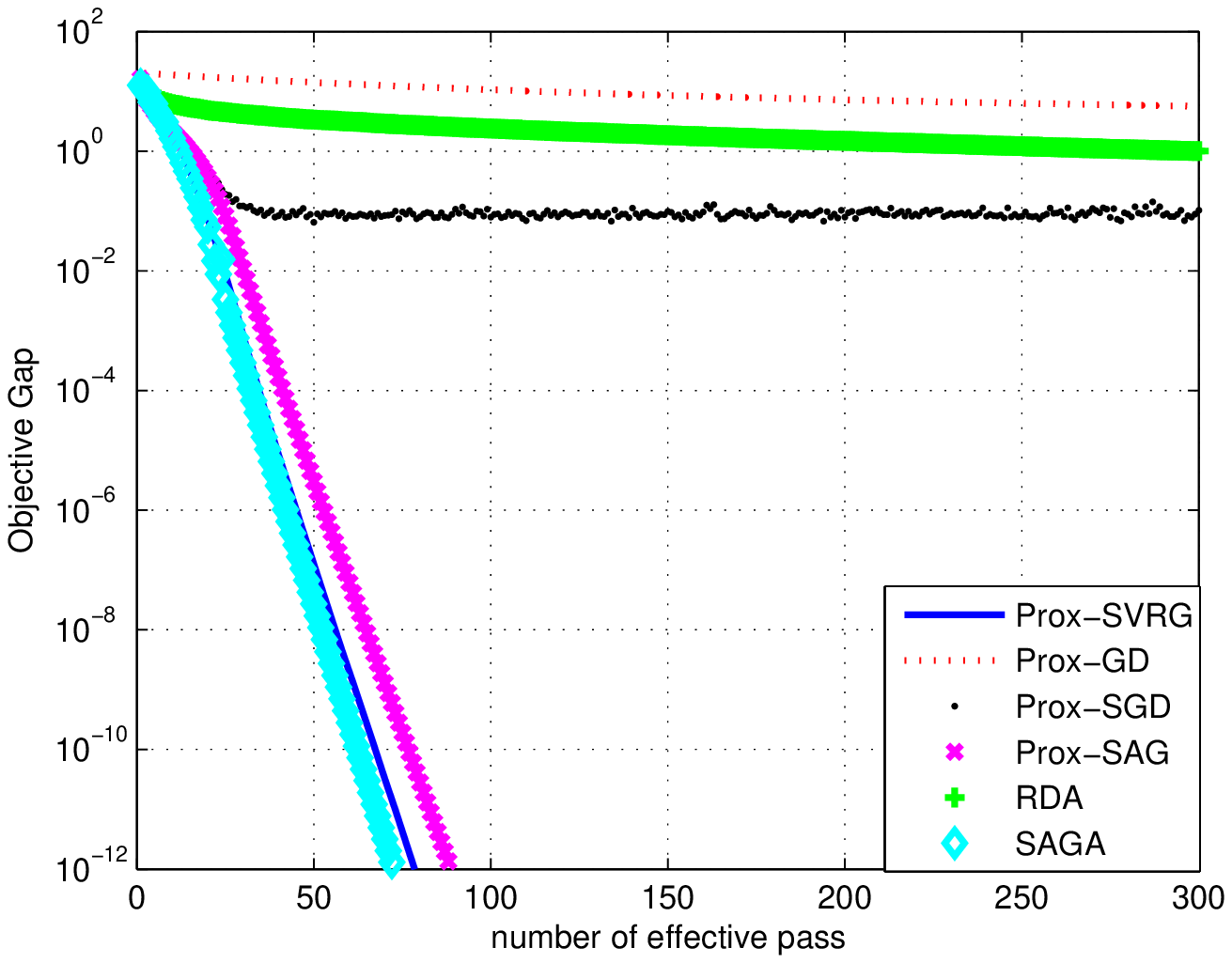}
		\caption{r=50,b=0.1}
	\end{subfigure}\quad
	\begin{subfigure}[b]{0.23\textwidth}
		\includegraphics[width=\textwidth]{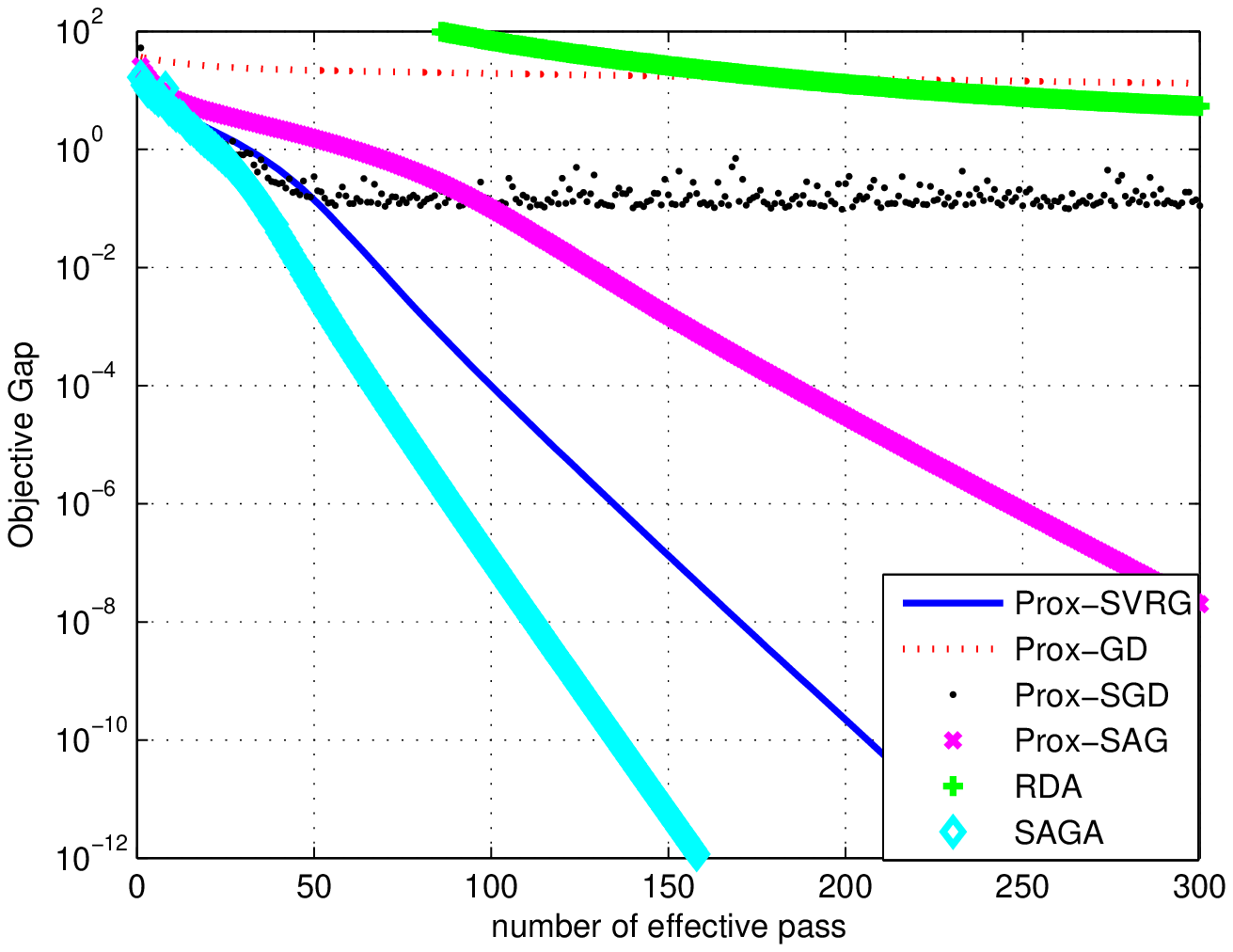}
		\caption{r=100,b=0.4}
	\end{subfigure}
	\caption{Comparison between six algorithms on Lasso. The x-axis is the number of passes over the dataset, and the y-axis is the objective gap $G(\theta^k)-G(\hat{\theta})$ with a log scale. } \label{Fig:lasso}
\end{figure}

Figure \ref{Fig:lasso} demonstrates that SAGA, Prox-SVRG and Prox-SAG enjoy a linear convergence rate in all settings. In the most challenging setup ($r=100,b=0.4$), SAGA outperforms Prox-SVRG and Prox-SAG. The batched method, Prox-GD converges linearly when $b=0$ and does not work well when $b=0.1$ and $b=0.4$. It is possibly because the condition number is large when $b\neq 0.$  We also observe that SAGA with sparser $r$ converges faster, which matches our Theorem \ref{Theorem:convex}. As we discussed in the remarks of Theorem \ref{Theorem:convex}   , $\frac{1}{\kappa}$ depends on $\bar{\sigma}/L$ and smaller $r$ cause larger $\bar{\sigma}$ thus faster convergence rate.

\subsubsection{Group Lasso}

We generate the observation $ y_i=x_i^T \theta^*+\xi_i$ with the feature vectors independently sampled from $N(0,\Sigma)$, where $\Sigma_{ii}=1$ and $\Sigma_{ij}=b, i\neq j$. The cardinality of non-zero group is $s_{\mathcal{G}}$, and the non-zero entries are sampled uniformly from $[-1, 1]$.  We vary the values of $b$, group size $m$ and group sparsity $s_{\mathcal{G}}$ and report the results in Figure \ref{Fig:gp_lasso} .
\begin{figure}
	\begin{subfigure}[b]{0.23\textwidth}
		\includegraphics[width=\textwidth]{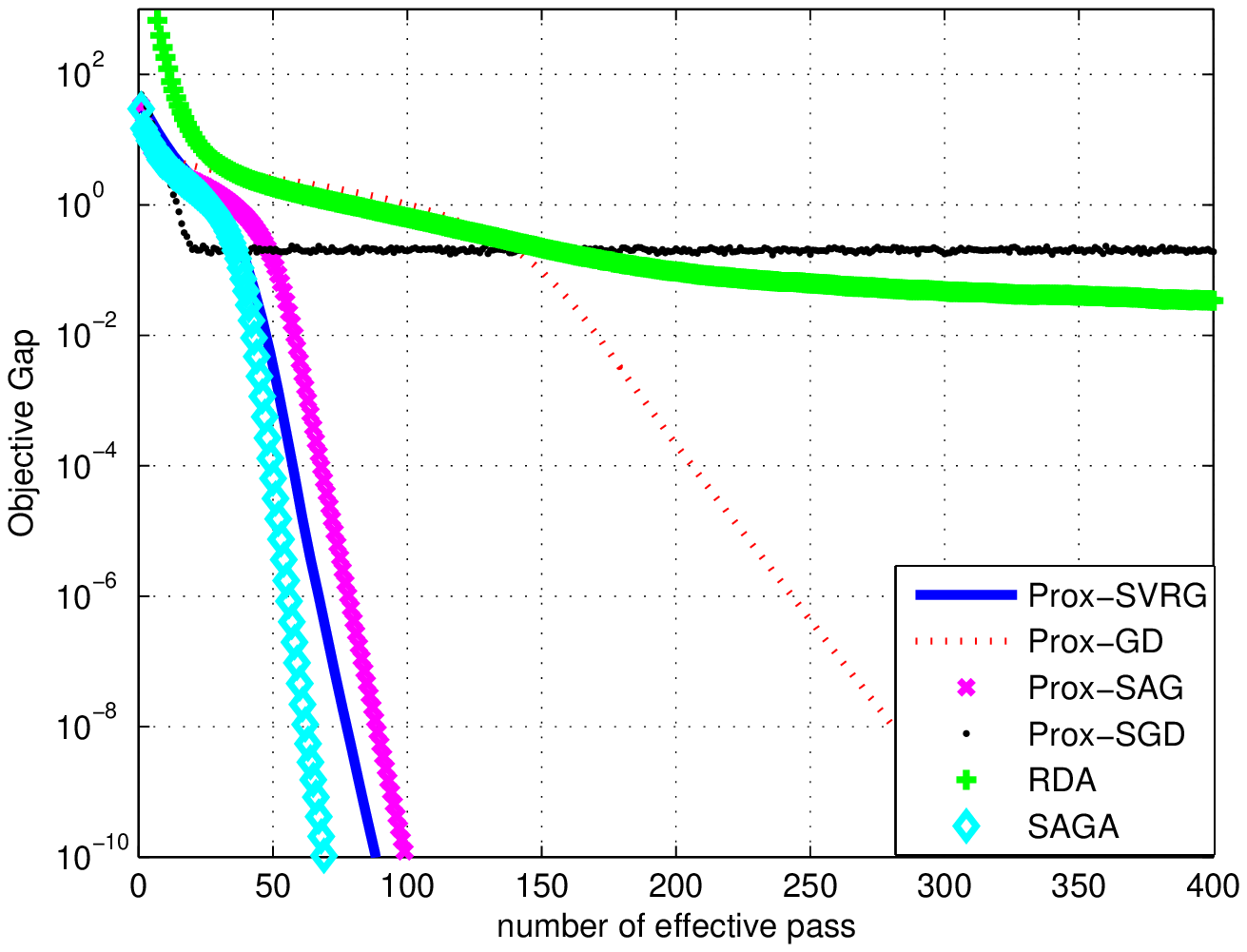}
		\caption{m=10, $s_{\mathcal{G}}=10, b=0$}
	\end{subfigure}
	\begin{subfigure}[b]{0.23\textwidth}
		\includegraphics[width=\textwidth]{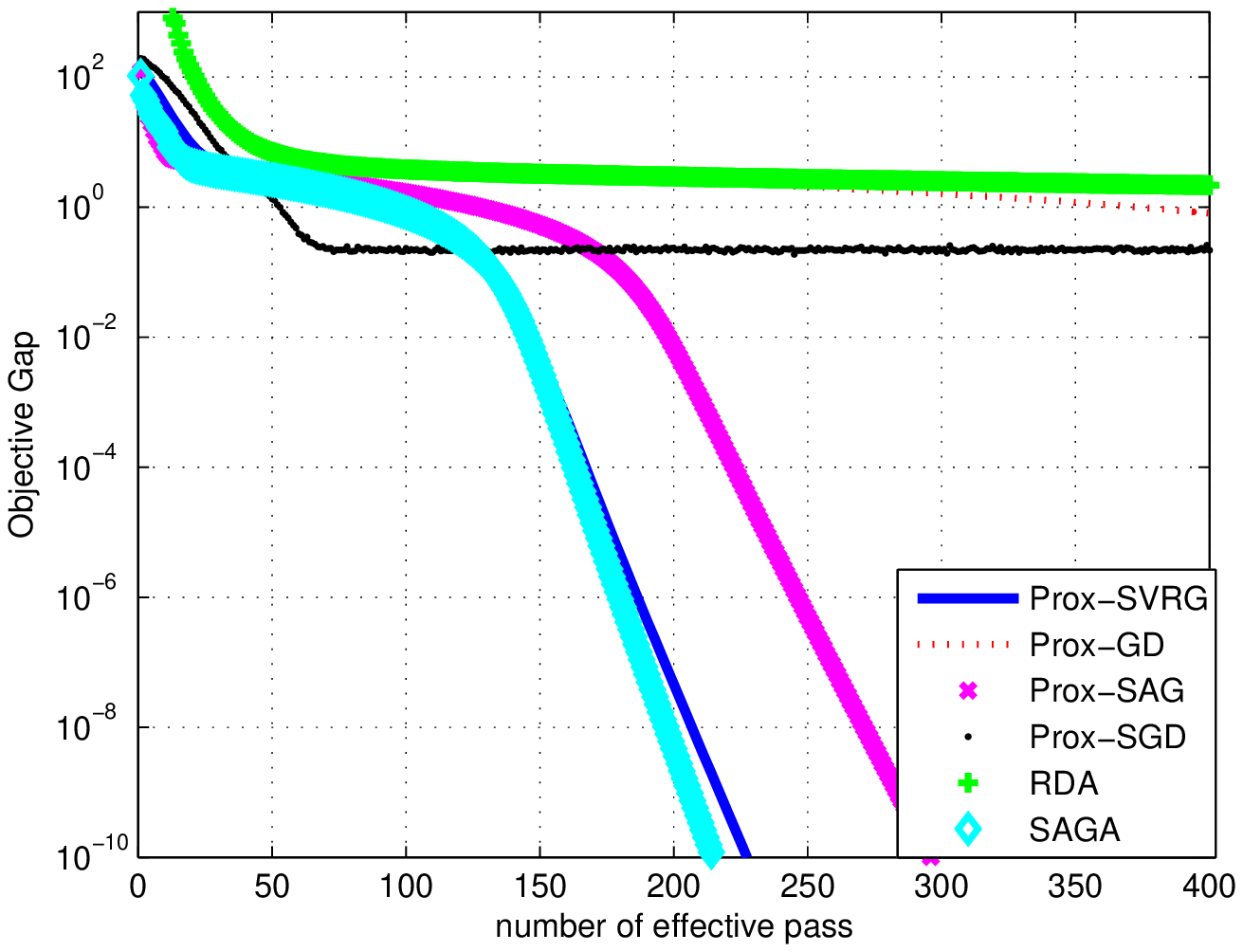}
		\caption{m=20, $s_{\mathcal{G}}=20, b=0$}
	\end{subfigure}
	\begin{subfigure}[b]{0.23\textwidth}
		\includegraphics[width=\textwidth]{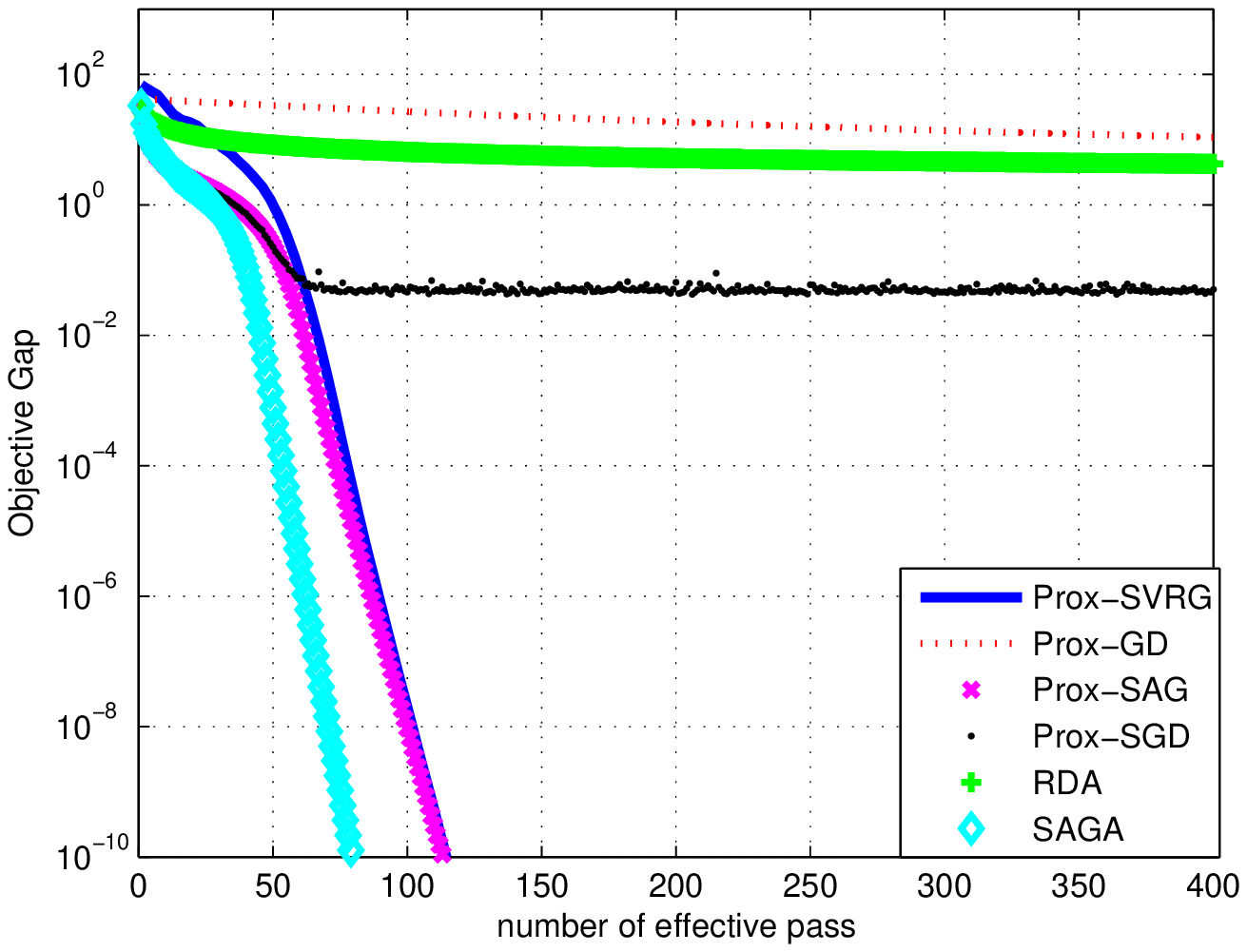}
		\caption{m=10, $s_{\mathcal{G}}=10, b=0.1$}
	\end{subfigure}~~~
	\begin{subfigure}[b]{0.23\textwidth}
		\includegraphics[width=\textwidth]{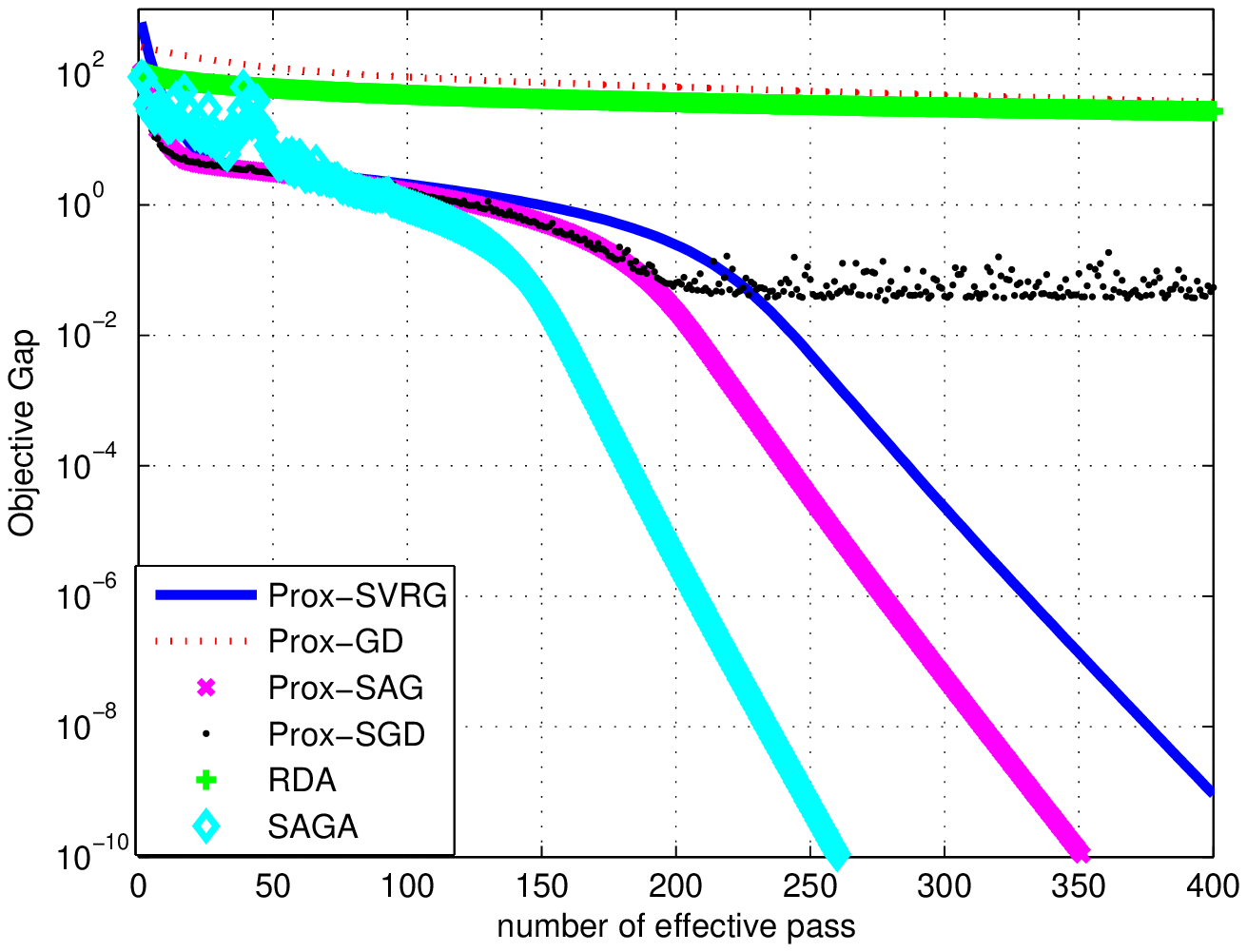}
		\caption{m=20, $s_{\mathcal{G}}=20, b=0.4$}
	\end{subfigure}
	
	\caption{Comparison between six algorithms on group Lasso. The x-axis is the number of passes over the dataset; the  y-axis is the objective gap $G(\theta^k)-G(\hat{\theta})$ with a log scale.} \label{Fig:gp_lasso}
\end{figure}

In all settings, SAGA, Prox-SVRG and Prox-SAG performs well. In the challenging setup ($m=20, s_{\mathcal{G}}=20$), SAGA outperforms the other two.  Prox-GD work with slower rate in the setting ($m=10, s_{\mathcal{G}}=10,b=0$), while its performance deteriorates in other three settings. Prox-GD and RDA have large optimality gap even after long time running. We have similar observation as that in Lasso, i.e., smaller $m$ and $s_\mathcal{G}$ lead to faster convergence. Again, it can be explained by the dependence of $\bar{\sigma}$ on $m$ and $s_\mathcal{G}$.

\subsubsection{Corrected Lasso}

We generate data as follows: $ y_i=x_i^T\theta^*+\xi_i$, where each data point $x_i\in \mathbb{R}^p$ is drawn from normal distribution $N(0,I)$, and the noise $ \xi_i$ is   drawn from $N(0,1)$. The coefficient $\theta^*$ is sparse with cardinality $r$, where the non-zero coefficient equals to $\pm 1$ generated from the Bernoulli distribution with parameter $0.5$. We set covariance matrix $\Sigma_w=\gamma_w I$. We choose $\lambda=0.05$ in the formulation.
\begin{figure}[h]
	\begin{subfigure}[b]{0.45\textwidth}
		\centering
		\includegraphics[width=\textwidth]{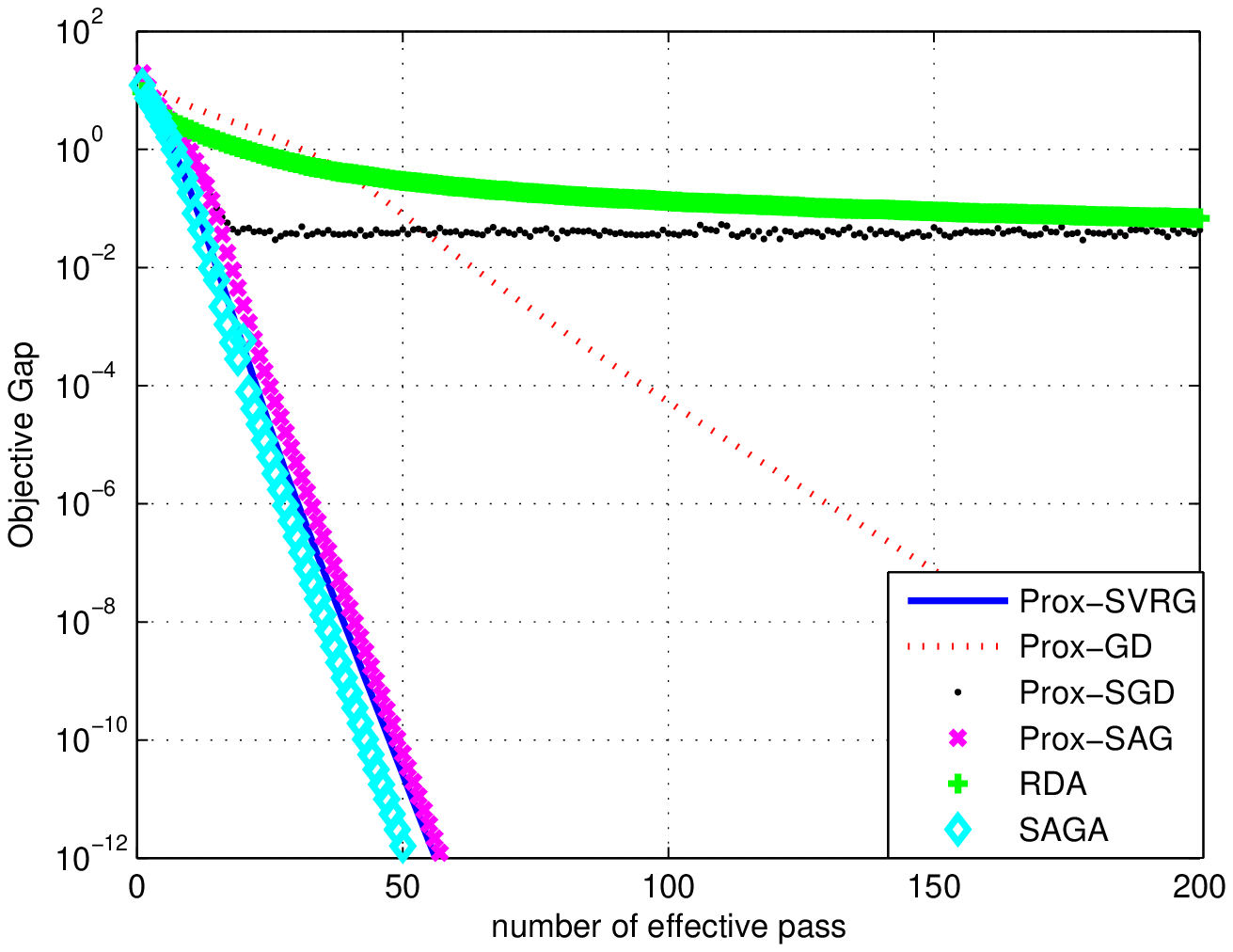}
		\caption{$n=2500, p=3000, r=50, \gamma_w=0.05$}
	\end{subfigure}
	\begin{subfigure}[b]{0.45\textwidth}
		\centering
		\includegraphics[width=\textwidth]{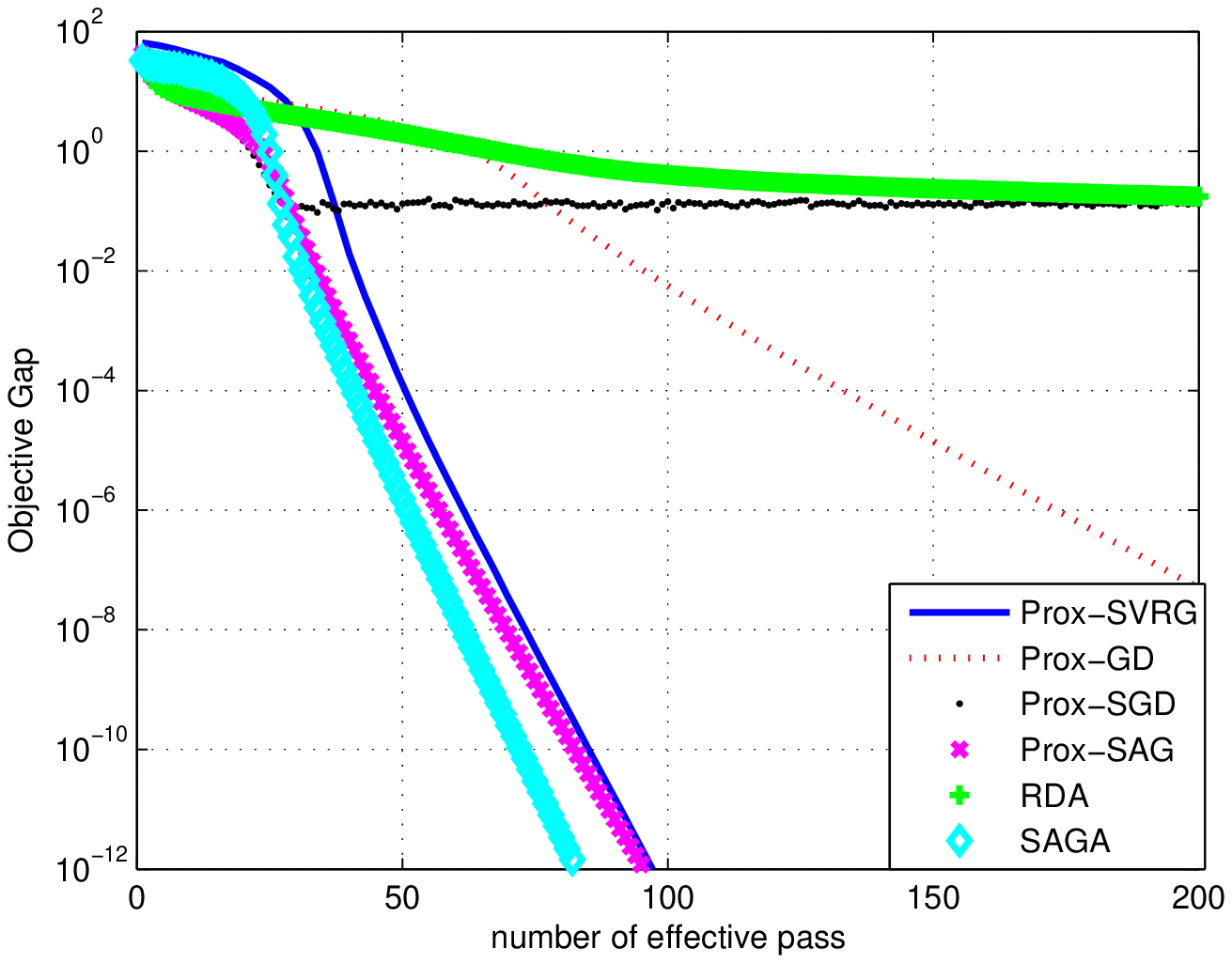}
		\caption{$n=2500,p=5000, r=100,\gamma_w=0.1$}
	\end{subfigure}
	\caption{The x-axis is the number of pass over the dataset. y-axis is the objective gap $G(\theta^k)-G(\hat{\theta})$ with  log scale. We try two different settings. }\label{fig:corrected_lasso}
\end{figure}

Figure \ref{fig:corrected_lasso} reports the result on Corrected Lasso.  In both settings, SAGA, Prox-SVRG and Prox-SAG work well and have similar performance. Prox-GD also enjoys the linear convergence rate but with a slower ratio. SGD and RDA have a large optimality gap even after 200 iterations,

\subsubsection{SCAD}
The way to generate data is same with Lasso. Here  $x_i\in \mathbb{R}^p$ is drawn from normal distribution $N(0,2I)$ (Here We choose $2I$ to satisfy the requirement of $\bar{\sigma}$ and $\mu$ in our Theorem, although if we choose $N(0, I)$, the algorithm still works. ). $\lambda=0.05$ in the formulation.  We present the result in Figure \ref{fig:SCAD}, for two settings on $n$, $p$, $r$, $\zeta$.

\begin{figure}[h]
	\begin{subfigure}[b]{0.45\textwidth}
		\centering
		\includegraphics[width=\textwidth]{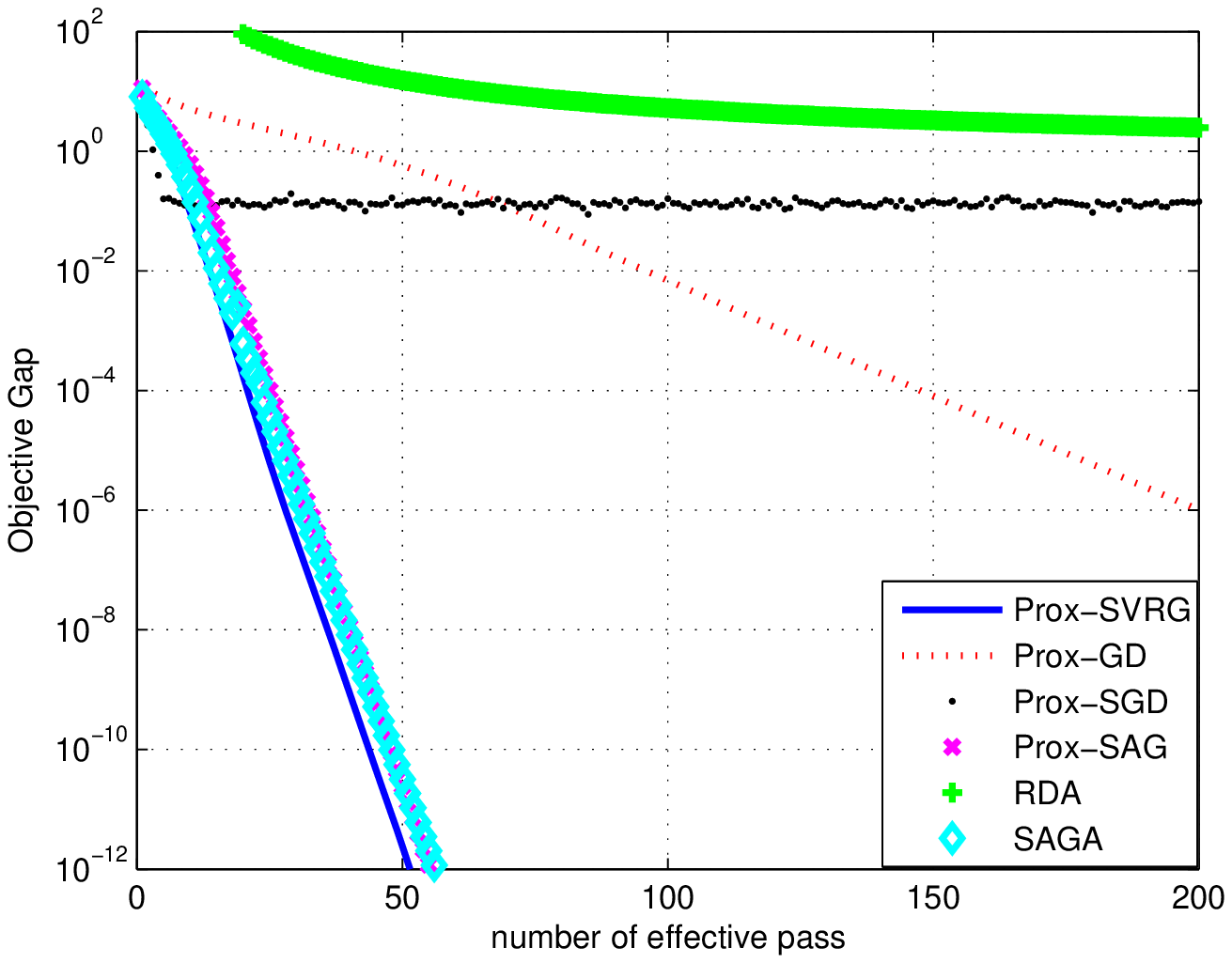}
		\caption{$n=3000, p=2500, r=30, \zeta=4.5 $ }
	\end{subfigure}
	\begin{subfigure}[b]{0.45\textwidth}
		\centering
		\includegraphics[width=\textwidth]{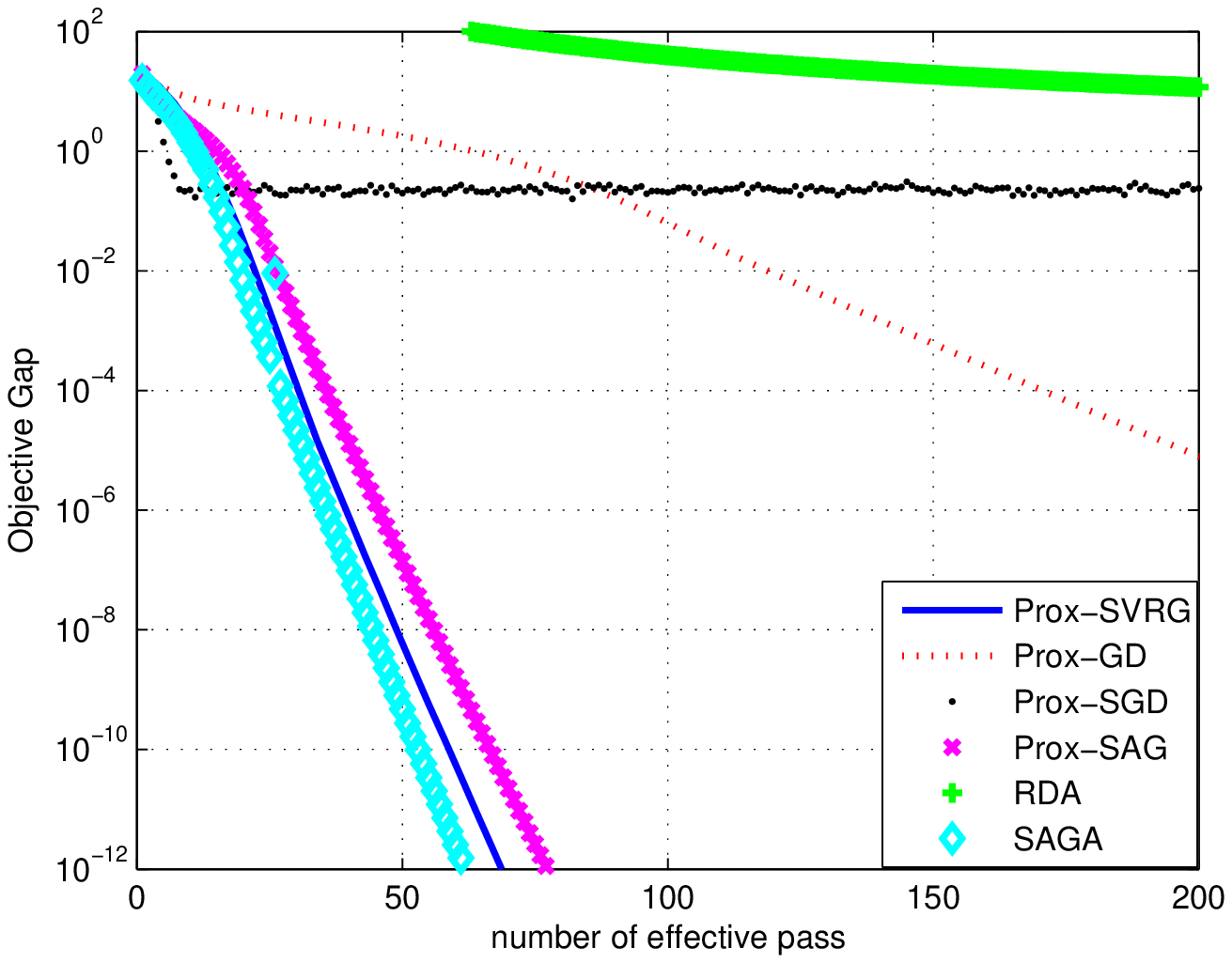}
		\caption{$n=2500, p=5000, r=50, \zeta=3.7 $  }
	\end{subfigure}
	\caption{The x-axis is the number of pass over the dataset. y-axis is the objective gap $G(\theta^k)-G(\hat{\theta})$ with  log scale. } \label{fig:SCAD}
\end{figure} 

Figure \ref{fig:SCAD} reports the simulation result on linear regression with SCAD regularizer. It is easy to see SAGA, Prox-SVRG and Prox-SAG works well, followed by Prox-GD. RDA and Prox-SGD does not converge well.

\subsection{Real datasets}

  \subsubsection{Sparse Classification Problem }
  In this section, we evaluate the performance of the algorithms when solving the logistic regression with $\ell_1$ regularization:
  $$ \min_{\theta} \sum_{i=1}^{n} \log (1+\exp (-y_ix_i^T\theta))+\lambda \|\theta\|_1.$$
  \begin{figure}
  	\begin{subfigure}[b]{0.45\textwidth}
  		\includegraphics[width=\textwidth]{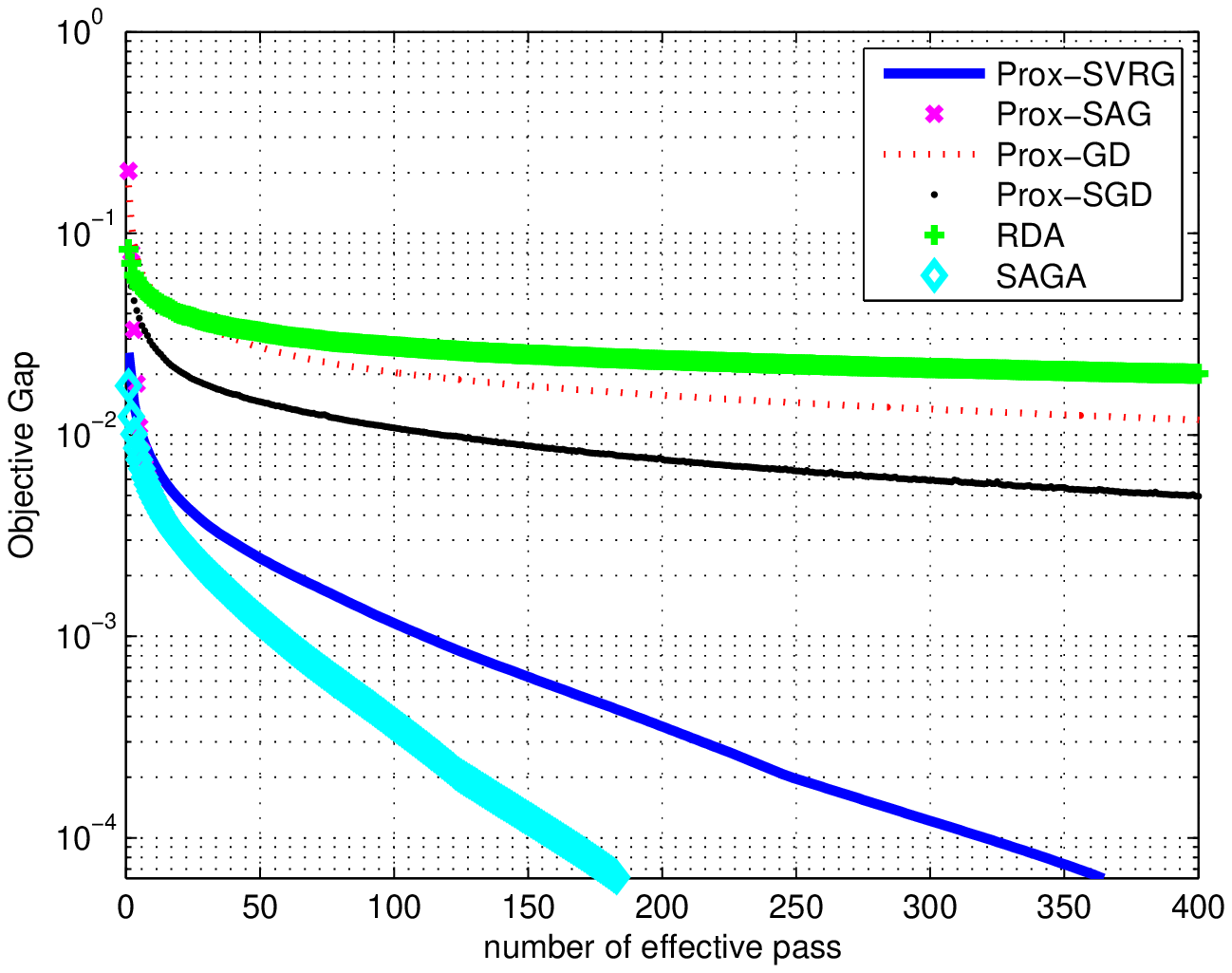}
  		\caption{ Sido0}\label{sido0}	
  	\end{subfigure}	
  	\begin{subfigure}[b]{0.45\textwidth}
  		\includegraphics[width=\textwidth]{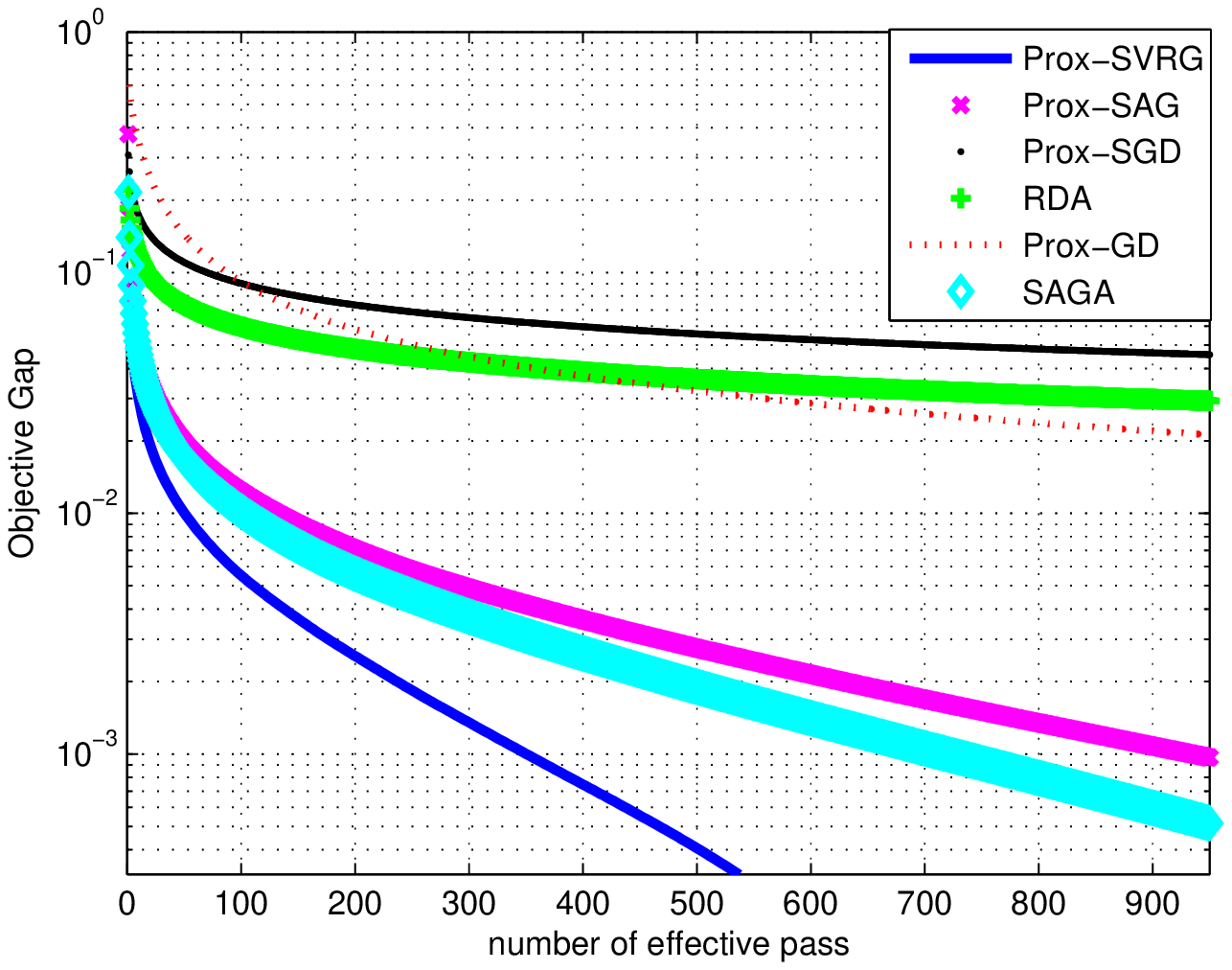}
  		\caption{Rcv1}	\label{rcv1}	
  	\end{subfigure}	
  	\caption{Different methods on sido0 and rcv1 dataset. The x-axis is the number of pass over the dataset, y-axis is the objective gap in the log-scale. }
  \end{figure}
  
  We conduct experiments on two real-world data sets:  sido0 ($n=12678, p=4932$)  \cite{SIDO} and rcv1 ($n=20242,p=47236$) \cite{lewis2004rcv1}. The regularization parameters are set as $\lambda=2\cdot 10^{-5}$ in rcv1 and $\lambda=10^{-4}$ in sido0, as suggested in \citet{xiao2014proximal}.

  
  Figure \ref{sido0} shows the results of the algorithms on the sido0 \cite{SIDO} dataset.  On this dataset, SAGA performs best and then followed by Prox-SAG (some part are overlapped with Prox-SVRG ) and then Prox-SVRG. The performance of  Prox-GD is even worse   than   Prox-SGD.   RDA converges the slowest. In Figure \ref{rcv1}, we report the performance of different algorithms on rcv1 dataset \cite{lewis2004rcv1}.  In this problem, Prox-SVRG performs best, and followed by SAGA, and then Prox-SAG. We observe that Prox-GD converges much slower, albeit in theory it should converges with a linear rate \cite{agarwal2010fast}, possibly because its contraction factor is close to one in this case. Prox-SGD and RDA converge slowly due to the variance in the stochastic gradient. The objective gaps of them remain significant even after 1000 passes of the whole dataset. 
\subsubsection{Sparse Regression Problem}
\begin{figure}[t]
	\includegraphics[width=0.32\textwidth]{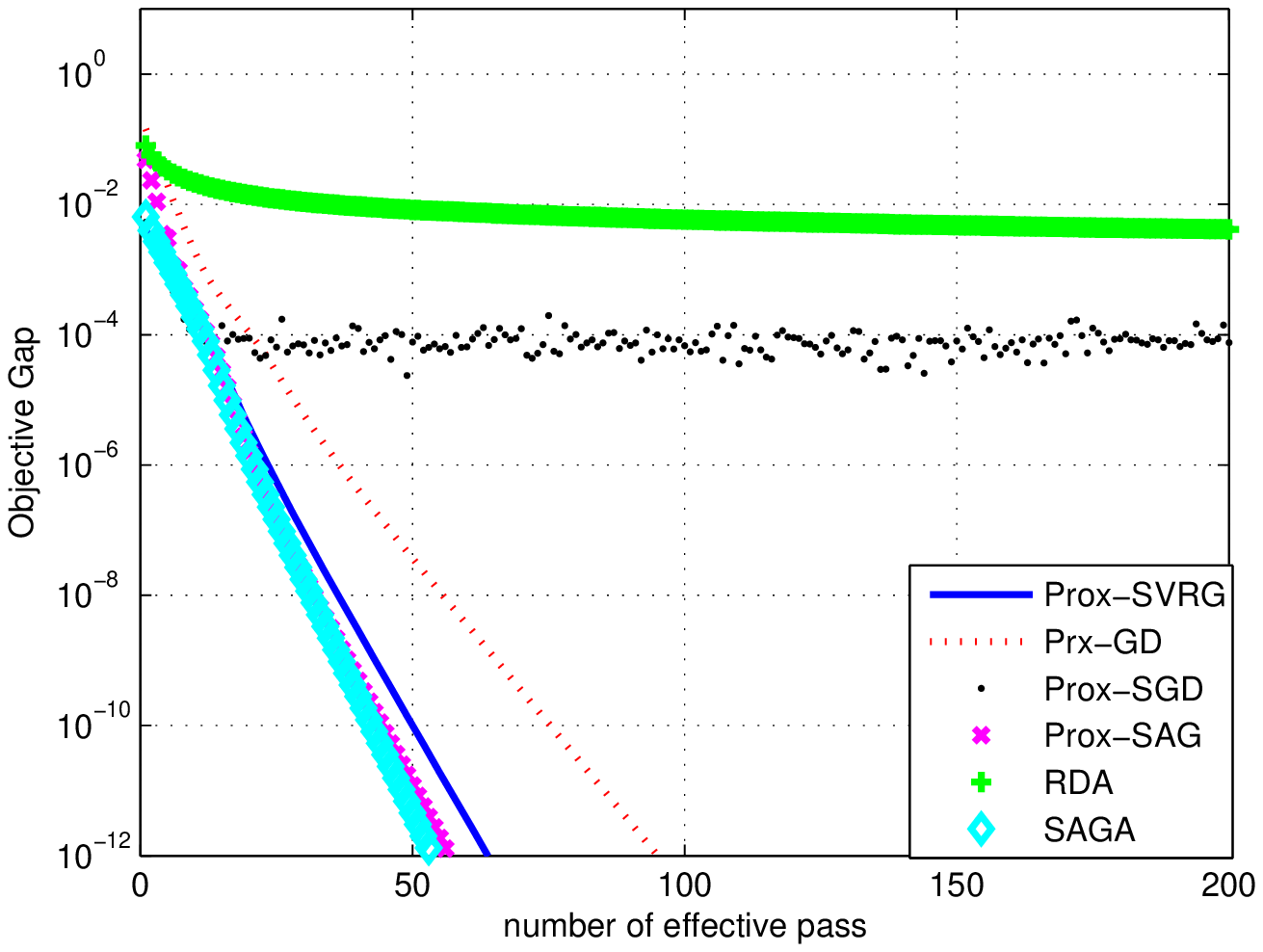}
	\includegraphics[width=0.32\textwidth]{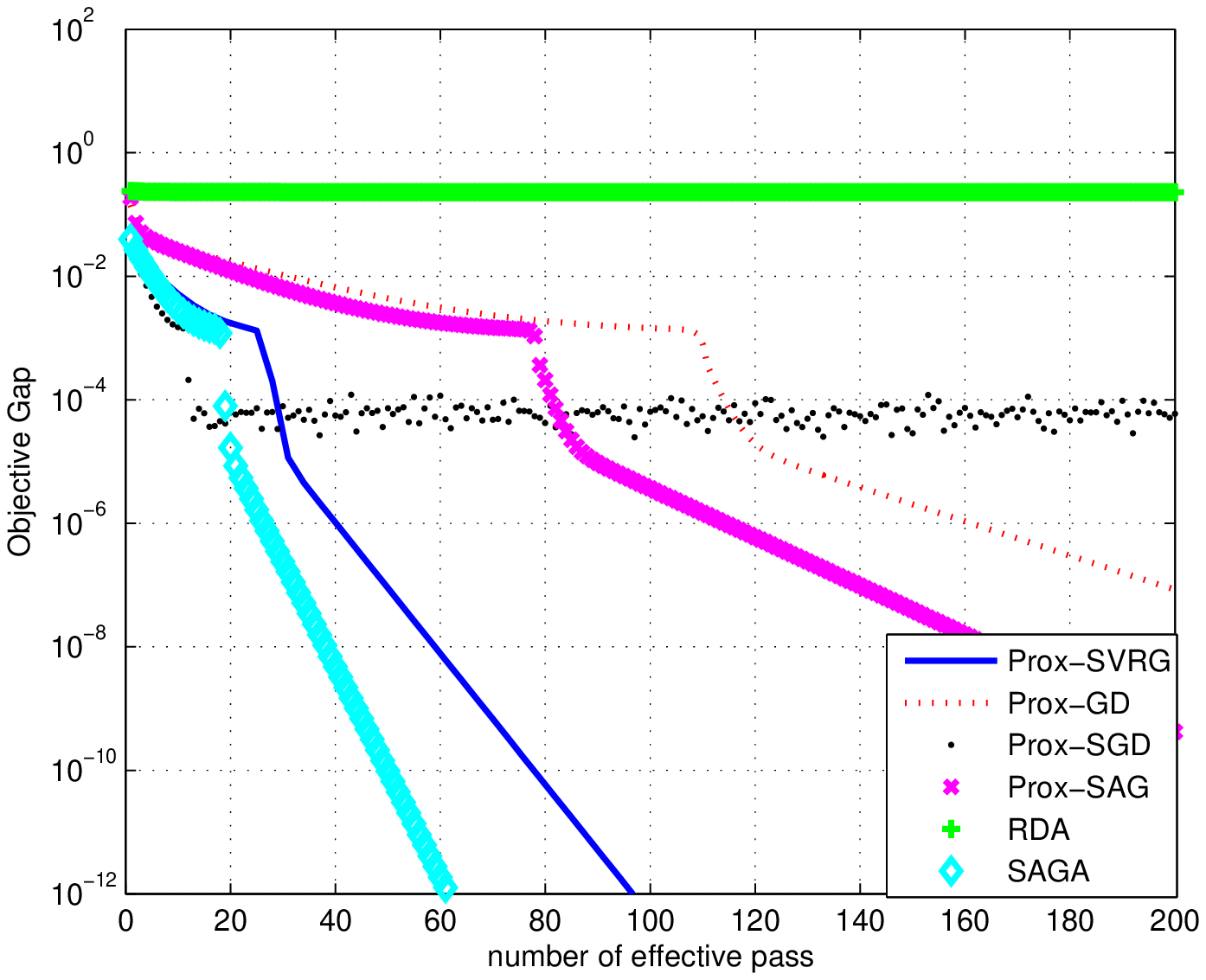}
	\includegraphics[width=0.32\textwidth]{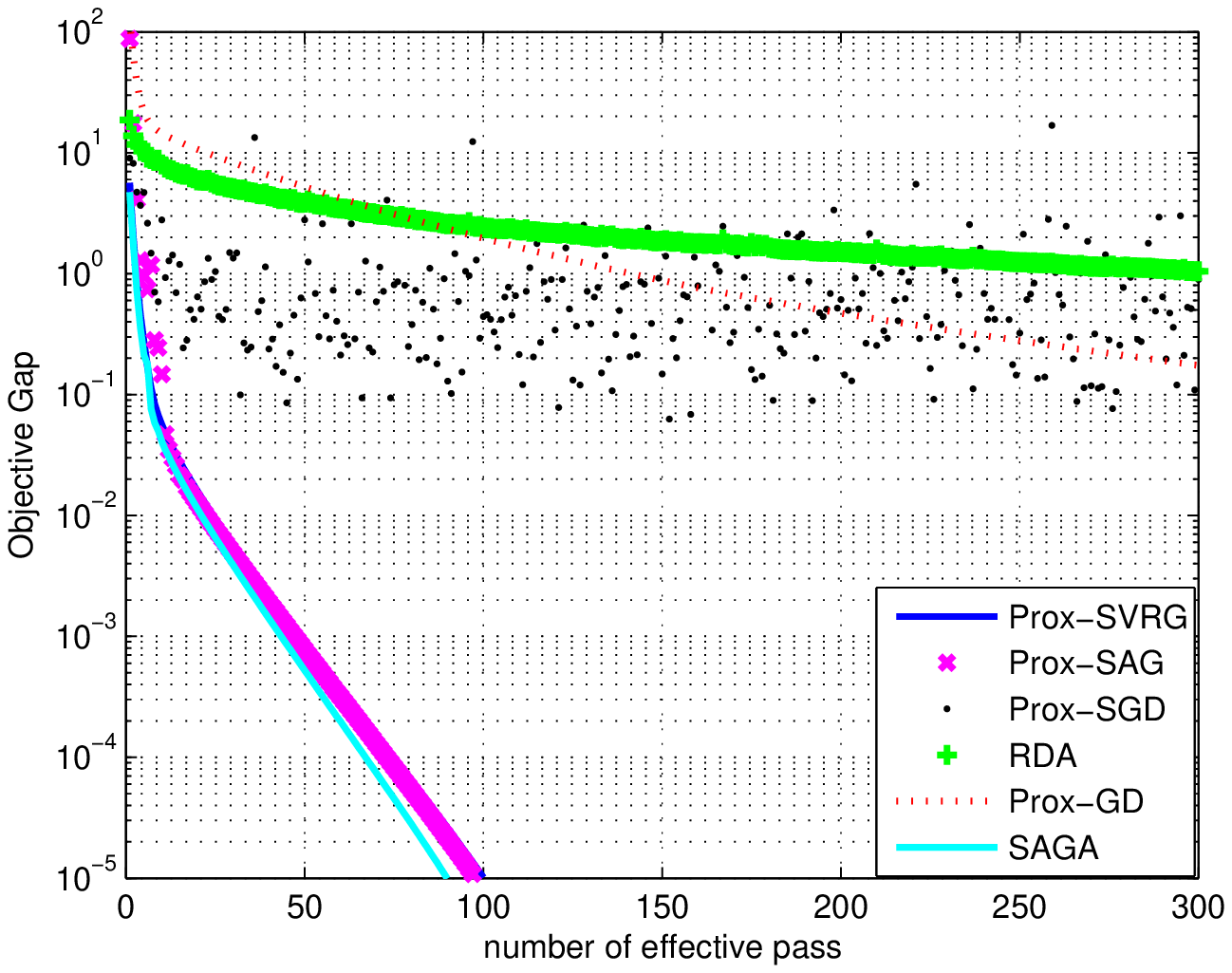}
	\caption{Six different algorithms on Lasso (left), linear regression  with SCAD regularization (middle) and Group Lasso (right).  The x-axis is the number of pass over the dataset, y-axis is the objective gap in the log-scale .}\label{Fig:regression problem}
\end{figure}
In this section, we consider regression problem on three different problems, namely Lasso, linear regression with SCAD regularization  and Group Lasso and report the results in Figure \ref{Fig:regression problem}. For Lasso and linear regression with SCAD regularization, we test all algorithms on IJCNN1 dataset ($n=49990,p=22$)  \cite{IJCNN1}. In particular, we set $\lambda=0.02$ in Lasso formulation and $\lambda=0.02$ and $\zeta=5$ in linear regression with SCAD regularization.
As to the group sparse regression problem, we conduct the experiment the Boston Housing dataset ($n=506,p=13$) \cite{uci:2013}.  As suggested in \citet{swirszcz2009grouped,xiang2014simultaneous}, to take into account the non-linear relationship between variables and response, up to third-degree polynomial expansion is applied on each feature. In particular, terms $x$, $x^2$ and $x^3$ are grouped together. We consider group Lasso model on this problem with $\lambda=0.1$.

It is easy to see that  for the Lasso problem, SAGA, Prox-SAG and Prox-SVRG have almost identical performance, and Prox-GD converges with linear rate but with slower rate. As to linear regression with SCAD regularization, SAGA performs best in this dataset and then followed by Prox-SVRG, Prox-SAG and Prox-GD. For both problems, Prox-SGD converges faster at the beginning but quickly slows down and eventually has a large optimality gap, possibly due to the variance in the gradient estimation.  RDA seems does not work (for both Lasso and  SCAD) in this dataset. In Group Lasso, SAGA, Prox-SVRG and prox-SAG have almost same performance. RDA and Prox-GD converge slowly. Prox-SGD does not converge and its value oscillates between 0.1 and 1. 

\section{Conclusion and future work}

In this paper, we analyze SAGA on a class of  non-strongly convex and non-convex problem and provide linear convergence analysis under the RSC condition. 

 \appendix
\section{Proofs}\label{app.proof}

In this section, we give the proof to all theorems and corollaries

\subsection{SAGA with convex objective function}

We start the proof with some technical Lemmas.

The following lemma is the theorem 2.1.5 in \cite{nesterov1998introductory}. 
\begin{lemma}\label{lemma.smooth}
	if $f(\theta)$ is convex and $L$ smooth, then we have 
	
	$$\|\nabla f(\theta_1)-\nabla f(\theta_2)\|_2^2\leq  2L [f(\theta_1)-f(\theta_2)- \langle \nabla f(\theta_2),\theta_1-\theta_2  \rangle] $$
\end{lemma}

The next lemma is a simple extension of a standard property proximal operator with a constraint $\Omega$. It is indeed the Lemma 5 in \cite{qu2016linear}, and we present here for completeness.
\begin{lemma}\label{lemma.prox}
	Define $prox_{h,\Omega}(x)=\arg\min_{z\in \Omega
	} h(z)+\frac{1}{2}\|z-x\|_2^2$, where $\Omega$ is a convex compact set, then $\|prox_{h,\Omega}(x)-prox_{h,\Omega}(y)\|_2\leq \|x-y\|_2 .$
\end{lemma}

The following two lemmas are similar to its batched counterpart in \cite{agarwal2010fast}.
\begin{lemma}\label{lemma.cone}
	
	Suppose that $f(\theta)$ is convex and $\psi(\theta)$ is decomposable with respect to $(M,\bar{M})$, if we choose $\lambda\geq 2\psi^*(\nabla f(\theta^*))$, $\psi(\theta^*)\leq \rho $ ,  define the error $\Delta^*=\hat{\theta}-\theta^*$, then we have the following condition holds, 
	$$ \psi(\Delta^*_{\bar{M}^\perp})\leq 3 \psi(\Delta^*_{\bar{M}})+4\psi(\theta^*_{M^\perp}), $$
	which further implies $ \psi (\Delta^*)\leq \psi(\Delta^*_{\bar{M}^\perp})+\psi(\Delta^*_{\bar{M}})\leq 4 \psi(\Delta^*_{\bar{M}})+4\psi(\theta^*_{M^\perp}).$
\end{lemma}

\begin{proof}
	Using the optimality of $\hat{\theta}$, we have
	$$ f(\hat{\theta})+\lambda \psi(\hat\theta)-f(\theta^*)-\lambda\psi(\theta^*)\leq 0.$$
	
	So we have
	$$ \lambda \psi(\theta^*)-\lambda \psi(\hat{\theta})\geq f(\hat{\theta})-f(\theta^*)\geq \langle \nabla f(\theta^*), \hat{\theta}-\theta^*  \rangle\geq -\psi^*(\nabla f(\theta ^*)) \psi (\Delta^*), $$
	where the second inequality holds from the convexity of $f(\theta)$, and the third holds using Holder inequality.
	
	Using triangle inequality, we have
	
	$$ \psi(\Delta^*)\leq \psi(\Delta^*_{\bar{M}})+\psi(\Delta^*_{\bar{M}^{\perp}}).$$
	
	So 
	\begin{equation}\label{equ:lemma_cone}
	\lambda \psi(\theta^*)-\lambda \psi(\hat{\theta})\geq -\psi^*(\nabla f(\theta ^*)) (\psi(\Delta^*_{\bar{M}})+\psi(\Delta^*_{\bar{M}^{\perp}}) )
	\end{equation}

	Notice $$\hat{\theta}=\theta^*+\Delta^*= \theta^*_{M}+\theta^*_{M^{\perp}}+\Delta^*_{\bar{M}}+\Delta^*_{\bar{M}^{\perp}},$$
	which leads to
	\begin{equation}
	\begin{split}
	\psi  (\hat{\theta})-\psi(\theta^*)& \overset{(a)}{\geq}  \psi (\theta^*_{M}+\Delta^*_{\bar{M}^{\perp}})-\psi(\theta^*_{M^{\perp}})-\psi(\Delta^*_{\bar{M}})-\psi(\theta^*)\\
	& \overset{(b)}{= }\psi (\theta^*_{M})+\psi (\Delta^*_{\bar{M}^{\perp}})-\psi(\theta^*_{M^{\perp}})-\psi(\Delta^*_{\bar{M}})-\psi(\theta^*)\\
	& \overset{(c)}{\geq} \psi (\theta^*_{M})+\psi (\Delta^*_{\bar{M}^{\perp}})-\psi(\theta^*_{M^{\perp}})-\psi(\Delta^*_{\bar{M}})-\psi(\theta^*_{M})-\psi(\theta^*_{M^{\perp}})\\
	&\geq \psi (\Delta^*_{\bar{M}^{\perp}})-2\psi(\theta^*_{M^{\perp}})-\psi(\Delta^*_{\bar{M}}),
	\end{split}
	\end{equation}
	where $(a)$ and $(c)$ holds from the triangle inequality, $(b)$ uses the decomposability of $\psi(\cdot)$.

	Substitute  left hand side of \ref{equ:lemma_cone} by above result and use the assumption that $\lambda\geq 2\psi^*(\nabla f(\theta^*))$, we have 
	$$-\frac{\lambda}{2} (\psi(\Delta^*_{\bar{M}})+\psi(\Delta^*_{\bar{M}^{\perp}}) )+\lambda (\psi (\Delta^*_{\bar{M}^{\perp}})-2\psi(\theta^*_{M^{\perp}})-\psi(\Delta^*_{\bar{M}}))\leq 0$$
	which implies 
	$$ \psi(\Delta^*_{\bar{M}^\perp})\leq 3  \psi(\Delta^*_{\bar{M}})+4\psi(\theta^*_{M^{\perp}}).$$
\end{proof}

\begin{lemma}\label{lemma.cone_optimization}
	$f(\theta)$ is convex and $\psi(\theta)$ is decomposable with respect to $(M,\bar{M})$, if we choose $\lambda\geq 2\psi^*(\nabla f(\theta^*))$, $\psi(\theta^*)\leq \rho $ and suppose there exist a time step $K>0$ and a given tolerance $\epsilon$ such that for all $k>K$, $ G(\theta^k)-G(\hat{\theta})\leq \epsilon $ holds,  then for the error $\Delta^k=\theta^k-\theta^* $ we have
	
	$$ \psi(\Delta^k_{\bar{M}^\perp})\leq 3 \psi(\Delta^k_{\bar{M}})+4\psi(\theta^*_{M^\perp})+ 2\min\{\frac{\epsilon}{\lambda},\rho\}, $$
	which implies $$ \psi(\Delta^k)\leq 4 \psi(\Delta^k_{\bar{M}})+4\psi(\theta^*_{M^\perp})+ 2\min\{\frac{\epsilon}{\lambda},\rho\}. $$
\end{lemma}
\begin{proof}
	First notice $ G(\theta^k)-G(\theta^*)\leq \epsilon $ holds by assumption since $G(\theta^*)\geq G(\hat{\theta}).$
	So we have
	$$ f(\theta^k)+\lambda \psi(\theta^k
	)-f(\theta^*)-\lambda\psi(\theta^*)\leq \epsilon.$$
	Follow same steps in the proof of Lemma \ref{lemma.cone}, we have
	
	$$ \psi(\Delta^k_{\bar{M}^\perp})\leq 3 \psi(\Delta^k_{\bar{M}})+4\psi(\theta^*_{M^\perp})+ 2\frac{\epsilon}{\lambda}. $$
	
	Notice  $\Delta^k=\Delta^k_{\bar{M}^\perp}+\Delta^k_{\bar{M}} $ so $\psi(\Delta^k_{\bar{M}^\perp})\leq \psi(\Delta^k_{\bar{M}})+\psi(\Delta^k) $ using the triangle inequality.\\
	Then use the fact that $\psi(\Delta^k)\leq \psi (\theta^*)+\psi(\theta^k)\leq 2\rho $, we establish
	$$ \psi(\Delta^k_{\bar{M}^\perp})\leq 3 \psi(\Delta^k_{\bar{M}})+4\psi(\theta^*_{M^\perp})+ 2\min\{\frac{\epsilon}{\lambda},\rho\}. $$
	The second statement follows immediately from $\psi(\Delta^k) \leq \psi(\Delta^k_{\bar{M}^\perp})+\psi(\Delta^k_{\bar{M}}) $.
\end{proof}

Using the above two lemmas we now prove modified restricted convexity.

\begin{lemma}\label{lemma.RSC_cone}
	Under the same assumptions of Lemma \ref{lemma.cone_optimization}, we have
	\begin{equation}
	\langle \nabla f(\theta^k)-\nabla f(\hat{\theta}), \theta^k-\hat{\theta}\rangle\geq  \left(\sigma-64\tau_\sigma H^2(\bar{M} )\right) \|\hat{\Delta}^k \|_2^2-2\epsilon^2(\Delta^*,M,\bar{M})
	\end{equation}
	and 
	\begin{equation}\label{RSC_cone}
	G(\theta^k)-G(\hat{\theta})\geq \left(\frac{\sigma}{2}-32\tau_\sigma H^2(\bar{M} )\right)\|\hat{\Delta}^k \|_2^2- \epsilon^2(\Delta^*,M,\bar{M}),
	\end{equation}
	
	where $ \epsilon^2(\Delta^*,M,\bar{M})=2\tau_{\sigma} (\delta_{stat}+\delta)^2 $, $\delta=2\min\{\frac{\epsilon}{\lambda},\rho\} $, and $\delta_{stat}= 8H(\bar{M})\|\Delta^*\|_2+8\psi(\theta^*_{M^\perp}) $.
\end{lemma} 

\begin{proof}

	At the beginning of the proof, we show a simple fact on $\hat{\Delta}^k=\theta^k-\hat{\theta}$. Notice the conclusion in Lemma \ref{lemma.cone_optimization} is on  $\Delta^k$, we need transfer it to $\hat{\Delta}^k$.
	\begin{equation}
	\begin{split}
	\psi(\hat{\Delta}^k)&\leq \psi (\Delta^k)+\psi(\Delta^*)\\
	&\leq  4 \psi(\Delta^k_{\bar{M}})+4\psi(\theta^*_{M^\perp})+ 2\min\{\frac{\epsilon}{\lambda},\rho\}+4 \psi(\Delta^*_{\bar{M}})+4\psi(\theta^*_{M^\perp})\\
	& \leq 4 H(\bar{M})\|\Delta^k\|_2+4H(\bar{M})\|\Delta^*\|_2+8\psi(\theta^*_{M^\perp})+2\min\{\frac{\epsilon}{\lambda},\rho\},
	\end{split}
	\end{equation}
	where the first inequality holds from the triangle inequality, the second inequality uses Lemma \ref{lemma.cone} and \ref{lemma.cone_optimization}, the third holds because of the definition of subspace compatibility. 
	
	We now use the above result to rewrite the RSC condition.
	We know 
	\begin{equation}\label{Eq.1.lemma}  
	f(\theta^k)-f(\hat{\theta})-\langle \nabla f(\hat{\theta}), \hat{\Delta}^k  \rangle\geq \frac{\sigma}{2} \|\hat{\Delta}^k \|_2^2-\tau_{\sigma}\psi^2(\hat{\Delta}^k) 
	\end{equation}
	
	and $$ f(\hat{\theta})-f(\theta^k)- \langle \nabla f(\theta^k),-\hat{\Delta}^k \rangle \geq \frac{\sigma}{2} \|\hat{\Delta}^k \|_2^2-\tau_{\sigma}\psi^2(\hat{\Delta}^k). $$
	
	Add above two together, we get 
	\begin{equation}\label{eq.3.lemma_proof}
	\langle \nabla f(\theta^k)-\nabla f(\hat{\theta}),\theta^k-\hat{\theta} \rangle\geq \sigma \|\hat{\Delta}^k \|_2^2-2\tau_{\sigma}\psi^2(\hat{\Delta}^k)
	\end{equation}

	Notice that
	\begin{equation}
	\begin{split}
	\psi(\hat{\Delta}^k)&\leq 4 H(\bar{M})\|\Delta^k\|_2+4H(\bar{M})\|\Delta^*\|_2+8\psi(\theta^*_{M^\perp})+2\min\{\frac{\epsilon}{\lambda},\rho\}\\
	&\leq  4 H(\bar{M})\|\hat{\Delta}^k\|_2+8H(\bar{M})\|\Delta^*\|_2+8\psi(\theta^*_{M^\perp})+2\min\{\frac{\epsilon}{\lambda},\rho\},
	\end{split}
	\end{equation}
	where the second inequality uses the triangle inequality.
	Now use the inequality $ (a+b)^2\leq 2 a^2+2b^2 $, we upper bound  $\psi^2 (\hat{\Delta}^k)$ with
	$$\psi^2 (\hat{\Delta}^k)\leq 32H^2(\bar{M})\|\hat{\Delta}^k
	\|_2^2+2\left(8H(\bar{M})\|\Delta^*\|_2+8\psi(\theta^*_{M^\perp})+2\min\{\frac{\epsilon}{\lambda},\rho\}\right)^2.$$
	Substitute  this upper bound into Equation~\eqref{eq.3.lemma_proof} , we have

	$$ \langle \nabla f(\theta^k)-\nabla f(\hat{\theta}),\theta^k-\hat{\theta} \rangle\geq \left(\sigma-64\tau_\sigma H^2(\bar{M} )\right)\|\hat{\Delta}^k \|_2^2-4\tau_\sigma\left(8H(\bar{M})\|\Delta^*\|_2+8\psi(\theta^*_{M^\perp})+2\min\{\frac{\epsilon}{\lambda},\rho\}\right)^2.$$
	
	Notice that by $\delta=2\min\{\frac{\epsilon}{\lambda},\rho\} $, $\delta_{stat}= 8H(\bar{M})\|\Delta^*\|_2+8\psi(\theta^*_{M^\perp}),$  and $ \epsilon^2(\Delta^*,M,\bar{M})=2\tau_{\sigma} (\delta_{stat}+\delta)^2 $, we have
	$$\epsilon^2(\Delta^*,M,\bar{M})=2\tau_\sigma\left(8H(\bar{M})\|\Delta^*\|_2+8\psi(\theta^*_{M^\perp})+2\min\{\frac{\epsilon}{\lambda},\rho\}\right)^2.$$
	We thus establish the result.
	
	\begin{equation}
	\langle \nabla f(\theta^k)-\nabla f(\hat{\theta}), \theta^k-\hat{\theta}\rangle\geq  \left(\sigma-64\tau_\sigma H^2(\bar{M} )\right)\|\hat{\Delta}^k\|_2^2-2\epsilon^2(\Delta^*,M,\bar{M})
	\end{equation}

	Using Equation \eqref{Eq.1.lemma} and the fact that $\hat{\theta}$ is the optimal solution  and $\phi(\cdot)$ is convex,	we obtain
	$ G(\theta^k)-G(\hat{\theta})\geq \frac{\sigma}{2} \|\hat{\Delta}^k \|_2^2-\tau_{\sigma}\psi^2(\hat{\Delta}^k).$
	We substitute the upper bound of  $\psi^2(\hat{\Delta}^k) $, and get
	$$ G(\theta^k)-G(\hat{\theta})\geq \left(\frac{\sigma}{2}-32\tau_\sigma H^2(\bar{M} )\right)\|\hat{\Delta}^k \|_2^2-2\tau_\sigma\left(8H(\bar{M})\|\Delta^*\|_2+8\psi(\theta^*_{M^\perp})+2\min\{\frac{\epsilon}{\lambda},\rho\}\right)^2.$$
	
	That is  
	\begin{equation}\label{RSC_modified}
	G(\theta^k)-G(\hat{\theta})\geq \left(\frac{\sigma}{2}-32\tau_\sigma H^2(\bar{M} )\right)\|\hat{\Delta}^k \|_2^2- \epsilon^2(\Delta^*,M,\bar{M}).
	\end{equation}
\end{proof}

\begin{lemma}\label{lemma.RSC_smooth}
	Under the same assumption of Lemma \ref{lemma.cone_optimization}, we have 
	
	\begin{equation}
	\begin{split}
	&\langle \nabla f(\theta^k)- \nabla f(\hat{\theta}),\theta^k-\hat{\theta}\rangle\geq \frac{1}{2} [f(\theta^k)-f(\hat{\theta})-\langle\nabla f(\hat{\theta}),\theta^k-\hat{\theta} \rangle]+\frac{\bar{\sigma}}{2}\|\theta^k-\hat{\theta}\|_2^2\\
	&+\frac{1}{4L} \|\nabla f(\theta^k)-\nabla f(\hat{\theta})\|_2^2-\epsilon^2 (\Delta^*,M,\bar{M}). 
	\end{split}
	\end{equation}
\end{lemma}

\begin{proof}
	
	\begin{equation}\label{eq.4.lemma_proof}
	\begin{split}
	&\langle \nabla f(\theta^k)- \nabla f(\hat{\theta}),\theta^k-\hat{\theta}\rangle-\frac{1}{2} [f(\theta^k)-f(\hat{\theta})-\langle\nabla f(\hat{\theta}),\theta^k-\hat{\theta} \rangle]\\
	=& \frac{1}{2}\langle \nabla f(\theta^k)- \nabla f(\hat{\theta}),\theta^k-\hat{\theta}\rangle+\frac{1}{2}[f(\hat{\theta})-f(\theta^k)-\langle \nabla f(\theta^k),\hat{\theta}-\theta^k \rangle].
	\end{split}	
	\end{equation}
	We use the modified RSC condition  and the smoothness of $f(\theta)$ in the proof, in particular, we have following holds from Lemma \ref{lemma.cone} and Lemma \ref{lemma.smooth}
	
	$$ \langle  \nabla f(\theta^k)-\nabla f(\hat{\theta}), \theta^k-\hat{\theta}\rangle\geq \bar{\sigma} \|\theta^k-\hat{\theta}\|_2^2-2\epsilon^2 (\Delta^*,M,\bar{M}),$$
	
	$$\| \nabla f(\theta^k)-\nabla f(\hat{\theta}) \|_2^2\leq 2L[f(\hat{\theta})-f(\theta^k)-\langle \nabla f(\theta^k), \hat{\theta}-\theta^k\rangle]. $$
	
	Substitute above bound  in the right hand side of \ref{eq.4.lemma_proof}, we establish the result.	
\end{proof}

The following Lemma is indeed Lemma 7.  We present here for the completeness.

\begin{lemma}\label{lemma.variance}
	Define $\Delta=-\frac{1}{\gamma} (w^{k+1}-\theta^k)-\nabla f(\theta^k) $, It holds that for any $\phi_i^k,\hat{\theta},\theta^k$ and $\beta>0$, we have
	\begin{equation}
	\begin{split}
	\mathbb{E} \|w^{k+1}-\theta^k-\gamma \nabla f(\hat{\theta})\|_2^2 &\leq \gamma^2 (1+\beta^{-1}) \mathbb{E} \|\nabla f_j (\phi_j^k)-\nabla f_j(\hat{\theta})\|_2^2\\
	&+\gamma^2 (1+\beta)\mathbb{E} \|\nabla f_j(\theta^k)-\nabla f_j(\hat{\theta})\|_2^2-\gamma^2\beta\|\nabla f(\theta^k)-\nabla f(\hat{\theta})\|_2^2. 
	\end{split}	
	\end{equation}
	and
	$$E \|\Delta\|_2^2\leq (1+\beta^{-1}) \mathbb{E} \|\nabla f_j(\phi_j^k)-\nabla f_j(\hat{\theta})\|_2^2 + (1+\beta) \mathbb{E} \| \nabla f_j(\theta^k)-\nabla f_j(\hat{\theta})\|_2^2.$$
	
\end{lemma}

Now we are ready to prove  Theorem \ref{Theorem:convex}
\begin{proof}[Proof of Theorem \ref{Theorem:convex}]
	Recall that we aim to prove that Lyapunov function $T_k$ converges geometrically until $G(\theta^k)-G(\hat{\theta})$ achieves tolerance related to statistical error. Recall the definition of $T_k$ is 
	
	$$T_{k}\triangleq \frac{1}{n} \sum_{i=1}^{n} \left( f_i(\phi_i^k)-f_i(\hat{\theta})- \langle \nabla f_i (\hat{\theta}),\phi_i^k-\hat{\theta} \rangle\right)+(c+\alpha) \|\theta^k-\hat{\theta
	}\|_2^2+b (G(\theta^{k})-G(\hat{\theta
})). $$

Now we bound $\mathbb{E} T_{k+1}$.	

\textbf{1. The first term on the right hand side of $\mathbb{E} T_{k+1}$}

It is easy to obtain
$$\mathbb{E}[\frac{1}{n}\sum_{i=1}^{n}f_i (\phi_i^{k+1}) ]=\frac{1}{n} f(\theta^k)+(1-\frac{1}{n})\frac{1}{n} \sum_{i=1}^{n} f_i(\phi_i^k). $$
$$\mathbb{E} [-\frac{1}{n}\sum_{i=1}^{n} \langle \nabla f_i(\hat{\theta}),\phi_i^{k+1}-\hat{\theta}\rangle] =-\frac{1}{n}\langle \nabla f(\hat{\theta}),\theta^k-\hat{\theta}\rangle-(1-\frac{1}{n}) \frac{1}{n} \sum_{i=1}^{n} \langle \nabla f_i(\hat{\theta}),\phi^k_i-\hat{\theta} \rangle.$$

\textbf{ 2. The second term $(c+\alpha) \mathbb{E}\|\theta^{k+1}-\hat{\theta}\|_2^2.$}

Notice we bound the term $c \mathbb{E} \|\theta^{k+1}-\hat{\theta}\|_2^2$  and $\alpha \mathbb{E} \|\theta^{k+1}-\hat{\theta}\|_2^2$ in different ways.

As for the term $c \mathbb{E} \|\theta^{k+1}-\hat{\theta}\|_2^2$, we have following bound.
\begin{equation}\label{Eq.0.theorem} 
\begin{split}
&\mathbb{E} \|\theta^{k+1}-\hat{\theta}\|_2^2\\
\leq & \mathbb{E}\|prox_{\psi,\Omega} (w^{k+1})-prox_{\psi, \Omega} (\hat{\theta}-\gamma\nabla f(\hat{\theta})) \|_2^2\\
\overset{(a)}{\leq} & \mathbb{E} \|w^{k+1}-\hat{\theta}+\gamma \nabla f(\hat{\theta})\|_2^2\\
\leq & \mathbb{E} \|\theta^k-\hat{\theta}+w^{k+1}-\theta^k+\gamma \nabla f(\hat{\theta})\|_2^2\\
=& \|\theta^k-\hat{\theta}\|_2^2+2\mathbb{E}[\langle w^{k+1}-\theta^k+\gamma \nabla f(\hat{\theta}),\theta^k-\hat{\theta} \rangle]+\mathbb{E} \|w^{k+1}-\theta^k+\gamma \nabla f(\hat{\theta})\|_2^2\\
\overset{(b)}{=}& \|\theta^k-\hat{\theta}\|_2^2-2\gamma \langle \nabla f(\theta^k)-\nabla f(\hat{\theta}),\theta^k-\hat{\theta} \rangle+\mathbb{E} \|w^{k+1}-\theta^k+\gamma \nabla f(\hat{\theta})\|_2^2,\\
\overset{(c)}{\leq} & \|\theta^k-\hat{\theta}\|_2^2-2\gamma \langle \nabla f(\theta^k)-\nabla f (\hat{\theta}),\theta^k-\hat{\theta} \rangle-\gamma^2\beta \|\nabla f(\theta^k)-\nabla f(\hat{\theta})\|_2^2\\
&+
(1+\beta^{-1}) \gamma^2 \mathbb{E} \|\nabla f_j(\phi_j^k)-\nabla f_j (\hat{\theta})\|_2^2+(1+\beta) \gamma^2 \mathbb{E}\|\nabla f_j (\theta^k)-\nabla f_j(\hat{\theta})\|_2^2.
\end{split}
\end{equation}

where (a) holds from the non-expansiveness of the proximal operator, i.e., Lemma \ref{lemma.prox}, and (b) holds from the fact that $\mathbb{E}[w^{k+1}]=\theta^k-\gamma\nabla f(\theta^k),$ (c) uses Lemma \ref{lemma.variance}.

Now we apply Lemma \ref{lemma.RSC_smooth} on $\langle \nabla f(\theta^k)-\nabla f(\hat{\theta}),\theta^k-\hat{\theta} \rangle $ and obtain

\begin{equation}\label{Eq.1.theorem}
\begin{split}
&\mathbb{E} \|\theta^{k+1}-\hat{\theta}\|_2^2\\
\leq &(1-\gamma \bar{\sigma}) \|\theta^k-\hat{\theta}\|_2^2-\gamma^2\beta \|\nabla f(\theta^k)-\nabla f(\hat{\theta})\|_2^2+(1+\beta^{-1}) \gamma^2 \mathbb{E} \|\nabla f_j(\phi_j^k)-\nabla f_j (\hat{\theta})\|_2^2\\
+&( (1+\beta) \gamma^2-\frac{\gamma}{2L})  \mathbb{E}\|\nabla f_j (\theta^k)-\nabla f_j(\hat{\theta})\|_2^2-\gamma [f(\theta^k)-f(\hat{\theta})-\langle \nabla f(\hat{\theta}),\theta^k-\hat{\theta} \rangle ]+2\gamma \epsilon^2 (\Delta^*, M,\bar{M})\\
\leq & (1-\gamma \bar{\sigma}) \|\theta^k-\hat{\theta}\|_2^2
+((1+\beta)\gamma^2-\frac{\gamma}{2L}) \mathbb{E}\|\nabla f_j (\theta^k)-\nabla f_j(\hat{\theta})\|_2^2 +2\gamma \epsilon^2 (\Delta^*, M,\bar{M})\\-&\gamma [f(\theta^k)-f(\hat{\theta})-\langle \nabla f(\hat{\theta}),\theta^k-\hat{\theta} \rangle ]+2(1+\beta^{-1})\gamma^2 L [\frac{1}{n} \sum_{i=1}^{n} f_i(\phi_i ^k)-f(\hat{\theta})-\frac{1}{n}\sum_{i=1}^{n} \langle \nabla f_i(\hat{\theta}),\phi_i^k-\hat{\theta} \rangle]\\
\end{split}
\end{equation}

Then we bound the term $\alpha \mathbb{E} \|\theta^{k+1}-\hat{\theta}\|_2^2$. Define $\Delta=-\frac{1}{\gamma} (w^{k+1}-\theta^k)-\nabla f(\theta^k) $.

The following equation can be obtained from second equation on pg. 12 in \cite{xiao2014proximal}. 

\begin{equation}
\begin{split}
\alpha \mathbb{E} \|\theta^{k+1}-\hat{\theta}\|_2^2 \leq \alpha \|\theta^k-\hat{\theta}\|_2^2-2\alpha \gamma \mathbb{E} [ G(\theta^{k+1})-G(\hat{\theta}) ]+2\alpha \gamma^2 \mathbb{E} \|\Delta\|_2^2,
\end{split}
\end{equation}

Notice although the definition of $\Delta$ is different,  they only use the property $E (\Delta)=0$ to prove above equation.

We apply Lemma \ref{lemma.variance} on $E \|\Delta\|_2^2$ and get

$$E \|\Delta\|_2^2\leq (1+\beta^{-1}) \mathbb{E} \|\nabla f_j(\phi_j^k)-\nabla f_j(\hat{\theta})\|_2^2 + (1+\beta) \mathbb{E} \| \nabla f_j(\theta^k)-\nabla f_j(\hat{\theta})\|_2^2.$$
Then 
\begin{equation}\label{Eq.2.theorem}
\begin{split}
&\mathbb{E} \alpha\|\theta^{k+1}-\hat{\theta}\|_2^2\leq \alpha\|\theta^k-\hat{\theta}\|_2^2-2\alpha\gamma \mathbb{E} [G(\theta^{k+1})-G(\hat{\theta})]\\
&+2\alpha(1+\beta^{-1})\gamma^2 \mathbb{E}\|\nabla f_j(\phi_j^k)-\nabla f_j(\hat{\theta})\|_2^2 +2\alpha(1+\beta)\gamma^2 \mathbb{E} \|\nabla f_j(\theta^k)-\nabla f_j (\hat{\theta})\|_2^2. 
\end{split}
\end{equation}

Combine the result \eqref{Eq.1.theorem} and \eqref{Eq.2.theorem} together and apply Lemma \ref{lemma.smooth} on $ \mathbb{E} \|\nabla f_j(\phi_j^k)-\nabla f_j(\hat{\theta})\|_2^2 $ we obtain 

\begin{equation}
\begin{split}
&(\alpha+c) \mathbb{E} \|\theta^{k+1}-\hat{\theta}\|_2^2\\
\leq &(c+\alpha-c\gamma \bar{\sigma}) \|\theta^k-\hat{\theta}\|_2^2+( (c+2\alpha) (1+\beta) \gamma^2-\frac{c\gamma}{2L})  \mathbb{E}\|\nabla f_j (\theta^k)-\nabla f_j(\hat{\theta})\|_2^2\\
+& 2 (c+2\alpha)(1+\beta^{-1}) \gamma^2   L [\frac{1}{n} \sum_{i=1}^{n} f_i(\phi_i ^k)-f(\hat{\theta})-\frac{1}{n}\sum_{i=1}^{n} \langle \nabla f_i(\hat{\theta}),\phi_i^k-\hat{\theta} \rangle]\\
+&2c\gamma \epsilon^2 (\Delta^*, M,\bar{M})-c\gamma [f(\theta^k)-f(\hat{\theta})-\langle \nabla f(\hat{\theta}),\theta^k-\hat{\theta} \rangle ]-2\alpha \gamma \mathbb{E} [G(\theta^{k+1})-G(\hat{\theta})].
\end{split}
\end{equation}
Combine all pieces together, we obtain 
\begin{equation}
\begin{split}
&\mathbb{E} T_{k+1}-T_k\leq -\frac{1}{\kappa} T_k+ (\frac{c+\alpha}{\kappa}-c\gamma \bar{\sigma}) \|\theta^k-\hat{\theta}\|_2^2\\
+ &\left( \frac{1}{\kappa}+ 2 (c+2\alpha)(1+\beta^{-1}) \gamma^2   L-\frac{1}{n}  \right)[\frac{1}{n} \sum_{i=1}^{n} f_i(\phi_i ^k)-f(\hat{\theta})-\frac{1}{n}\sum_{i=1}^{n} \langle \nabla f_i(\hat{\theta}),\phi_i^k-\hat{\theta} \rangle]\\
+ & \left( \frac{1}{n}-c\gamma \right)[f(\theta^k)-f(\hat{\theta})-\langle \nabla f(\hat{\theta}),\theta^k-\hat{\theta} \rangle ]+ \left( (c+2\alpha) (1+\beta) \gamma^2-\frac{c\gamma}{2L}\right)  \mathbb{E}\|\nabla f_j (\theta^k)-\nabla f_j(\hat{\theta})\|_2^2\\
+&2c\gamma \epsilon^2 (\Delta^*, M,\bar{M}) -2\alpha \gamma \mathbb{E} [G(\theta^{k+1})-G(\hat{\theta})]+b \mathbb{E}(G(\theta^{k+1})-G(\hat{\theta})) -(1-\frac{1}{\kappa})b [G(\theta^k)-G(\hat{\theta})] 
\end{split}
\end{equation}

Recall that we choose $c=2\alpha$,$b=2 \alpha \gamma$, $\beta=2$, $\gamma=\frac{1}{9L}$, $\frac{1}{\kappa}=\min\{\frac{\bar{\sigma}}{14L}, \frac{1}{9n}\}$,  $c=\frac{9L}{n}$, so that the coefficient $\frac{c+\alpha}{\kappa}-c\gamma \bar{\sigma}, \frac{1}{\kappa}+ 2 (c+2\alpha)(1+\beta^{-1}) \gamma^2   L-\frac{1}{n}, \frac{1}{n}-c\gamma,  (c+2\alpha) (1+\beta) \gamma^2-\frac{c\gamma}{2L}$ are all non-positive.

Thus, we obtain

$$\mathbb{E} T_{k+1}-T_k\leq -\frac{1}{\kappa} T_k+2c\gamma \epsilon^2 (\Delta^*, M,\bar{M}) -(1-\frac{1}{\kappa}) c \gamma [G(\theta^{k})-G(\hat{\theta})].$$


%
%
%
%
%
%
%
%
%
\textbf{3. The geometrical convergence of $T_k$}

Next we prove the Lyapunov function decreases geometrically until $G(\theta^k)-G(\hat{\theta})$ achieves the tolerance $\delta$. In high level, we divide the time steps $k=1,2,...$ into several epochs, i.e., $ ([ T_0,T_1], (T_1,T_2],...)$.  At the end of  each epoch $j$,  we prove that $T_k$ decreases with linear rate until the optimality gap $G(\theta^k)-G(\hat{\theta})$ decreases to some tolerance $\xi_j$. We then prove that $(\xi_1,\xi_2,\xi_3,...)$ is a decreasing sequence and finish the proof.

Now we analyze the progress of  $T_k$ across different epochs. Suppose time step $k$ is in the epoch $j$, and we have $G(\theta^k)-G(\hat{\theta})\leq \xi_{j-1}$. Then we apply Lemma \ref{lemma.non_convex_RSC_cone} and have 

$$\epsilon_{j}^2(\Delta^*,M,\bar{M})=2\tau_\sigma\left(\delta_{stat}+ \delta_{j-1}\right)^2, $$

with  $\delta_{j-1}=2\min\{\frac{\xi_{j-1}}{\lambda},\rho\},$ and $\delta_{stat}=8H(\bar{M})\|\Delta^*\|_2+8\psi(\theta^*_{M^\perp})$.

Now we start the induction step. Although we do not know $\xi_0$, we can choose $\delta_0=2\rho.$ In this case,  $ \epsilon_1^2 (\Delta^*, M,\bar{M})=2\tau_\sigma\left(\delta_{stat}+ 2\rho\right)^2$.

We choose $T_1$ such that 

$$ (1-\frac{1}{\kappa})\left( G(\theta^{T_1-1})-G(\hat{\theta})\right)\geq 2\epsilon_1^2 (\Delta^*, M,\bar{M}) $$ 
and $$ (1-\frac{1}{\kappa})\left( G(\theta^{T_1})-G(\hat{\theta})\right)\leq 2\epsilon_1^2 (\Delta^*, M,\bar{M}).$$

Notice such $T_1$ must exist, otherwise we have $ (1-\frac{1}{\kappa})\left( G(\theta^{k})-G(\hat{\theta})\right)\geq 2\epsilon_1^2 (\Delta^*, M,\bar{M})$ and $ \mathbb{E} T_{k+1}\leq (1-\frac{1}{\kappa}) T_k$ holds for every $k$,i.e., $T_k$ converges geometrically, which is a contradiction with  $ (1-\frac{1}{\kappa})\left( G(\theta^{k})-G(\hat{\theta})\right)\geq 2\epsilon_1^2 (\Delta^*, M,\bar{M})$.

Now we know $ \left( G(\theta^{T_1})-G(\hat{\theta})\right)\leq \frac{2}{1-1/\kappa}\epsilon_1^2 (\Delta^*, M,\bar{M})$, thus we choose $\xi_1=\frac{4}{1-1/\kappa} \tau_\sigma\left(\delta_{stat}+ 2\rho\right)^2.$ It is time to follow the same argument in the second epoch. Recall we have

$$ \epsilon_2^2 (\Delta^*, M,\bar{M})=2\tau_\sigma\left(\delta_{stat}+ \delta_1\right)^2,$$ where $\delta_1=2 \min\{ \frac{\xi_1}{\lambda},\rho \}.$

We choose $T_2$ such that 

$$ (1-\frac{1}{\kappa})\left( G(\theta^{T_2-1})-G(\hat{\theta})\right)\geq 2\epsilon_2^2 (\Delta^*, M,\bar{M}) $$ 
and $$ (1-\frac{1}{\kappa})\left( G(\theta^{T_2})-G(\hat{\theta})\right)\leq 2\epsilon_2^2 (\Delta^*, M,\bar{M}),$$

and $\xi_2=\frac{4}{1-1/\kappa} \tau_\sigma\left(\delta_{stat}+ \delta_1\right)^2.$

Similarly, in epoch $j$, we choose $T_j$ such that 

$ (1-\frac{1}{\kappa})\left( G(\theta^{T_j-1})-G(\hat{\theta})\right)\geq 2\epsilon_j^2 (\Delta^*, M,\bar{M}) $ 
and $ (1-\frac{1}{\kappa})\left( G(\theta^{T_j})-G(\hat{\theta})\right)\leq 2\epsilon_j^2 (\Delta^*, M,\bar{M}),$
and $\xi_j=\frac{4}{1-1/\kappa} \tau_\sigma (\delta_{stat}+\delta_{j-1})^2$.

In this way, we arrive at recursive equalities of the tolerance $\{\xi_j\}^{\infty}_{j=1}$ where  $\xi_j=\frac{4}{1-1/\kappa} \tau_\sigma\left(\delta_{stat}+ \delta_{j-1}\right)^2$ and  $\delta_{j-1}=2 \min\{ \frac{\xi_{j-1}}{\lambda},\rho \}$.

We claim that following holds, until $\delta_i=\delta_{stat}$. 
\begin{equation}\label{equ.thm1step5}
\begin{split}
(I)\quad &\xi_{k+1}\leq \xi_k/ (4^{2^{k-1}}) \\
\mbox{and}\quad  (II)\quad &\quad\frac{\xi_{k+1}}{\lambda}\leq \frac{\rho}{4^{2^k}} \quad for \quad k=1,2,3,... 
\end{split}
\end{equation}

Assume $\frac{1}{\kappa}\leq \frac{1}{3}$ (Notice it is safe to do so since $\frac{1}{\kappa}=\min \{ \frac{2\bar{\sigma}}{3L}, \frac{1}{2n}\}$ , and when $n>2$ it holds), we have $\xi_j\leq 6 \tau_\sigma\left(\delta_{stat}+ \delta_{j-1}\right)^2.$

The proof of Equation~\eqref{equ.thm1step5} is same with Equation (60) in \cite{agarwal2010fast}, which we present here for  completeness.

We assume $\delta_0\geq \delta_{stat}$ (otherwise the statement is true trivially), so $\xi_1\leq 96 \tau_\sigma \rho^2$. We assume that $ \lambda\geq 384\tau_\sigma \rho $, so $ \frac{\xi_1}{\lambda}\leq \frac{\rho}{4}$ and $ \xi_1\leq \xi_0$. 

In the second epoch we have $$\xi_2 \overset{(1)}{\leq} 12\tau_\sigma (\delta^2_{stat}+\delta^2_1)\leq 24\tau_\sigma \delta_1^{2} \leq \frac{96 \tau_\sigma \xi_1^2}{\lambda^2}\overset{(2)}{\leq} \frac{96\tau_\sigma \xi_1}{4\lambda}\overset{(3)}{\leq} \frac{\xi_1}{4},$$
where (1) holds from the fact that $(a+b)^2\leq 2a^2+2b^2$, (2) holds using $ \frac{\xi_1}{\lambda}\leq \frac{\rho}{4}$, (3) uses the assumption on $\lambda$. Thus,
$$ \frac{\xi_2}{\lambda}\leq \frac{\xi_1}{4\lambda} \leq \frac{\rho}{16}.$$
In $i+1$th step, with similar argument, and by induction assumption we have
$$\xi_{j+1}\leq \frac{96\tau_\sigma\xi_j^2}{\lambda^2}\leq \frac{96\tau_\sigma \xi_j}{4^{2^{j}}\lambda}\leq \frac{\xi_j}{4^{2^{k-1}}} $$
and 
$$ \frac{\xi_{j+1}}{\lambda}\leq \frac{\xi_j}{4^{2^{j-1}} \lambda}\leq \frac{\rho}{4^{2^j}}.$$

Thus we know $\xi_j$ is a decreasing sequence, and  $ \mathbb{E} T_{k+1}\leq (1-\frac{1}{\kappa}) T_k$ holds until $ G(\theta^k)-G(\hat{\theta})\leq 6\tau_\sigma (2\delta_{stat})^2. $ We establish the result.
\end{proof}

\subsection{SAGA with non-convex objective function}

We start with some technical Lemmas. The following lemma is Lemma 5 extract from \cite{loh2013regularized}, we present here for the completeness.
\begin{lemma}\label{lemma.non_convex_norm}
	For any vector $ \theta\in R^p$, let $A$ denote the index set of its $r$ largest elements in magnitude, under assumption on $g_{\lambda,\mu}$ in Section \ref{section:nonconvex_regularizer}, we have 	
	$$ g_{\lambda,\mu}(\theta_A)-g_{\lambda,\mu} (\theta_{A^{c}})\leq \lambda L_g (\|\theta_A\|_1-\|\theta_{A^c}\|_1).$$ 	
	Moreover, for an arbitrary vector $\theta\in R^p$, we have	
	$$ g_{\lambda,\mu} (\theta^*)-g_{\lambda,\mu} (\theta)\leq \lambda L_g (\|\nu_A\|_1-\|\nu_{A^c}\|_1), $$
	where $\nu=\theta-\theta^*$ and $\theta^*$ is r sparse.
\end{lemma}

The next lemma is a non-convex counterpart of Lemma~\ref{lemma.cone} and Lemma \ref{lemma.cone_optimization}
\begin{lemma}\label{lemma.non_convex_cone_optimizatin}
	Suppose $g_{\lambda,\mu}(\cdot)$ satisfies the assumptions in section \ref{section:nonconvex_regularizer}, $\lambda L_g\geq8\rho\tau \frac{\log p }{n}$, $\lambda\geq \frac{4}{L_g} \|\nabla f (\theta^*)\|_\infty$, $\theta^*$ is feasible,   and there exists a pair $(\epsilon,K)$ such that
	$$G(\theta^k)-G(\hat{\theta})\leq \epsilon, \forall k\geq K.$$
	Then for any iteration $k\geq K$, we have	
	$$ \|\theta^k-\hat{\theta}\|_1\leq 4\sqrt{r} \|\theta^k-\hat{\theta}\|_2+8\sqrt{r} \|\theta^*-\hat{\theta}\|_2+2\min (\frac{\epsilon}{\lambda L_g}, \rho).$$
\end{lemma}
\begin{proof}
	Fix an arbitrary feasible $\theta$,
	Define $\Delta=\theta-\theta^* $.
	Suppose $G(\theta)-G(\hat{\theta})\leq \epsilon$, since we know $G(\hat{\theta})\leq G(\theta^*) $ so we have $ G(\theta)\leq G(\theta^*)+\epsilon$
	, which implies 
	$$ f(\theta^*+\Delta)+g_{\lambda,\mu} (\theta^*+\Delta)\leq f(\theta^*)+g_{\lambda,\mu}(\theta^*) +\epsilon .$$
	Subtract $\langle \nabla f(\theta^*),\Delta \rangle$ and use the RSC condition we have
	\begin{equation}
	\begin{split}
	&\frac{\sigma}{2} \|\Delta\|_2^2-\tau\frac{\log p}{n} \|\Delta\|_1^2+g_{\lambda,\mu} (\theta^*+\Delta)-g_{\lambda,\mu}(\theta^*)\\
	\leq &\epsilon-\langle \nabla f(\theta^*),\Delta \rangle\\
	\leq &\epsilon+\|\nabla f(\theta^*)\|_\infty \|\Delta\|_1
	\end{split}
	\end{equation}
	where the last inequality holds from Holder's inequality.
	Rearrange terms and use the fact that 
	$\|\Delta\|_1\leq 2\rho $ (by feasiblity of $\theta$ and $\theta^*$) and the assumptions $\lambda L_g\geq8\rho\tau \frac{\log p }{n}$ , $\lambda\geq \frac{4}{L_g} \|\nabla f (\theta^*)\|_\infty$, we obtain
	$$ \epsilon+\frac{1}{2} \lambda L_{g} \|\Delta\|_1+g_{\lambda,\mu}(\theta^*)-g_{\lambda,\mu}(\theta^*+\Delta)\geq \frac{\sigma}{2}\|\Delta\|_2^2\geq 0. $$
	
	By Lemma \ref{lemma.non_convex_norm}, we have
	$$ g_{\lambda,\mu} (\theta^*)-g_{\lambda,\mu} (\theta)\leq \lambda L_g (\|\Delta_A\|_1-\|\Delta_{A^c}\|_1) ,$$
	where $A$ indexes the top $r$ components of $\Delta$ in magnitude.
	So we have 
	$$ \frac{3\lambda L_g}{2} \|\Delta_A\|_1-\frac{\lambda L_g}{2} \|\Delta_{A^c}\|_1+\epsilon\geq 0, $$
	and consequently 
	$$ \|\Delta\|_1\leq \|\Delta_A\|_1+\|\Delta_{A^c}\|_1\leq 4\|\Delta_A\|_1+\frac{2\epsilon}{\lambda L_g} \leq 4\sqrt{r} \|\Delta\|_2 +\frac{2\epsilon}{\lambda L_g}.$$
	Combining this with $  \|\Delta\|_1\leq 2\rho$ leads to
	$$ \|\Delta\|_1\leq 4\sqrt{r} \|\Delta\|_2 +2\min \{\frac{\epsilon}{\lambda L_g}, \rho\}.$$
	Since this holds for any feasible $\theta$,  we have $\|\theta^k-\theta^*\|_1\leq 4\sqrt{r} \|\theta^k-\theta^*\|_2+2\min \{\frac{\epsilon}{\lambda L_g}, \rho\}.$
	
	Notice $G(\hat{\theta})-G(\theta^*)\leq 0$, so following same derivation as above and set $\epsilon=0$ we have
	$\|\hat{\theta}-\theta^*\|_1\leq 4\sqrt{r} \|\hat{\theta}-\theta^*\|_2.$
	
	Combining the two, we have 
	$$\|\theta^k-\hat{\theta}\|_1\leq \|\theta^k-\theta^*\|_1+\|\theta^*-\hat{\theta}\|_1\leq  4\sqrt{r} \|\theta^k-\hat{\theta}\|_2+8\sqrt{r} \|\theta^*-\hat{\theta}\|_2+2\min (\frac{\epsilon}{\lambda L_g}, \rho). $$
\end{proof}

Now we provide a counterpart of Lemma \ref{lemma.RSC_cone} in the non-convex case. Notice the main difference from the convex case is the coefficient in front of $ \|\theta^k-\hat{\theta}\|_2^2. $
\begin{lemma}\label{lemma.non_convex_RSC_cone}
	Under the same assumption of Lemma \ref{lemma.non_convex_cone_optimizatin}, we  have
	
	\begin{equation}
	\langle \nabla F(\theta^k)-\nabla F(\hat{\theta}), \theta^k-\hat{\theta}\rangle\geq  \left(\sigma-\mu-64\tau_\sigma r \right) \|\hat{\Delta}^k \|_2^2-2\epsilon^2 (\Delta^*,r )
	\end{equation}
	and
	$$ G(\theta^k)-G(\hat{\theta})\geq (\frac{\sigma-\mu}{2}-32r\tau_\sigma)\|\theta^k-\hat{\theta}\|_2^2-\epsilon^2 (\Delta^*,r ), $$
	where $\Delta^*=\hat{\theta}-\theta^*$, $\hat{\Delta}^k=\theta^k-\hat{\theta}$ and $ \epsilon^2 (\Delta^*,r )=2\tau_\sigma (8\sqrt{r}\|\hat{\theta}-\theta^*\|_2+2\min(\frac{\epsilon}{\lambda L_g},\rho))^2$, $\tau_{\sigma}=\tau\frac{\log p}{n}$.
\end{lemma}
\begin{proof} We have the following:
	
	\begin{equation} \label{Eq.2.lemma_proof}
	\begin{split}
	&F(\theta^k)-F(\hat{\theta})-\langle \nabla F(\hat{\theta}),\theta^k-\hat{\theta} \rangle\\
	=&f(\theta^k)-\frac{\mu}{2}\|\theta^k\|_2^2-f(\hat{\theta})+\frac{\mu}{2} \|\hat{\theta}\|_2^2-\langle \nabla f(\hat{\theta})-\mu\hat{\theta},\theta^k-\hat{\theta}  \rangle\\
	=& f(\theta^k)-f(\hat{\theta})-\langle \nabla f(\hat{\theta}),\theta^k-\hat{\theta} \rangle-\frac{\mu}{2}\|\theta^k-\hat{\theta}\|_2^2\\
	\geq &\frac{\sigma-\mu}{2} \|\theta^k-\hat{\theta}\|_2^2-\tau \frac{\log p}{n} \|\theta^k-\hat{\theta}\|_1^2,  
	\end{split}
	\end{equation}
	where the inequality uses the RSC condition.

	By Lemma \ref{lemma.non_convex_cone_optimizatin} we have 
	\begin{equation}
	\begin{split}
	\|\theta^k-\hat{\theta}\|_1^2\leq  &(4\sqrt{r} \|\theta^k-\hat{\theta}\|_2+8\sqrt{r} \|\theta^*-\hat{\theta}\|_2+2\min (\frac{\epsilon}{\lambda L_g}, \rho))^2 \\
	\leq &32 r \|\theta^k-\hat{\theta}\|_2^2+2 (8\sqrt{r}\|\hat{\theta}-\theta^*\|_2+2\min(\frac{\epsilon}{\lambda L_g},\rho))^2.
	\end{split}
	\end{equation}
	
	Substitute this into Equation \eqref{Eq.2.lemma_proof},  we obtain

	$$F(\theta^k)-F(\hat{\theta})-\langle \nabla F(\hat{\theta}),\theta^k-\hat{\theta} \rangle\geq (\frac{\sigma-\mu}{2}-32r\tau_\sigma)\|\theta^k-\hat{\theta}\|_2^2-2\tau_\sigma (8\sqrt{r}\|\hat{\theta}-\theta^*\|_2+2\min(\frac{\epsilon}{\lambda L_g},\rho))^2 $$

	Similarly, we have 
	
	$$F(\hat{\theta})-F(\theta^k)-\langle \nabla F(\theta^k),\hat{\theta}-\theta^k \rangle\geq (\frac{\sigma-\mu}{2}-32r\tau_\sigma)\|\theta^k-\hat{\theta}\|_2^2-2\tau_\sigma (8\sqrt{r}\|\hat{\theta}-\theta^*\|_2+2\min(\frac{\epsilon}{\lambda L_g},\rho))^2 $$
	
	Add above two equations together, we establish the result
	
	\begin{equation}
	\langle \nabla F(\theta^k)-\nabla F(\hat{\theta}), \theta^k-\hat{\theta}\rangle\geq  \left(\sigma-\mu-64\tau_\sigma r \right) \|\hat{\Delta}^k \|_2^2-2\epsilon^2 (\Delta^*,r ).
	\end{equation}	
	
	Next we bound $G(\theta^k)-G(\hat{\theta})$.	
	\begin{equation}\label{equ.xu.proof-lemma8}
	\begin{split}
	& G(\theta^k)-G(\hat{\theta})\\
	= & f(\theta^k)-f(\hat{\theta})-\frac{\mu}{2}\|\theta^k\|_2^2+\frac{\mu}{2}\|\hat{\theta}\|_2^2+ \lambda g_\lambda(\theta^k)-\lambda g_\lambda(\hat{\theta})\\
	\geq & \langle \nabla f(\hat{\theta}),\theta^k-\hat{\theta} \rangle+\frac{\sigma}{2}\|\theta^k-\hat{\theta}\|_2^2- \langle \mu\hat{\theta}, \theta^k-\hat{\theta} \rangle-\frac{\mu}{2} \|\theta^k-\hat{\theta}\|_2^2\\ &\qquad+ \lambda g_\lambda(\theta^k)-\lambda g_\lambda(\hat{\theta})-\tau \frac{\log p}{n} \|\theta^k-\hat{\theta}\|_1^2\\
	\geq &  \langle \nabla f(\hat{\theta}),\theta^k-\hat{\theta} \rangle+\frac{\sigma}{2}\|\theta^k-\hat{\theta}\|_2^2- \langle \mu\hat{\theta}, \theta^k-\hat{\theta} \rangle-\frac{\mu}{2} \|\theta^k-\hat{\theta}\|_2^2\\&\qquad+ \lambda \langle \partial g_\lambda(\hat{\theta}),\theta^k-\hat{\theta}\rangle-\tau \frac{\log p}{n} \|\theta^k-\hat{\theta}\|_1^2 \\
	= & \frac{\sigma-\mu}{2}\|\theta^k-\hat{\theta}\|_2^2-\tau\frac{\log p}{n}\|\theta^k-\hat{\theta}\|_1^2,
	\end{split}
	\end{equation}
	where the first inequality uses the RSC condition, the second inequality uses the convexity of $g_\lambda(\theta)$, and the last equality holds from the optimality condition of $\hat{\theta}$.

	Remind we have	
	\begin{equation}
	\begin{split}
	\|\theta^k-\hat{\theta}\|_1^2\leq  32 r \|\theta^k-\hat{\theta}\|_2^2+2 (8\sqrt{r}\|\hat{\theta}-\theta^*\|_2+2\min(\frac{\epsilon}{\lambda L_g},\rho))^2.
	\end{split}
	\end{equation}
	
	Substitute this into Equation~\eqref{equ.xu.proof-lemma8} we obtain
	$$ G(\theta^k)-G(\hat{\theta})\geq (\frac{\sigma-\mu}{2}-32r\tau_\sigma)\|\theta^k-\hat{\theta}\|_2^2-2\tau_\sigma (8\sqrt{r}\|\hat{\theta}-\theta^*\|_2+2\min(\frac{\epsilon}{\lambda L_g},\rho))^2.$$
\end{proof}

\begin{lemma} \label{lemma.non-convex-RSC_smooth}
	Under the same assumption of Lemma \ref{lemma.non_convex_cone_optimizatin}, we have 
	
	\begin{equation}
	\begin{split}
	&\langle \nabla F(\theta^k)- \nabla F(\hat{\theta}),\theta^k-\hat{\theta}\rangle\geq \frac{1}{2} [F(\theta^k)-F(\hat{\theta})-\langle\nabla F(\hat{\theta}),\theta^k-\hat{\theta} \rangle]+\left(\frac{\bar{\sigma}}{2}-\frac{\mu}{4} -\frac{\mu^2}{4L}\right) \|\theta^k-\hat{\theta}\|_2^2\\
	&+\frac{1}{8L} \|\nabla F(\theta^k)-\nabla F(\hat{\theta})\|_2^2-\epsilon^2 (\Delta^*,r). 
	\end{split}
	\end{equation}
\end{lemma}

\begin{proof}
	
	\begin{equation}\label{eq.5.lemma_proof}
	\begin{split}
	&\langle \nabla F(\theta^k)- \nabla F(\hat{\theta}),\theta^k-\hat{\theta}\rangle-\frac{1}{2} [F(\theta^k)-F(\hat{\theta})-\langle\nabla F(\hat{\theta}),\theta^k-\hat{\theta} \rangle]\\
	=& \frac{1}{2}\langle \nabla F(\theta^k)- \nabla F(\hat{\theta}),\theta^k-\hat{\theta}\rangle+\frac{1}{2}[F(\hat{\theta})-F(\theta^k)-\langle \nabla F(\theta^k),\hat{\theta}-\theta^k \rangle]
	\end{split}	
	\end{equation}
	Then we bound the right hand side of above equation.	
	Notice	
	\begin{equation}
	\begin{split}
	\|\nabla F(\theta^k)-\nabla F(\hat{\theta})\|_2^2&\leq 2\|\nabla f(\theta^k)-\nabla f(\hat{\theta})\|_2^2+2\mu^2\|\theta^k-\hat{\theta}\|_2^2\\
	& \leq 4L [f(\hat{\theta})-f(\theta^k)-\langle \nabla f(\theta^k), \hat{\theta}-\theta^k  \rangle]+2\mu^2 \|\theta^k-\hat{\theta}\|_2^2\\
	& \leq 	 4L [F(\hat{\theta})-F(\theta^k)-\langle \nabla F(\theta^k), \hat{\theta}-\theta^k  \rangle+\frac{\mu}{2}\|\theta^k-\hat{\theta}\|_2^2 ]+2\mu^2\|\theta^k-\hat{\theta}\|_2^2. 	
	\end{split}
	\end{equation}
	
	
	We then apply Lemma  \ref{lemma.non_convex_RSC_cone} and recall our definition on $\bar{\sigma}$ on non-convex case, we have   
	
	$$ \langle  \nabla F(\theta^k)-\nabla F(\hat{\theta}), \theta^k-\hat{\theta}\rangle\geq \bar{\sigma} \|\theta^k-\hat{\theta}\|_2^2-2\epsilon^2 (\Delta^*,r),$$
	
	
	Substitute above bound  in the right hand side of Equation \eqref{eq.5.lemma_proof}  we establish the result.	
\end{proof}

We are now ready to prove the main Theorem on non-convex $G(\theta)$, i.e., Theorem \ref{Theorem:non-convex}

\begin{proof}[Proof of Theorem \ref{Theorem:non-convex}]
	
	Recall the definition of $F_i(\theta)$ is   $F_i(\theta)=f_i(\theta)-\frac{\mu}{2}\|\theta\|_2^2 $ and the Lyapunov function
	
	\begin{equation}
	\begin{split}
	T_{k}&\triangleq \frac{1}{n} \sum_{i=1}^{n}\left( f_i(\phi_i^k)-f_i(\hat{\theta})- \langle \nabla f_i (\hat{\theta}),\phi_i^k-\hat{\theta} \rangle \right)+(c+\alpha) \|\theta^k-\hat{\theta
	}\|_2^2+b (G(\theta^{k})-G(\hat{\theta
}))\\
=& \frac{1}{n} \sum_{i=1}^{n} \left( F_i(\phi_i^k)-F_i(\hat{\theta})- \langle \nabla F_i (\hat{\theta}),\phi_i^k-\hat{\theta} \rangle+ \frac{\mu}{2}\|\phi^k_i-\hat{\theta}\|_2^2 \right)+ (c+\alpha) \|\theta^k-\hat{\theta
}\|_2^2+b (G(\theta^{k})-G(\hat{\theta
})).
\end{split}
\end{equation}

\textbf{1. Bound the first term on the right hand side fo $\mathbb{E} T_{k+1}$.}

Following similar steps in the convex case, we obtain

$$\mathbb{E}[\frac{1}{n}\sum_{i=1}^{n}F_i (\phi_i^{k+1}) ]=\frac{1}{n} F(\theta^k)+(1-\frac{1}{n})\frac{1}{n} \sum_{i=1}^{n} F_i(\phi_i^k) $$

$$\mathbb{E} [-\frac{1}{n}\sum_{i=1}^{n} \langle \nabla F_i(\hat{\theta}),\phi_i^{k+1}-\hat{\theta}\rangle] =-\frac{1}{n}\langle \nabla F(\hat{\theta}),\theta^k-\theta^*\rangle-(1-\frac{1}{n}) \frac{1}{n} \sum_{i=1}^{n} \langle \nabla F_i(\hat{\theta}),\phi^k_i-\hat{\theta} \rangle.$$

$$  \frac{\mu}{2n} \mathbb{E}\sum_{i=1}^{n} \|\phi_i^{k+1}-\hat{\theta}\|_2^2=\frac{\mu}{2} [ \frac{1}{n} \|\theta^k-\hat{\theta}\|_2^2+ (1-\frac{1}{n})\sum_{i=1}^{n} \frac{1}{n} \|\phi_i^k-\hat{\theta}\|_2^2   ].     $$

\textbf{ 2. Bound $(c+\alpha) \mathbb{E} \|\theta^{k+1}-\hat{\theta}\|_2^2 $}

In the following, we provide a upper bound on $c\mathbb{E} \|\theta^{k+1}-\hat{\theta}\|_2^2 $.  Notice  Equation \eqref{Eq.0.theorem} and the proof of Lemma \ref{lemma.variance} does not use convexity (The proof of Lemma 7 just use the fact $\mathbb{E} \|X-\mathbb{E}X\|_2^2= \mathbb{E} \|X\|_2^2-\|\mathbb{E}X\|_2^2 $, see \cite{defazio2014saga} for detail),  thus replace $f_i(\theta)$ by $F_i(\theta)$ we obtain the bound

\begin{equation}\label{eq.1.non_convex_proof}
\begin{split}
&\mathbb{E} \|\theta^{k+1}-\hat{\theta}\|_2^2 \\
\leq & \|\theta^k-\hat{\theta}\|_2^2-2\gamma \langle \nabla F(\theta^k)-\nabla F(\hat{\theta}),\theta^k-\hat{\theta} \rangle-\gamma^2\beta \|\nabla F(\theta^k)-\nabla F(\hat{\theta})\|_2^2\\
&+(1+\beta^{-1}) \gamma^2 \mathbb{E} \|\nabla F_j(\phi_j^k)-\nabla F_j (\hat{\theta})\|_2^2+(1+\beta) \gamma^2 \mathbb{E}\|\nabla F_j (\theta^k)-\nabla F_j(\hat{\theta})\|_2^2\\
\leq & (1-\gamma (\bar{\sigma}-\frac{\mu}{2}-\frac{\mu^2}{2L}))\|\theta^k-\hat{\theta}\|_2^2+2\gamma \epsilon^2(\Delta^*,r)-\gamma [F(\theta^k)-F(\hat{\theta})-\langle \nabla F(\hat{\theta}),\theta^k-\hat{\theta} \rangle ]\\
&+(1+\beta^{-1}) \gamma^2 \mathbb{E} \|\nabla F_j(\phi_j^k)-\nabla F_j (\hat{\theta})\|_2^2+\left((1+\beta) \gamma^2-\frac{\gamma}{4L}\right)  \mathbb{E}\|\nabla F_j (\theta^k)-\nabla F_j(\hat{\theta})\|_2^2,\\
\end{split}
\end{equation}

where the second inequality uses Lemma \ref{lemma.non-convex-RSC_smooth}.

%
%

Then we bound the term $\alpha \mathbb{E}\|\theta^{k+1}-\hat{\theta}\|_2^2 $. Using the Equation (30)  in \cite{qu2016linear}, we obtain

\begin{equation}
\begin{split}
\alpha \mathbb{E} \|\theta^{k+1}-\hat{\theta}\|_2^2 \leq \alpha  (1+\gamma \mu)\|\theta^k-\hat{\theta}\|_2^2-2\alpha \gamma \mathbb{E} [ G(\theta^{k+1})-G(\hat{\theta}) ]+2\alpha \gamma^2 \mathbb{E} \|\Delta\|_2^2.
\end{split}
\end{equation}

Same with convex case, we apply Lemma \ref{lemma.variance} on $E \|\Delta\|_2^2$ and obtain
\begin{equation}\label{eq.2.non_convex_proof}
\begin{split}
&\mathbb{E} \alpha\|\theta^{k+1}-\hat{\theta}\|_2^2\leq \alpha (1+\gamma\mu)\|\theta^k-\hat{\theta}\|_2^2-2\alpha\gamma \mathbb{E} [G(\theta^{k+1})-G(\hat{\theta})]\\
&+2\alpha(1+\beta^{-1})\gamma^2 \mathbb{E}\|\nabla F_j(\phi_j^k)-\nabla F_j(\hat{\theta})\|_2^2 +2\alpha(1+\beta)\gamma^2 \mathbb{E} \|F_j(\theta^k)-F_j (\hat{\theta})\|_2^2.
\end{split}
\end{equation}

Combine \eqref{eq.1.non_convex_proof} and \eqref{eq.2.non_convex_proof} together, we obtain 

\begin{equation}\label{eq.3.non_convex_proof}
\begin{split}
(\alpha+c) E\|\theta^{k+1}-\hat{\theta}\|_2^2\leq & \left(c-c \gamma(\bar{\sigma}-\mu/2-\frac{\mu^2}{2L})+\alpha(1+\gamma\mu)\right)\|\theta^k-\hat{\theta}\|_2^2+2c\gamma \epsilon^2(\Delta^*,r)\\
+&( (c+2\alpha) (1+\beta) \gamma^2-\frac{c\gamma}{4L})  \mathbb{E}\|\nabla F_j (\theta^k)-\nabla F_j(\hat{\theta})\|_2^2\\ + & (c+2\alpha)(1+\beta^{-1}) \gamma^2  \mathbb{E}\|\nabla F_j(\phi_j^k)-\nabla F_j(\hat{\theta})\|_2^2 -2\alpha \gamma \mathbb{E} [ G(\theta^{k+1})-G(\hat{\theta}) ]\\
-&c\gamma [F(\theta^k)-F(\hat{\theta})-\langle \nabla F(\hat{\theta}),\theta^k-\hat{\theta} \rangle ]
\end{split}
\end{equation}

We then bound $\mathbb{E} \|\nabla F_j(\phi_j^k)-\nabla F_j (\hat{\theta})\|_2^2 .$

\begin{equation}
\begin{split}
& \mathbb{E} \|\nabla F_j(\phi_j^k)-\nabla F_j (\hat{\theta})\|_2^2\\
=& \mathbb{E} \|\nabla f_j(\phi_j^k)-\nabla f_j (\hat{\theta})-u (\phi_j^k-\hat{\theta})\|_2^2\\
\leq & 2\mathbb{E} \|\nabla f_j(\phi_j^k)-\nabla f_j (\hat{\theta})\|_2^2+ 2\mu^2 \mathbb{E} \|\phi_j^k-\hat{\theta}\|_2^2\\
\leq & \frac{4L}{n} \sum_{i=1}^{n} [ f_i (\phi_i^k)-f_i(\hat{\theta})-\langle \nabla f_i(\phi_i^k) ,\phi_i^k-\hat{\theta} \rangle ]+ \frac{2\mu^2}{n }\sum_{i=1}^{n}\|\phi_i^k-\hat{\theta}\|_2^2\\
= &       \frac{4L}{n} \sum_{i=1}^{n} [ F_i (\phi_i^k)-F_i(\hat{\theta})-\langle\nabla F_i(\phi_i^k) ,\phi_i^k-\hat{\theta} \rangle +\frac{\mu}{2} \|\phi_i^k-\hat{\theta}\|_2^2 ]+ \frac{2\mu^2}{n }\sum_{i=1}^{n}\|\phi_i^k-\hat{\theta}\|_2^2\\
\leq & \frac{4L}{n} \sum_{i=1}^{n} [ F_i (\phi_i^k)-F_i(\hat{\theta})-\langle \nabla F_i(\phi_i^k) ,\phi_i^k-\hat{\theta} \rangle]+\frac{2\mu (L+\mu)}{n} \sum_{i=1}^{n}\|\phi_i^k-\hat{\theta}\|_2^2.
\end{split}
\end{equation}
where the first inequality holds from the fact $(a+b)^2\leq 2a^2+2b^2$, the second one uses the convexity and smoothness of $f_i(\theta).$

Substitute above bound in Equation \eqref{eq.3.non_convex_proof}, we obtain
\begin{equation}
\begin{split}
&(\alpha+c) E\|\theta^{k+1}-\hat{\theta}\|_2^2\leq \left(c-c \gamma(\bar{\sigma}-\mu/2-\frac{\mu^2}{2L})+\alpha(1+\gamma\mu)\right)\|\theta^k-\hat{\theta}\|_2^2+2c\gamma \epsilon^2(\Delta^*,r)\\
+&( (c+2\alpha) (1+\beta) \gamma^2-\frac{c\gamma}{4L})  \mathbb{E}\|\nabla F_j (\theta^k)-\nabla F_j(\hat{\theta})\|_2^2\\ + & (c+2\alpha)(1+\beta^{-1}) \gamma^2  
\left(\frac{4L}{n} \sum_{i=1}^{n} [ F_i (\phi_i^k)-F_i(\hat{\theta})-\langle \nabla F_i(\phi_i^k) ,\phi_i^k-\hat{\theta} \rangle]+\frac{2\mu(L+\mu)}{n} \sum_{i=1}^{n}\|\phi_i^k-\hat{\theta}\|_2^2 \right)\\
&-2\alpha \gamma \mathbb{E} [ G(\theta^{k+1})-G(\hat{\theta}) ]-c\gamma [F(\theta^k)-F(\hat{\theta})-\langle \nabla F(\hat{\theta}),\theta^k-\hat{\theta} \rangle ].
\end{split}
\end{equation}
\textbf{3 . Relate $\mathbb{E} T_{k+1}$ to $T_{k}$}

Combine all above together, we obtain 

\begin{equation}
\begin{split}
&\mathbb{E} T_{k+1}-T_k\leq -\frac{1}{\kappa} T_k+ (\frac{c+\alpha}{\kappa}-c\gamma (\bar{\sigma} -\frac{\mu}{2}-\frac{\mu^2}{2L})+\alpha \gamma\mu  ) \|\theta^k-\hat{\theta}\|_2^2\\
+ &\left( \frac{1}{\kappa}+ 4 (c+2\alpha)(1+\beta^{-1}) \gamma^2   L-\frac{1}{n}  \right)[\frac{1}{n} \sum_{i=1}^{n} F_i(\phi_i ^k)-F(\hat{\theta})-\frac{1}{n}\sum_{i=1}^{n} \langle \nabla F_i(\hat{\theta}),\phi_i^k-\hat{\theta} \rangle]\\
+ & \left( \frac{1}{n}-c\gamma \right)[F(\theta^k)-F(\hat{\theta})-\langle \nabla F(\hat{\theta}),\theta^k-\hat{\theta} \rangle ]+ \left( (c+2\alpha) (1+\beta) \gamma^2-\frac{c\gamma}{4L}\right)  \mathbb{E}\|\nabla F_j (\theta^k)-\nabla F_j(\hat{\theta})\|_2^2\\
+&2c\gamma \epsilon^2 (\Delta^*,r) -2\alpha \gamma \mathbb{E} [G(\theta^{k+1})-G(\hat{\theta})]+b \mathbb{E}(G(\theta^{k+1})-G(\hat{\theta})) -(1-\frac{1}{\kappa})b [G(\theta^k)-G(\hat{\theta})]\\
+&( \frac{\mu}{2\kappa}- \frac{\mu}{2n}+2(c+2\alpha) (1+\beta^{-1})\gamma^2 \mu(\mu+L) ) \frac{1}{n}\sum_{i=1}^{n} \|\phi_i^k-\hat{\theta}\|_2^2. 
\end{split}
\end{equation}

We choose $\beta=2$, $\gamma=\frac{1}{24L}$, $c=2\alpha$, $c=\frac{24L}{n}$, $b=2\alpha \gamma$, $\frac{1}{\kappa}=\frac{1}{24}\min \{ \frac
{2\bar{\sigma}}{5L}, \frac{1}{n} \}$ and recall our assumption that $\mu\leq \frac{\bar{\sigma}}{3}  $ and $\mu\leq \frac{L}{3}$, we obtain

$$ \mathbb{E} T_{k+1}-T_k\leq -\frac{1}{\kappa} T_k+2c\gamma\epsilon^2 (\Delta^*,r)-(1-\frac{1}{\kappa}) c \gamma [G(\theta^{k})-G(\hat{\theta})].$$

The following argument is identical to the convex $G(\theta)$, and the only difference is to replace $\lambda$ by $\lambda L_g $. Thus we have 
$ \mathbb{E} T_{k+1}\leq (1-\frac{1}{\kappa}) T_k$ holds until $ G(\theta^k)-G(\hat{\theta})\leq 6\tau_\sigma (2\delta_{stat})^2. $ Thus we establish the result.

\end{proof}

\subsection{Proof of corollaries}
We now prove the corollaries instantiating our main theorems to different statistical estimators. 
\begin{proof}[Proof of Corollary on Lasso, i.e., corollary \ref{cor.lasso}  ]
	We begin the proof, by presenting the below lemma of the RSC,  proved in \cite{raskutti2010restricted}, and we then use it in the case of Lasso.
	\begin{lemma}
		if each data point $x_i$ is i.i.d.\ random sampled from the distribution $N(0,\Sigma) $, then there are some universal constants $c_0$ and $c_1$ such that
		$$ \frac{\|X\Delta\|_2^2}{n}\geq \frac{1}{2}\|\Sigma^{1/2}\Delta\|_2^2-c_1\nu(\Sigma)\frac{\log p}{n} \|\Delta\|_1^2, \quad \mbox{for all } \Delta \in \mathbb{R}^p ,$$
		with probability at least $1-\exp(-c_0n)$. Here, $X$ is the data matrix where each row is data point $x_i $.
	\end{lemma}
	
	Since $\theta^*$ is support on a subset $S$ with cardinality r,  we choose $$ \bar{M}(S):=\{ \theta\in \mathbb{R}^p | \theta_j=0 \text{~for all ~} j\notin S   \}. $$ It is straightforward to choose $M(S)=\bar{M}(S)$ and notice that $\theta^*\in M(S)$.
	In Lasso formulation, $f(\theta)=\frac{1}{2n}\|y-X\theta\|_2^2$, and hence it is easy to verify that
	$$ f(\theta+\Delta)-f(\theta)-\langle \nabla f(\theta),\Delta \rangle\geq \frac{1}{2n} \|X\Delta\|_2^2\geq \frac{1}{4}\|\Sigma^{1/2}\Delta\|_2^2-\frac{c_1}{2}\nu(\Sigma)\frac{\log p}{n} \|\Delta\|_1^2  .$$
	Also, $\psi(\cdot)$ is $\|\cdot\|_1$ in Lasso, and hence $H(\bar{M})=\sup_{\theta\in \bar{M}\backslash \{0\}} \frac{\|\theta\|_1}{\|\theta\|_2}=\sqrt{r}$. Thus we have 
	$$\bar{\sigma}=\frac{1}{2}\sigma_{\min} (\Sigma)-64c_1 \nu(\Sigma)\frac{ r\log p}{n} .$$

	On the other hand,	the tolerance is 	
	\begin{equation}
	\begin{split}
	\delta &=24\tau_\sigma(8H(\bar{M})\|\Delta^*\|_2+8\psi(\theta^*_{M^\perp}))^2\\
	&=c_2\nu(\Sigma)\frac{r\log p}{n}\|\Delta^*\|_2^2,
	\end{split}
	\end{equation}
	where we use the fact that $\theta^*\in M(S)$, which implies  $\psi(\theta^*_{M^{\perp}})=0$.\\
	
	Recall we require $\lambda$ to satisfy $\lambda\geq 2\psi^*(\nabla f(\theta^*))$. In Lasso we have $\psi^*(\cdot)=\|\cdot\|_\infty$. Using the fact that $y_i=x_i^T\theta^*+\xi_i$, we have $ \lambda\geq \frac{2}{n}\|X^T\xi\|_\infty $. Then we apply the tail bound on the Gaussian variable and use union bound to obtain that
	$$\frac{2}{n}\|X^T\xi\|_\infty\leq 6\varsigma\sqrt{\frac{\log p}{n}} $$ holds
	with probability at least $1-\exp(-3\log p)$.
\end{proof}

\begin{proof}[Proof of Corollary on Group Lasso, i.e., corollary \ref{cor.group_lasso}] 
	We use the following fact on the RSC condition of Group Lasso \cite{negahban2009unified}\cite{negahban2012supplement}: if each data point $x_i$ is i.i.d.\ randomly sampled from the distribution $N(0,\Sigma) $, then  there exists strictly positive constant  $(\sigma_1(\Sigma),\sigma_2(\Sigma))$ which only depends on $\Sigma$ such that , 
	$$ \frac{ \|X\Delta\|_2^2}{2n} \geq \sigma_1(\Sigma) \|\Delta\|^2_2-\sigma_2(\Sigma) (\sqrt{\frac{m}{n}}+\sqrt{\frac{3 \log N_{\mathcal{G}}}{n}} )^2 \|\Delta\|^2_{\mathcal{G},2}, \quad \mbox{for all } \Delta \in \mathbb{R}^p ,$$
	with probability at least $1-c_3\exp (-c_4n).$ 
	
	Remind we define the subspace
	$$  \bar{M}(S_{\mathcal{G}})= M(S_{\mathcal{G}})=\{\theta|\theta_{G_i}=0 \text{~for all~} i\notin S_{\mathcal{G}}\} $$
	where $S_{\mathcal{G}}$ corresponds to non-zero group of $\theta^*$.
	
	The subspace compatibility can be computed by $$ H(\bar{M})=\sup_{\theta\in \bar{M}\backslash \{0\}} \frac{\|\theta\|_{\mathcal{G},2}}{\|\theta\|_2}=\sqrt{s_{\mathcal{G}}}.$$
	Thus, the modified RSC parameter 
	$$\bar{\sigma}=\sigma_1(\Sigma)-c\sigma_2(\Sigma)s_{\mathcal{G}} (\sqrt{\frac{m}{n}}+\sqrt{\frac{3 \log N_{\mathcal{G}}}{n}} )^2  .$$
	
	We then bound the value of $\lambda$. 
	As the regularizer in Group Lasso is $\ell_{1,2}$ grouped norm, its dual norm  is $(\infty,2)$ grouped norm.
	So it suffices to have any $\lambda$ such that $$\lambda\geq 2 \max_{i=1,...,N_{\mathcal{G}}} \|\frac{1}{n}(X^T\xi)_{G_i}\|_2.$$
	
	Using Lemma 5 in \cite{negahban2009unified}, we know 
	$$\max_{i=1,...,N_{\mathcal{G}}} \|\frac{1}{n}(X^T\xi)_{G_i}\|_2\leq 2\varsigma (\sqrt{\frac{m}{n}}+\sqrt{\frac{\log N_{\mathcal{G}}}{n}}) $$
	with probability at least $1-2\exp (-2\log N_{\mathcal{G}})$.
	Thus it suffices to choose $\lambda= 4\varsigma (\sqrt{\frac{m}{n}}+\sqrt{\frac{\log N_{\mathcal{G}}}{n}})$.
	
	%

	The tolerance is given by,
	\begin{equation}
	\begin{split}
	\delta&=24\tau_\sigma(8H(\bar{M})\|\Delta^*\|_2+8\psi(\theta^*_{M^\perp}))^2\\
	&=c_2\sigma_2(\Sigma)s_{\mathcal{G}} (\sqrt{\frac{m}{n}}+\sqrt{\frac{3 \log N_{\mathcal{G}}}{n}} )^2   \|\Delta^*\|_2^2,
	\end{split}
	\end{equation}
	where we use the fact $\psi(\theta^*_{M^\perp})=0$.
\end{proof}

\begin{proof}[Proof of Corollary on  SCAD, i.e., corollary \ref{cor.SCAD}]
	
	The proof is very similar to that of Lasso. In the proof of results for Lasso, we established
	
	$$ \|\nabla f(\theta^*)\|_\infty=\frac{1}{n}\|X^T\xi \|_\infty\leq 3 \varsigma \sqrt{\frac{\log p}{n}} $$ 
	and the RSC condition 
	$$ \frac{\|X\Delta\|_2^2}{n}\geq \frac{1}{2}\|\Sigma^{1/2}\Delta\|_2^2-c_1\nu(\Sigma)\frac{\log p}{n} \|\Delta\|_1^2.$$
	Recall that $\mu=\frac{1}{\zeta-1}$ and $L_g=1$, we establish the corollary. 
\end{proof}

\begin{proof}[Proof of corollary on Corrected Lasso, i.e., corollary \ref{cor.corrected_lasso}]
	First notice
	$$ \|\nabla f(\theta^*)\|_\infty=\|\hat{\Gamma}\theta^*-\hat{\gamma}\|_{\infty}=\|\hat{\gamma}-\Sigma\theta^* +(\Sigma-\hat{\Gamma})\theta^*\|_{\infty} \leq \|\hat{\gamma}-\Sigma\theta^*\|_\infty+\|(\Sigma-\hat{\Gamma})\theta^*\|_{\infty}.$$
	As shown in literature (Lemma 2 in \cite{loh2011high}), both terms on the right hand side can be bounded by
	$c_1 \varphi \sqrt{\frac{\log p}{n}}$, where $\varphi\triangleq(\sqrt{\sigma_{\max} (\Sigma)}+\sqrt{\gamma_w}) (\varsigma+\sqrt{\gamma_w} \|\theta^*\|_2)$, with probability at least $1-c_1\exp(-c_2\log p )$. 
	
	To obtain the RSC condition,  we apply Lemma 1 in \cite{loh2011high}, to get
	$$ \frac{1}{n}\Delta^T\hat{\Gamma}\Delta \geq \frac{\sigma_{\min} (\Sigma)}{2} \|\Delta\|_2^2-c_3 \sigma_{\min}(\Sigma)\max \left( (\frac{\sigma_{\max} (\Sigma)+\gamma_w}{\sigma_{\min}(\Sigma)})^2,1 \right) \frac{\log p}{n} \|\Delta\|_1^2,$$
	with probability at least $1-c_4 \exp \left(-c_5 n\min \big( \frac{\sigma^2_{\min} (\Sigma)}{( \sigma_{\max}(\Sigma)+\gamma_w)^2},1 \big)    \right)$.
	
	Combine these together, we establish the corollary.
\end{proof}
\bibliography{SAGA}

\begin{thebibliography}{38}
\providecommand{\natexlab}[1]{#1}
\providecommand{\url}[1]{\texttt{#1}}
\expandafter\ifx\csname urlstyle\endcsname\relax
  \providecommand{\doi}[1]{doi: #1}\else
  \providecommand{\doi}{doi: \begingroup \urlstyle{rm}\Url}\fi

\bibitem[Agarwal et~al.(2010)Agarwal, Negahban, and
  Wainwright]{agarwal2010fast}
Alekh Agarwal, Sahand Negahban, and Martin~J Wainwright.
\newblock Fast global convergence rates of gradient methods for
  high-dimensional statistical recovery.
\newblock In \emph{Advances in Neural Information Processing Systems}, pages
  37--45, 2010.

\bibitem[Allen-Zhu and Hazan(2016)]{allen2016variance}
Zeyuan Allen-Zhu and Elad Hazan.
\newblock Variance reduction for faster non-convex optimization.
\newblock \emph{arXiv preprint arXiv:1603.05643}, 2016.

\bibitem[Defazio et~al.(2014)Defazio, Bach, and
  Lacoste-Julien]{defazio2014saga}
Aaron Defazio, Francis Bach, and Simon Lacoste-Julien.
\newblock Saga: A fast incremental gradient method with support for
  non-strongly convex composite objectives.
\newblock In \emph{Advances in Neural Information Processing Systems}, pages
  1646--1654, 2014.

\bibitem[Fan and Li(2001)]{fan2001variable}
Jianqing Fan and Runze Li.
\newblock Variable selection via nonconcave penalized likelihood and its oracle
  properties.
\newblock \emph{Journal of the American statistical Association}, 96\penalty0
  (456):\penalty0 1348--1360, 2001.

\bibitem[Goodfellow et~al.(2016)Goodfellow, Bengio, and
  Courville]{Goodfellow-et-al-2016}
Ian Goodfellow, Yoshua Bengio, and Aaron Courville.
\newblock \emph{Deep Learning}.
\newblock MIT Press, 2016.
\newblock \url{http://www.deeplearningbook.org}.

\bibitem[Guyon(2008)]{SIDO}
I.~Guyon.
\newblock A phamacology dataset, 06 2008.
\newblock URL \url{http://www.causality.inf.ethz.ch/data/SIDO.html}.

\bibitem[Hajinezhad et~al.(2016)Hajinezhad, Hong, Zhao, and
  Wang]{hajinezhad2016nestt}
Davood Hajinezhad, Mingyi Hong, Tuo Zhao, and Zhaoran Wang.
\newblock Nestt: A nonconvex primal-dual splitting method for distributed and
  stochastic optimization.
\newblock In \emph{Advances in Neural Information Processing Systems}, pages
  3207--3215, 2016.

\bibitem[Harikandeh et~al.(2015)Harikandeh, Ahmed, Virani, Schmidt,
  Kone{\v{c}}n{\`y}, and Sallinen]{harikandeh2015stopwasting}
Reza Harikandeh, Mohamed~Osama Ahmed, Alim Virani, Mark Schmidt, Jakub
  Kone{\v{c}}n{\`y}, and Scott Sallinen.
\newblock Stopwasting my gradients: Practical svrg.
\newblock In \emph{Advances in Neural Information Processing Systems}, pages
  2251--2259, 2015.

\bibitem[Harrison and Rubinfeld(2013)]{uci:2013}
D.~Harrison and D.L. Rubinfeld.
\newblock {UCI} machine learning repository, 2013.
\newblock URL \url{http://archive.ics.uci.edu/ml}.

\bibitem[Johnson and Zhang(2013)]{johnson2013accelerating}
Rie Johnson and Tong Zhang.
\newblock Accelerating stochastic gradient descent using predictive variance
  reduction.
\newblock In \emph{Advances in Neural Information Processing Systems}, pages
  315--323, 2013.

\bibitem[Karimi et~al.(2016)Karimi, Nutini, and Schmidt]{karimi2016linear}
Hamed Karimi, Julie Nutini, and Mark Schmidt.
\newblock Linear convergence of gradient and proximal-gradient methods under
  the polyak-{\l}ojasiewicz condition.
\newblock In \emph{Joint European Conference on Machine Learning and Knowledge
  Discovery in Databases}, pages 795--811. Springer, 2016.

\bibitem[Lewis et~al.(2004)Lewis, Yang, Rose, and Li]{lewis2004rcv1}
David~D Lewis, Yiming Yang, Tony~G Rose, and Fan Li.
\newblock Rcv1: A new benchmark collection for text categorization research.
\newblock \emph{Journal of machine learning research}, 5\penalty0
  (Apr):\penalty0 361--397, 2004.

\bibitem[Li et~al.(2016)Li, Zhao, Arora, Liu, and Haupt]{li2016stochastic}
Xingguo Li, Tuo Zhao, Raman Arora, Han Liu, and Jarvis Haupt.
\newblock Stochastic variance reduced optimization for nonconvex sparse
  learning.
\newblock \emph{arXiv preprint arXiv:1605.02711}, 2016.

\bibitem[Loh and Wainwright(2011)]{loh2011high}
Po-Ling Loh and Martin~J Wainwright.
\newblock High-dimensional regression with noisy and missing data: Provable
  guarantees with non-convexity.
\newblock In \emph{Advances in Neural Information Processing Systems}, pages
  2726--2734, 2011.

\bibitem[Loh and Wainwright(2013)]{loh2013regularized}
Po-Ling Loh and Martin~J Wainwright.
\newblock Regularized m-estimators with nonconvexity: Statistical and
  algorithmic theory for local optima.
\newblock In \emph{Advances in Neural Information Processing Systems}, pages
  476--484, 2013.

\bibitem[Negahban et~al.(2012)Negahban, Ravikumar, Wainwright, and
  Yu]{negahban2012supplement}
S~Negahban, P~Ravikumar, MJ~Wainwright, and B~Yu.
\newblock Supplement to “a unified framework for high-dimensional analysis of
  $ m $-estimators with decomposable regularizers.”, 2012.

\bibitem[Negahban et~al.(2009)Negahban, Yu, Wainwright, and
  Ravikumar]{negahban2009unified}
Sahand Negahban, Bin Yu, Martin~J Wainwright, and Pradeep~K Ravikumar.
\newblock A unified framework for high-dimensional analysis of $ m $-estimators
  with decomposable regularizers.
\newblock In \emph{Advances in Neural Information Processing Systems}, pages
  1348--1356, 2009.

\bibitem[Nesterov(1998)]{nesterov1998introductory}
Yu~Nesterov.
\newblock Introductory lectures on convex programming volume i: Basic course.
\newblock \emph{Lecture notes}, 1998.

\bibitem[Nesterov(2013)]{nesterov2013introductory}
Yurii Nesterov.
\newblock \emph{Introductory lectures on convex optimization: A basic course},
  volume~87.
\newblock Springer Science \& Business Media, 2013.

\bibitem[Prokhorov(2001)]{IJCNN1}
Danil Prokhorov.
\newblock Ijcnn 2001 neural network competition, 2001.

\bibitem[Qu et~al.(2016)Qu, Li, and Xu]{qu2016linear}
Chao Qu, Yan Li, and Huan Xu.
\newblock Linear convergence of svrg in statistical estimation.
\newblock \emph{arXiv preprint arXiv:1611.01957}, 2016.

\bibitem[Qu et~al.(2015)Qu, Richt{\'a}rik, and Zhang]{qu2015quartz}
Zheng Qu, Peter Richt{\'a}rik, and Tong Zhang.
\newblock Quartz: Randomized dual coordinate ascent with arbitrary sampling.
\newblock In \emph{Advances in neural information processing systems}, pages
  865--873, 2015.

\bibitem[Quattoni et~al.(2009)Quattoni, Carreras, Collins, and
  Darrell]{quattoni2009efficient}
Ariadna Quattoni, Xavier Carreras, Michael Collins, and Trevor Darrell.
\newblock An efficient projection for l 1, infinity regularization.
\newblock In \emph{Proceedings of the 26th Annual International Conference on
  Machine Learning}, pages 857--864. ACM, 2009.

\bibitem[Raskutti et~al.(2010)Raskutti, Wainwright, and
  Yu]{raskutti2010restricted}
Garvesh Raskutti, Martin~J Wainwright, and Bin Yu.
\newblock Restricted eigenvalue properties for correlated gaussian designs.
\newblock \emph{Journal of Machine Learning Research}, 11\penalty0
  (Aug):\penalty0 2241--2259, 2010.

\bibitem[Reddi et~al.(2016)Reddi, Sra, Poczos, and Smola]{reddi2016fast}
Sashank~J Reddi, Suvrit Sra, Barnabas Poczos, and Alex Smola.
\newblock Fast stochastic methods for nonsmooth nonconvex optimization.
\newblock \emph{arXiv preprint arXiv:1605.06900}, 2016.

\bibitem[Schmidt et~al.(2013)Schmidt, Roux, and Bach]{schmidt2013minimizing}
Mark Schmidt, Nicolas~Le Roux, and Francis Bach.
\newblock Minimizing finite sums with the stochastic average gradient.
\newblock \emph{arXiv preprint arXiv:1309.2388}, 2013.

\bibitem[Shalev-Shwartz(2016)]{shalev2016sdca}
Shai Shalev-Shwartz.
\newblock Sdca without duality, regularization, and individual convexity.
\newblock \emph{arXiv preprint arXiv:1602.01582}, 2016.

\bibitem[Shalev-Shwartz and Zhang(2014)]{shalev2014accelerated}
Shai Shalev-Shwartz and Tong Zhang.
\newblock Accelerated proximal stochastic dual coordinate ascent for
  regularized loss minimization.
\newblock In \emph{ICML}, pages 64--72, 2014.

\bibitem[Shamir(2015)]{shamir2015fast}
Ohad Shamir.
\newblock Fast stochastic algorithms for svd and pca: Convergence properties
  and convexity.
\newblock \emph{arXiv preprint arXiv:1507.08788}, 2015.

\bibitem[Swirszcz et~al.(2009)Swirszcz, Abe, and Lozano]{swirszcz2009grouped}
Grzegorz Swirszcz, Naoki Abe, and Aurelie~C Lozano.
\newblock Grouped orthogonal matching pursuit for variable selection and
  prediction.
\newblock In \emph{Advances in Neural Information Processing Systems}, pages
  1150--1158, 2009.

\bibitem[Turlach et~al.(2005)Turlach, Venables, and
  Wright]{turlach2005simultaneous}
Berwin~A Turlach, William~N Venables, and Stephen~J Wright.
\newblock Simultaneous variable selection.
\newblock \emph{Technometrics}, 47\penalty0 (3):\penalty0 349--363, 2005.

\bibitem[Xiang et~al.(2014)Xiang, Yang, and Ye]{xiang2014simultaneous}
Shuo Xiang, Tao Yang, and Jieping Ye.
\newblock Simultaneous feature and feature group selection through hard
  thresholding.
\newblock In \emph{Proceedings of the 20th ACM SIGKDD international conference
  on Knowledge discovery and data mining}, pages 532--541. ACM, 2014.

\bibitem[Xiao(2010)]{xiao2010dual}
Lin Xiao.
\newblock Dual averaging methods for regularized stochastic learning and online
  optimization.
\newblock \emph{Journal of Machine Learning Research}, 11\penalty0
  (Oct):\penalty0 2543--2596, 2010.

\bibitem[Xiao and Zhang(2013)]{xiao2013proximal}
Lin Xiao and Tong Zhang.
\newblock A proximal-gradient homotopy method for the sparse least-squares
  problem.
\newblock \emph{SIAM Journal on Optimization}, 23\penalty0 (2):\penalty0
  1062--1091, 2013.

\bibitem[Xiao and Zhang(2014)]{xiao2014proximal}
Lin Xiao and Tong Zhang.
\newblock A proximal stochastic gradient method with progressive variance
  reduction.
\newblock \emph{SIAM Journal on Optimization}, 24\penalty0 (4):\penalty0
  2057--2075, 2014.

\bibitem[Yuan and Lin(2006)]{yuan2006model}
Ming Yuan and Yi~Lin.
\newblock Model selection and estimation in regression with grouped variables.
\newblock \emph{Journal of the Royal Statistical Society: Series B (Statistical
  Methodology)}, 68\penalty0 (1):\penalty0 49--67, 2006.

\bibitem[Zhang and Zhang(2012)]{zhang2012general}
Cun-Hui Zhang and Tong Zhang.
\newblock A general theory of concave regularization for high-dimensional
  sparse estimation problems.
\newblock \emph{Statistical Science}, pages 576--593, 2012.

\bibitem[Zhang and Lin(2015)]{zhang2015stochastic}
Yuchen Zhang and Xiao Lin.
\newblock Stochastic primal-dual coordinate method for regularized empirical
  risk minimization.
\newblock In \emph{ICML}, pages 353--361, 2015.

\end{thebibliography}
\bibliographystyle{plainnat}

\end{document}